%% file: ms2.tex
\documentclass[11pt]{article}

 \pdfoutput=1 
\usepackage{mkolar_definitions}
\usepackage{color}
\usepackage{fullpage}
\usepackage{wrapfig}
\usepackage{algorithm}
\usepackage{algorithmic}
\usepackage{multirow}
\usepackage{tabularx}
\usepackage{array}
\usepackage{natbib}
\usepackage{mathabx}
\usepackage{graphicx}
\usepackage{rotating}
\usepackage{multirow}
\usepackage{bibunits}
\usepackage{amsfonts}
\usepackage{caption}
\usepackage{amssymb}
\usepackage{pifont}
\usepackage{authblk}
\usepackage[usenames,dvipsnames,svgnames,table]{xcolor}
\usepackage[colorlinks=true,
            linkcolor=blue,
            urlcolor=blue,
            citecolor=blue]{hyperref}

\numberwithin{equation}{section}
\numberwithin{theorem}{section}

\newcommand{\cmark}{\ding{51}}%
\newcommand{\xmark}{\ding{55}}%

\DeclareMathOperator{\dist}{dist}
\DeclareMathOperator{\pr}{pr}

\DeclareMathOperator{\Unif}{unif}
\def\BIC{\textsc{BIC}}

\DeclareMathOperator{\var}{var}

\DeclareMathOperator{\vol}{vol}

\usepackage{xargs}
\usepackage[colorinlistoftodos,prependcaption,textsize=tiny]{todonotes}
\newcommandx{\unsure}[2][1=]{\todo[linecolor=red,backgroundcolor=red!25,bordercolor=red,#1]{#2}}
\newcommandx{\change}[2][1=]{\todo[linecolor=blue,backgroundcolor=blue!25,bordercolor=blue,#1]{#2}}
\newcommandx{\info}[2][1=]{\todo[linecolor=OliveGreen,backgroundcolor=OliveGreen!25,bordercolor=OliveGreen,#1]{#2}}
\newcommandx{\improvement}[2][1=]{\todo[linecolor=Plum,backgroundcolor=Plum!25,bordercolor=Plum,#1]{#2}}

\usepackage{enumitem}
\newlist{steps}{enumerate}{1}
\setlist[steps, 1]{label = Step \arabic*:}

\usepackage{thmtools}
\usepackage{thm-restate}
\usepackage{hyperref}
\usepackage{cleveref}

\newcommand{\singv}[2]{\sigma_{#1}(#2)}
\newcommand{\singvtwo}[2]{\sigma_{#1}^2(#2)}
\newcommand{\Gnorm}[3]{\left\|#1\right\|_{#2,#3}}
\newcommand{\staterr}{\varepsilon_{stat}}
\newcommand{\gradAwrtcov}[1]{W(#1)}
\newcommand{\defgradAwrtcov}[2]{\gradAwrtcov{#1}=[w_{kj}(#2_k)]\in\mathbb{R}^{K\times J},\ \ 
w_{kj}(#2_k)=#2_k^T\nabla\ell_{N,j}(\Sigma_j)#2_k,\ \  
\nabla\ell_{N,j}({\Sigma}_j) = {\Sigma}_j - \sampleCov{j}}
\newcommand{\gradAwrtpcov}[1]{W^\star(#1)}

\newcommand{\defgradAwrtpcov}[2]{\gradAwrtpcov{#1}=[w^\star_{kj}(#2_k)]\in\mathbb{R}^{K\times J},\ \ 
w^\star_{kj}(#2_k)=#2_k^T\nabla\ell_{N,j}(\pCov_j)#2_k
}

\newcommand{\rowkA}{A_{k\cdot}}
\newcommand{\projA}{\mathcal{C}_A}
\newcommand{\defprojA}[2]{\projA(#1, #2) = 
\{\alpha=Qu\in\mathbb{R}^J:0\leq\alpha_j\leq #1, u^T\Lambda u\leq #2\}}
\newcommand{\projV}{\mathcal{C}_V}
\newcommand{\defprojV}[1]{\projV(#1)=\{ v\in\mathbb{R}^P:\|v\|_0\leq #1,\| v\|_2=1\}}

\newcommand{\pCov}{\Sigma^\star}
\newcommand{\pA}{A^\star}
\newcommand{\pa}{a^\star}
\newcommand{\pZ}{Z^\star}

\newcommand{\suppU}{\mathcal{S}_U}
\newcommand{\suppV}{\mathcal{S}_V}
\newcommand{\defsuppV}[2]{\suppV(#1,#2)=\{ \mathcal{S}_{V_k} \in \Omega(#2,P)\}_{k\in[#1]}}

\newcommand{\iniC}{I_0}
\newcommand{\graell}{\nabla\ell_{N,j}(\Sigma_j)}
\newcommand{\graellstar}{\nabla\ell_{N,j}(\pCov_j)}

\newcommand{\rkG}[1]{r(#1)}
\newcommand{\sampleCov}[1]{S_{N,#1}}
\newcommand{\sumsampleCov}{M_N}
\newcommand{\sumpCov}{M^\star}
\newcommand{\defsumsampleCov}{\sumsampleCov=J^{-1}\sum_{j=1}^J\sampleCov{j}}
\newcommand{\defsumpCov}{\sumpCov = E(\sumsampleCov)}
\newcommand{\rotmat}{R}
\newcommand{\defR}{\rotmat=\argmin_{Y\in\mathcal{O}(K)}\|V-V^\star Y\|_F^2}
\newcommand{\defRo}{\rotmat^0=\argmin_{Y\in\mathcal{O}(K)}\|V^0-V^\star Y\|_F^2}
\newcommand{\defRplus}{\rotmat^+=\argmin_{Y\in\mathcal{O}(K)}\|V^+-V^\star Y\|_F^2}

\newcommand{\projtrnA}{\tilde{\mathcal{C}}_A}
\newcommand{\defprojtrnA}[2]{\projtrnA(#1,#2)=\{\alpha=\trnQ u:0\leq\alpha_j\leq #1, u^T\trnLam u\leq#2\}}
\newcommand{\projqrV}{\tilde{\mathcal{C}}_V}

\newcommand{\trnG}{\tilde{G}}
\newcommand{\deftrnGinv}{\trnG^\dagger=\trnQ\trnLam \trnQ^T}
\newcommand{\trnQ}{\tilde{Q}}
\newcommand{\trnLam}{\tilde{\Lambda}}

\newcommand{\ptrnZ}{\tilde{Z}^{\star}}
\newcommand{\ptrnCov}{\tilde{\Sigma}^\star}
\newcommand{\ptrnA}{\tilde{A}^\star}
\newcommand{\ptrna}{\tilde{a}^\star}

\newcommand{\trnerror}{\delta_A}
\newcommand{\deftrnerrorinq}{\trnerror\leq(16\gamma^\star)^{-1}\min_{j\in[J]}\singvtwo{K}{\pCov_j}}

\newcommand{\ellipsoid}[3]{\{\alpha=#1 u:u^{T}#2 u\leq#3\}}
\newcommand{\dset}[4]{\Upsilon(#1,#2,#3,#4)}
\newcommand{\spatialset}[2]{\Vcal(#1,#2)}
\newcommand{\defspatialset}[2]{\spatialset{#1}{#2}=\{V\in\mathbb{R}^{P\times #2}:\|v_k\|_2=1,\|V_{#1^c}\|=0,k\in[#2]\}}
\newcommand{\temporalset}[2]{\Tcal(#1,#2)}
\newcommand{\defvarphitwo}{\varphi^2=\max_{j\in[J]}\{1+{4\surd{2}\|A^\star\|_\infty}/{\sigma_K(\pCov_j)}\}}
\newcommand{\deftau}{\tau=J^{-1}\{9/2 + (1/2\vee K/8)\}}
\newcommand{\defmu}{\mu=\max_{j\in[J]}(17/8)\norm{\pZ_j}_2}
\newcommand{\defeigengap}{g =\sigma_K(M^\star)-\sigma_{K+1}(M^\star)>0}

\newcommand{\defchi}{\chi=4\beta^{1/2}(1-2\iniC/\surd{J})^{-2}(1+32\|\pA\|_\infty^2)}
\newcommand{\defbeta}{\beta=1-\eta/(4J\xi^2)}

\begin{document}

\begin{bibunit}[my-plainnat]

\title{A Nonconvex Framework for Structured  Dynamic Covariance Recovery}
\author[1]{Katherine Tsai}
\author[2]{Mladen Kolar}
\author[3]{Oluwasanmi Koyejo}
\affil[1]{Department of Electrical and Computer Engineering, University of Illinois at Urbana-Champaign}
\affil[2]{Booth School of Business, The University of Chicago}
\affil[3]{Department of Computer Science, Beckman Institute for Advanced Science and
Technology, and Statistics, University of Illinois at Urbana-Champaign}
\date{}
\maketitle

\begin{abstract}
We propose a flexible yet interpretable model for high-dimensional data
with time-varying second-order statistics, motivated and applied to functional neuroimaging data. Our approach implements the neuroscientific hypothesis of discrete cognitive processes by factorizing the covariances into sparse spatial and smooth temporal components. While this factorization results in parsimony and domain interpretability, the resulting estimation problem is nonconvex. 
We design a two-stage optimization scheme with a tailored spectral initialization, combined with iteratively refined alternating projected gradient descent. We prove a linear convergence rate up to a 
nontrivial statistical error for the proposed descent scheme and establish sample complexity guarantees for the estimator. 
Empirical results using simulated data and brain imaging data illustrate that our approach outperforms existing baselines. 
\end{abstract}

\noindent {\bf Keywords: }Dynamic covariance; Structured factor model; Alternating projected gradient descent; Time-series data; Functional connectivity.

\section{Introduction}

The manuscript proposes and evaluates a model for dynamic functional brain network connectivity, defined as the time-varying covariance of associations between brain regions~\citep{fox2007spontaneous}. Understanding the variation of brain connectivity between individuals is believed to be a crucial step towards uncovering the mechanisms of neural information processing~\citep{sakouglu2010method,chang2016tracking}, with potentially transformative applications to understanding and treating neurological and neuropsychiatric 
disorders~\citep{calhoun2014chronnectome}. 

In the neuroscience literature, estimators for time-varying covariances range from sliding window methods to hidden Markov models. The commonly used sliding window sample covariance estimator is computationally efficient~\citep{preti2017dynamic}. However, 
this estimate is sensitive to the selected window length, and 
spurious correlations may occur when the underlying window 
length is misspecified~\citep{leonardi2015spurious}. Discrete-state hidden Markov models construct interpretable estimates of brain connectivity in terms of recurring connectivity patterns~\citep{vidaurre2017brain}, yet they fail to capture the smooth nature of brain dynamics~\citep{shine2016dynamics, shine2016temporal}. These shortcomings motivate a new approach. Specifically, our proposed approach implements the neuroscientific hypothesis that brain functions are interactions between cognitive processes~\citep{posner1988a}, which we model as weighted combinations of low-rank components~\citep{andersen2018bayesian}. Beyond the neuroscientific underpinnings, high-dimensional data often has a low dimensional representation~\citep{udell2019big},  and low rank can help prevent overfitting~\citep{udell2016generalized}. Specifically, we propose a smooth, structured low-rank time-varying covariance model inspired by the observed sparsity of brain factors~\citep{eavani2012sparse}, and temporal dynamics of brain activity~\citep{shine2016dynamics, shine2016temporal}. Hence, we constrain the temporal components to be smoothly varying via projection to a temporal kernel and restrict the sparsity of the spatial components via hard-thresholding, respectively.

We estimate parameters of 
the resulting model 
using a first-order optimization scheme
that
is analogous to a Burer-Monteiro factorization~\citep{burer2003nonlinear,burer2005local}. 
While the first-order approach reduces the computational complexity 
as compared to semidefinite programming, the resulting optimization program is nonconvex, 
and special care is needed to design and analyze an 
optimization scheme that avoids converging to bad local optima.
To this end, we build on the growing literature studying
matrix estimation problems~\citep{candes2015phase, chi2019nonconvex} using a two-stage algorithm.
First, spectral initialization is used to find an initial point lying
within a local region, where the objective satisfies local regularity conditions. 
Next, projected gradient descent is used to refine
the estimate and find a stationary point of the objective.

In summary, our contributions include a novel dynamic covariance 
model motivated by neuroscientific models of functional brain connectivity networks. 
We provide an efficient procedure for estimation, 
along with the convergence analysis and sample complexity. 
Specifically, under the assumption that spatial components are shared across time,
we develop a structured spectral initialization method, 
which effectively uses the available samples and provides a 
better spatial estimate than separate initialization per individual. 
We prove linear convergence of the factored gradient method to an
estimate with a nontrivial statistical error
and provide a non-asymptotic bound on the statistical error
when data are Gaussian.
Experiments show that the model successfully recovers 
temporal smoothness and detects temporal changes induced by task activation.

\section{Background}
\subsection{Notation}
The inner product of two matrices is denoted as
$\langle{X},{Y}\rangle=\tr({X}^T{Y})$.  
For a matrix ${X}$,
${\sigma}_k({X})$ denotes the $k$th largest singular value, 
$\|{X}\|_F^2={\tr({ X}^T{ X})}$ denotes the Frobenius norm, 
$\|{ X}\|_2=\sigma_1({ X})$ denotes the spectral norm, and
$\|X\|_\infty=\max_{i,j}|X_{i,j}|$ denotes the max norm.
For two symmetric matrices $X$ and $Y$, $X\preceq Y$ denotes $Y-X$ is positive semi-definite.
The pseudoinverse of $X$ is denoted $X^\dagger$.
The set of $K\times K$ rotation matrices is denoted as $\mathcal{O}(K)$.  
We use $\kappa(\cdot,\cdot)$ to denote a positive-definite kernel function. 
The function $\diag:\mathbb{R}^{K}\rightarrow\mathbb{R}^{K\times K}$ converts
a $K$-dimensional vector to a $K\times K$ diagonal matrix. 
For scalars $a$ and b, 
$a\vee b$ denotes $\max(a,b)$ and 
$a \wedge b$ denotes $\min(a,b)$. 
We use $a\gtrsim b$ ($a \lesssim b$) to 
denote that there exists a constant $C>0$ such that $a\geq Cb$ ($a \leq Cb$). We use $a\asymp b$ to denote $a\gtrsim b$ and $a\lesssim b$. 
We use $[J]$ to denote the index set $\{1,\ldots,J\}$.

\subsection{Problem Statement}\label{subsec:problem_statement}

Given samples from $N$ subjects recorded at $J$ time points, 
denoted
${ x}^{(n)}_j\in\mathbb{R}^P$, $n\in[N]$, $j \in[J]$, 
let $\sampleCov{j}=N^{-1}\sum_{n=1}^N{x}_j^{(n)}{x}_j^{(n) T}$ 
be the sample covariance across subjects at time $j$. 
We assume the population covariance takes a factorized form as
\begin{equation} \label{eq:model}
E(\sampleCov{j})
={\Sigma}_j^\star+{E}_j
={V}^\star\diag({a}_j^\star){V}^{\star T}+{ E}_j,\quad j\in[J],
\end{equation}
where ${\Sigma}_j^\star$ is at most rank $K$ 
and ${E}_j$ is a noise matrix such that the largest singular 
value of $E_j$ is strictly smaller than the smallest nonzero singular value of $\pCov_j$. 
This factorization employs time-invariant and columnwise orthonormal spatial components
${V}^\star=({v}_1^\star,\ldots, {v}_K^\star)\in\mathbb{R}^{P\times K}$ that are the top-$K$ 
eigenvectors of $\{E(\sampleCov{j})\}_{j\in[J]}$. Analogously, 
${A}^\star=({a}_1^\star,\ldots,{a}_J^\star)\in\mathbb{R}^{K\times J}$ 
represents the temporal components. 
To facilitate estimation in a high-dimensional setting, we further 
assume that the columns of ${V}^\star$ are sparse
and belong to 
$\defprojV{s^\star}$. 
The rows of ${ A}^\star$, denoted as $A^\star_{k\cdot}$, $k\in[K]$, 
are smooth, bounded, and belong
to 
$\defprojA{c^\star}{\gamma^\star}$,
where
$G = (G_{x,y}=\kappa(x,y))_{x,y\in[J]} \in \mathbb{R}^{J\times J}$ 
is a positive semi-definite kernel matrix,
the kernel $\kappa$ is known as a priori,
and 
$G^\dagger=Q\Lambda Q^T$ is the eigendecomposition of $G^\dagger$.
The kernel $\kappa$ is used to model temporal smoothness 
of the rows of ${ A}^\star$
and the box constraint ensures that 
$\alpha_j\geq0$, so the covariance model is 
positive semi-definite, 
and is upper bounded by a positive constant for $j\in[J]$.

Eigenvalues of the kernel matrix $G$ may decay quickly, which
may result in numerically unstable algorithms when projecting onto 
the set $\mathcal{C}_A$. For example, 
eigenvalues of a kernel matrix corresponding to the 
Sobolev kernel decay at a polynomial rate, 
while for the Gaussian kernel 
they decay at an exponential-polynomial rate~\citep{scholkopf2002learning}.
Instead of working with the kernel matrix $G$,
we are going to construct a low-rank 
approximation, $\trnG$, of $G$ by truncating small eigenvalues.
Write $Q=(\trnQ, Q_1)$, where the columns of $\trnQ$ 
are eigenvectors of $G$ corresponding to eigenvalues 
greater or equal to $\trnerror$,
and $\trnLam^{-1} = \diag ( \Lambda_{jj}^{-1} \geq \trnerror \mid j \in [J] )$.
Then $\deftrnGinv$.
We define $\defprojtrnA{c}{\gamma}$ and the rank of $\trnG$ is denoted as $\rkG{\trnG}$. 

Under the model \eqref{eq:model}, 
we estimate the parameters ${Z^\star = ({V}^{\star T},{A}^\star)^T}$ by minimizing the following objective
\begin{align}
    \min_{Z} f_N({ Z}) 
    = \min_{
    \substack{ {v_k} \in \mathcal{C}_V(s),\;k\in[K] \\ {\rowkA} \in \projtrnA(c, \gamma),\;k\in[K]}
    }
    \frac{1}{J}\sum_{j=1}^J\frac{1}{2}\|{ S}_{N,j}-{ V}\diag({ a}_j){ V}^T\|_F^2,\label{eq:obj}
\end{align} 
where ${\rowkA}$ is the $k$th row of $A$.
Although $f_N$ is nonconvex with respect to $Z=(V^T, A)^T$,
the corresponding covariance loss $\ell_{N,j}({\Sigma}_j)=\frac{1}{2}\|{ S}_{N,j}-{\Sigma}_j\|_F^2$
is $m$-strongly convex and $L$-smooth
with $m=L=1$ \citep{nesterov2013introductory}. We use alternating projected gradient
descent to update $V$ and $A$.
The selection of tuning parameters of $\mathcal{C}_V$ and $\projtrnA$ 
is discussed in Section~\ref{subsec:stp}.

\subsection{Related work}\label{sec:background}

Dynamic covariance models are common for analyzing time-series data in applications ranging from computational finance and economics~\citep{engle2019large}  
to epidemiology~\citep{fox2015bayesian} and 
neuroscience~\citep{foti2019statistical}. Factor models are among the most popular analysis approaches, some of which encode temporal structure using latent kernel regularization~\citep{paciorek2003nonstationary,kastner2017efficient}. 
For instance, \cite{andersen2018bayesian} encoded smooth temporal dynamics by introducing a latent Gaussian process prior. \cite{li2019multivariate} also used piecewise Gaussian process factors to capture the combinations of gradual and abrupt changes.
For spatial structure in factor models, \citet{kolar2010estimating} and
\citet{danaher2014joint} implemented variants of group lasso and fused lasso to 
impose sparsity. Along similar lines, our approach implements temporal and spatial structure through projection onto suitable constraint sets. 

Our work is also related to dictionary learning~\citep{olshausen1997sparse, mairal2010online},
which can be viewed as a type of factorization where the signal is 
decomposed into atoms and coefficients. In such a factorization, 
sparsity is controlled through a sparse penalty on the coefficients. 
\citet{mishne2019learning} extended this approach to encode temporal 
data by constructing time-trace atoms with spatial coefficients. 
In comparison, our model has shared spatial structure and individual temporal structure. 

Autoregressive models have also been applied to model 
dynamic connectivity in fMRI~\citep{qiu2016joint,liegeois2019resting}.
Although autoregressive models employ different modeling 
assumptions from ours, they can capture smooth temporal dynamics of signals. 
However, the forecasts of autoregressive models can become 
unreliable in high-dimensional settings~\citep{banbura2010large}. To this end, 
various implementations of structured
transition matrices~\citep{davis2016sparse,ahelegbey2016bayesian, SKRIPNIKOV2019164} have
been proposed and shown to improve computational efficiency and prediction accuracy.

The optimization problem in \eqref{eq:obj} is nonconvex and is 
optimized by alternating minimization. Recent literature has 
established a linear convergence rate to global optima~\citep{jain2013low, hardt2014understanding,gu2016low}. In particular, 
our work builds on~\citet{bhojanapalli2016dropping}, 
who showed linear convergence in $V$ when the underlying 
objective function is strongly convex with
respect to $X=VV^T$. Subsequently \citet{park2018finding} 
and \citet{yu2020recovery} proved a linear convergence rate for
non-symmetric matrices. Unlike previous work, our 
factorization scheme $V\diag(a_j)V^T$ imposes
additional structure on the eigenvalues, 
thus having potential applications in regularizing graph-structured
models~\citep{kumar2020unified}.

In nonconvex optimization, finding a good initialization in
a local region is often useful to avoid convergence 
to bad local optima (e.g., $Z=0$ is a trivial stationary point in our model). 
Spectral methods are typically employed 
for this task as they have good consistency 
properties~\citep{chen2015solving}. We employ a problem-specific
spectral approach to develop a novel initialization method.
Post-initialization, a first-order gradient descent
method is sufficient to ensure convergence to desired optima~\citep{candes2015phase}. 
Combining with the structured constraints, 
\citet{chen2015fast} provided a theoretical framework 
for projected gradient descent method onto convex constraint sets.
In our work, we are projecting onto a nonconvex set,
which might increase 
the distance $\|V-V^\star R\|_F^2$. Therefore, we need 
a problem-specific analysis to quantify the 
expansion coefficient.

\section{Methodology}

\subsection{Two-stage algorithm}

We develop a two-stage algorithm for solving the optimization problem in~\eqref{eq:obj}.
As the objective is nonconvex, a local iterative procedure
may converge to bad local optima or saddle points.
In the first stage of the algorithm, spectral 
decomposition is used to find an initialization point.
In the second stage, projected gradient descent 
is used to locally refine the initial estimate and
find  a stationary point that is within the 
statistical error of the population parameters.
Algorithm~\ref{alg:spectral_ini} summarizes our initialization procedure. Here,  
the eigendecomposition of \{${ S}_{N,j}\}_{j\in[J]}$
is performed to obtain initial estimates of ${ V}^\star$ and ${ A}^\star$.
Specifically, the initialization uses the shared spatial structure of
$\{{\Sigma}^\star_j\}_{j\in[J]}$ to increase the effective sample size, 
i.e., the initial estimate ${ V}^0$ is obtained from the eigenvectors
corresponding to the largest $K$ eigenvalues of the covariance matrix pooled across time, 
$\defsumsampleCov$.
The initial estimate of the temporal coefficients, ${ A}^0$, is obtained
by projecting $\{{ S}_{N,j}\}_{j\in[J]}$ onto ${ V}^0$. 

\begin{algorithm}
\caption{Spectral initialization} \label{alg:spectral_ini}
\begin{tabbing}
   \qquad \enspace Set ${ M}_{N}=(NJ)^{-1}\sum_{j=1}^J\sum_{n=1}^N{ x}_j^{(n)}{ x}_j^{(n) T}$\\
   \qquad \enspace Set ${ V}^0=({ v}_1^0,{ v}_2^0,\ldots,{ v}_k^0)\leftarrow\text{ top $K$ eigenvectors of ${M}_N$}$\\
   \qquad \enspace For $j=1$ to $j=J$ and $k=1$ to $k=K$ \\
    \qquad \qquad  $a_{k,j}^0\leftarrow { v}_k^{0 T}{ S}_{N,j}{ v}_k^0$\\
\qquad \enspace Set ${ A}^0=(a_{k,j}^0)_{k\in[K],j\in[J]}$\\
\qquad \enspace Output ${ V}^0$, ${ A}^0$
\end{tabbing}
\end{algorithm}

After initialization, we iteratively refine estimates of ${ V}$ and ${A}$
via alternating projected gradient descent.
In each iteration, the iterates $V$ and $A$
are updated using the gradient of $f_N$, where 
$\eta$ denotes the step size.
Note that we scale down the step size for the ${V}$ update by
$J$ to balance the magnitude of the gradient. 
After a gradient update, we project the iterates onto the constraint 
sets $\mathcal{C}_{ V}$ and $\projtrnA$ to enforce sparsity on $V$ and
smoothness on $A$.
Details are given in Algorithm~\ref{alg:main}.

\begin{algorithm}
\caption{Dynamic covariance estimation} \label{alg:main}
\begin{tabbing}
   \quad \enspace Set ${ V}^0, { A}^0=\text{Spectral initialization}(\{{ x}_j^{(n)}\}_{n\in[N],j\in[J]})$\\
   \quad\enspace While $|f_N({ Z}^{i-1})-f_N({ Z}^{i-2})|>\varepsilon$ \\
   \quad\quad $\widehat{ A}^i\leftarrow{ A}^{i-1}-\eta\nabla_{ A}f_N({ Z}^{i-1})$\\
   \quad\quad ${ A}^i\leftarrow\text{Project rows of }\widehat{ A}^i\text{ to }\projtrnA$\\
   \quad\quad $\widehat{ V}^i\leftarrow{ V}^{i-1}-\frac{\eta}{J}\nabla_{ V}f_N({ Z}^{i-1})$\\
   \quad\quad ${ V}^i\leftarrow\text{Project columns of }\widehat{ V}^i\text{ to }\mathcal{C}_{ V}$\\
\quad \enspace Output ${ V}$, ${ A}$
\end{tabbing}
\end{algorithm}


Although $\mathcal{C}_{ V}$ is a nonconvex set, projection onto this set 
can be computed efficiently by picking the top-$s$ largest 
entries in magnitude and then projecting the constructed 
vector to the unit sphere. Despite projecting onto a nonconvex set,
we are able to show that the gradient and projection step
jointly result in a contraction (see Supplementary Material).
On the other hand, the projection onto 
the convex set $\projtrnA$ can be computed efficiently via convex programming: we 
project onto $\projtrnA$ by iteratively projecting 
onto $\{\alpha\in\mathbb{R}^J:0\leq\alpha_j\leq c, j\in[J]\}$
and $\ellipsoid{\trnQ}{\trnLam}{\gamma}$,
which gives us a point in the intersection of the sets
by von-Neumann's theorem~\citep{escalante2011alternating}.

\subsection{Selection of tuning parameters}\label{subsec:stp}

The parameters of the proposed model include the sparsity level $s$,
the rank $K$, the kernel length scale $l$,
the smoothness coefficient $\gamma$, the truncation level $\trnerror$, and the limits of the projected
upper bound of the box $c$. For some kernels (e.g., Gaussian kernel, Mat\'ern five-half kernel,
and other radial basis function kernels), one must also select 
the length scale parameter $l$,
which captures the smoothness of the curves (i.e., $\{\rowkA^\star\}_{k\in[K]}$); For example, a Gaussian kernel function is $\kappa_l(x,y)=\sigma^2\exp\{-(x-y)^2/(2l^2)\}$, where $l$ affects the slope of the eigenvalues decay. We denote such kernel functions as $\kappa_l$ rather than $\kappa$. 
Our theory suggests that $\trnerror$ should be upper bounded by the magnitude of $\min_{j\in[J]}\singvtwo{K}{\pCov_j}$ to obtain good statistical error. 
Further, $\trnerror$ is selected for numerical stability. In experiments, we find that $\trnerror=10^{-5}$ is a good empirical choice, and satisfies the sufficient conditions.
In principle, we do not want to cut off any important signals, 
so we choose $c$ as a value greater than $\max_{j\in[J]}\|\sampleCov{j}\|_2$ and $c^\star = \max_{j\in[J]} {\|\pCov_j\|_2}$.
In terms of the estimation performance, we observe that the selection of sparsity 
and rank have a larger effect than the selection of $\gamma$ and $l$. While under-selection of $s$ and $K$ leads
to poor evaluation scores, improper selection of $l$ 
and $\gamma$ have relatively minor influence. Hence, we adopt a two-stage
approach to selecting parameters. In the first stage, we
perform grid search on $s$, $K$, $\gamma$, $l$ and find the configuration that minimizes 
the Bayesian information criterion
$ \BIC=\log N\sum_{k=1}^K\|v_k\|_0-2\widehat{L}_N$, where $\widehat{L}_N$ is the maximized Gaussian log-likelihood function. However, varying $\gamma$ and $l$ 
have subtle influence on $\BIC$. Consequently, in the second stage,
we fix $s$, $K$ with values selected in the first stage and select $\gamma$ and $l$ using
$5$-fold cross-validation with the Gaussian log-likelihood, 
which is motivated by prior work on nonparametric dynamic
covariances~\citep{yin2010nonparametric, kar68837}. Empirically,
we find that tuning the length scale parameter $l$ is more effective 
than tuning $\gamma$ in producing globally smooth temporal structures (see Supplementary Material). 

\section{Theory}\label{sec:theory}

\subsection{Preliminaries}

Before presenting our main theoretical results, we introduce two tools 
that will help us establish the results.

First, we discuss orthogonalization.
The spatial component $V$ produced by Algorithm~\ref{alg:main} is not necessarily orthonormal. 
However, $V^\star$ is full rank and
if $\min_{Y\in\mathcal{O}(K)}\|V-V^\star Y\|_2^2<1$ 
is guaranteed
at each iteration, 
then $V$ is full rank as well.
As a result,
the subspace spanned by columns of $V$ is equal
to the subspace spanned by columns of the orthogonalized version of it. 
To simplify the analysis of Algorithm~\ref{alg:main},
we add a QR decomposition step that orthogonalizes $V$ 
after the projection onto $\mathcal{C}_{ V}$.
That is, in each iteration we compute 
\begin{align*}
    {V}^{i}_{ortho}\leftarrow V^i({L}^{i})^{-1}\quad\text{(QR decomposition)},
\end{align*}
where $L^i$ is the upper triangular matrix, with diagonal entries less or equal to $1$.
Note that orthogonalization of $V$ 
in each iteration of Algorithm~\ref{alg:main}
is not needed in practice and is only used in
establishing theoretical properties. 
Such an approach is commonly used in the literature~\citep{jain2013low,zhao2015nonconvex}.  
We further note that an addition of the 
QR decomposition only increases the distance 
of the iterate $V^i$ to $V^\star\rotmat$ 
by a mild constant~\citep{stewart1977perturbation, zhao2015nonconvex} (see Supplementary Material).
Furthermore, QR decomposition increases the number of nonzero elements of
the iterate $V$ to at most $Ks$. 
As we consider the rank $K$ to be fixed and $P\gtrsim s$, the effect of the QR decomposition is mild.
Our experiments further demonstrate that optimization with 
and without the QR decomposition step result in comparable performance.

Next we introduce the notion 
of the  statistical error,
which allows us to 
quantify the distance of the population parameters from the stationary point
to which the optimization algorithm 
converges. Note that the notion of statistical error has
been previously adopted in M-estimation~\citep{loh2015regularized}. 
Let ${\cal B}_t = \{ v \in \RR^P \mid \|{ v}\|_0\leq t, \|v\|_2\leq1\}$
and
\[
\dset{r}{t}{h}{\trnerror}
=
\{\{\Delta_j=V\diag(a_j)W^T\}_{j\in J} \mid
v_k \in {\cal B}_t, 
w_k \in {\cal B}_t,
\rowkA^T\trnG^\dagger \rowkA\leq h, k\in[r]\},
\]
where $\trnG$ is the truncation of $G$ at the level of $\trnerror$.
We define the statistical error as
\begin{equation*}
    \staterr = 
    \staterr(2K,2s+s^\star,2\gamma,\trnerror)=
    \max_{\{\Delta_j\}_{j\in[J]}\in\dset{2K}{2s+s^\star}{2\gamma}{\trnerror}}\frac{\sum_{j=1}^J\langle\nabla\ell_{N,j}({\Sigma}_j^\star),\Delta_j\rangle}{\rbr{\sum_{j=1}^J\|\Delta_j\|_F^2}^{1/2}}.
\end{equation*}
The statistical error describes the geometric landscape around the optimum---it 
quantifies the magnitude of gradient of the
empirical loss function evaluated at 
the population parameter in the directions 
constrained to the set $\Upsilon$.

\subsection{Assumptions and Main Results}

We begin by stating the assumptions needed to establish 
the main results. Note that $V$ in this section
is used to denote an iterate in 
after the QR factorization step.

An upper bound on the step size is required 
for convergence of Algorithm~\ref{alg:main}.
Let $Z_j^{0}=(V^{0T},\diag(a_j^0))^T$, $j\in[J]$,
denote the output of Algorithm~\ref{alg:spectral_ini}.
\begin{assumption}
\label{assumption_stepsize}
The step size satisfies 
$\eta\leq\min_{j\in[J]}\ J^{1/2}/(64\|{Z}_j^0\|_2^2)$.
\end{assumption}
Note that the step size depends on the initial estimate,
but remains constant throughout the iterations. 
Let $\defbeta<1$, $\defchi$, and $\deftau$, where
\begin{equation}\label{eq:I0xi2}
    \iniC^2=\left\{\frac{1}{16\xi^2}\frac{1}{(1+\|\pA\|_\infty^2J^{-1})}\wedge \frac{J}{4}\right\},
    \quad \xi^2=\max_{j\in[J]}\left\{\frac{16}{\sigma_K^2(\pCov_j)}+\left(1+\frac{8c}{\sigma_K(\pCov_j)}\right)^2\right\}.
\end{equation}
We also require the tuning parameters to be selected appropriately.
\begin{assumption}\label{assumption_para}
We have
$c\geq c^\star$, $\gamma\geq\gamma^\star$, $s\geq [\{4(1/\chi-1)^{-2} +1\} \vee 2]s^\star$. The matrix 
$\trnG$ is obtained with  the truncation level
$\deftrnerrorinq$.
\end{assumption}
Note that the condition on $\trnerror$ is mild.
It guarantees that we do not truncate too much of the signal. 
Finally, we require an assumption on the statistical error.
\begin{assumption}\label{assumption_statcondition}
We have
$\staterr^2\leq J\iniC^2\{{(\beta^{1/2}-\beta)}/{(\tau\eta)}\wedge\min_{j\in[J]}3\|\pZ_j\|_2^2\}$.
\end{assumption}
Assumption~\ref{assumption_statcondition} is essentially a requirement on the sample size $N$,
since for a large enough $N$ the assumption will be satisfied with high probability.
Notice that as the sample size increases, the statistical error gets smaller, while
the radius of the local region of convergence, $\iniC$, stays constant.
Furthermore, if Assumption~\ref{assumption_statcondition} is not satisfied, this implies that 
the initialization point is already close enough to the population parameters
and the subsequent refinement by Algorithm~\ref{alg:main} is not needed.

With these assumptions, we are ready to state the main result,
which tells us how far are the estimates obtained
by Algorithm~\ref{alg:spectral_ini} and~\ref{alg:main}
from the population parameters. Let
$\Sigma_j^I = V^I\diag(a_j^I)(V^{I})^T$, $j\in[J]$, denote the estimate of the covariance 
at the $I$th iteration.

\begin{theorem}\label{theorem:combine12} 
Suppose Assumption~\ref{assumption_stepsize}--\ref{assumption_statcondition}
are satisfied and $J\geq4$.
Furthermore, 
for a sufficiently large constant $C_0$,
suppose that we are given 
$N=C_0KP\log(PJ/\delta_0)$
independent samples such that $\|x_j^{(n)}\|_2^2\leq P\norm{\pA}_\infty$ 
almost surely, $j\in[j]$, with zero mean and covariance as in \eqref{eq:model}.
Then, with probability at least $1-\delta_0$,
the estimate 
obtained by Algorithm~\ref{alg:spectral_ini} and Algorithm~\ref{alg:main}
satisfies
\begin{align}
\sum_{j=1}^J\norm{\Sigma_j^I-\pCov_j}_F^2\leq\beta^{I/2}(4\mu^2\xi^2)\sum_{j=1}^J\norm{\Sigma_j^0-\pCov_j}_F^2+\frac{2\tau\mu^2\eta}{\beta^{1/2}-\beta}\staterr^2+2K\gamma^\star\trnerror,\label{eq:main1}
\end{align}
where $\defmu$.
\end{theorem}
The first term on the right hand side of~\eqref{eq:main1}
corresponds to the optimization error
and we observe a linear rate of convergence.
The second and third term of~\eqref{eq:main1}
correspond to the statistical error and 
approximation error due to the truncation of the kernel matrix,
respectively. 
From the bound,
we observe a trade-off between $\staterr$ and 
the truncation error $\trnerror$:
if $\trnerror$
is decreased, 
$\staterr$ increases.

The proof of Theorem~\ref{theorem:combine12} is given in 
two steps. First, we establish the convergence rate of 
iterates obtained by Algorithm~\ref{alg:main} 
by first assuming that $V^0$ and $A^0$ lie in
a neighborhood around $V^\star$ and $A^\star$ (see $\S$\ref{sec:lin_conv}). 
Subsequently, we show in Theorem~\ref{theorem:lineardist} 
that Algorithm~\ref{alg:spectral_ini} provides suitable
$V^0$ and $A^0$ with high probability (see $\S$\ref{Subsec:sc}). 

To give an example of Theorem~\ref{theorem:combine12}, we consider
the case where data are
generated from a multivariate Gaussian distribution and for a Gaussian kernel.

\begin{proposition}\label{prop:gaussian_example} 
Let $x_j^{(n)}\in\RR^P$ 
be independent Gaussian samples
with mean zero and covariance as in \eqref{eq:model} 
with $J\geq 4$ and $N \gtrsim K(P+\log J/\delta_0)$.
Suppose $G$ is
a Gaussian kernel matrix whose
eigenvalue decays at the rate $\exp(-l^2 j^2)$
for some length-scale $l>0$.
Let $\trnerror\asymp(\gamma^\star l N)^{-1}\{\log(\gamma^\star l N)\}^{1/2}$. 
Suppose that
Assumption~\ref{assumption_stepsize}--\ref{assumption_para} hold, $s^\star\log(P/s^\star)< PJ$ and $\max_{j\in[J]}\|E_j\|_2\lesssim \iniC$.
Then after $I \gtrsim \log(1/\delta_1)$ iterations
of Algorithm~\ref{alg:main},
with probability at least $1-\delta_0$,
we have
\begin{align*}
\sum_{j=1}^J\norm{\Sigma_j^I-\pCov_j}_F^2
\lesssim \delta_1 
+ 
   \frac{1}{N}\sbr{K\cbr{\frac{1}{l}\rbr{\log \gamma^\star l N}^{1/2}+s^\star\log\frac{P}{s^\star}}
   + \log\delta_0^{-1}} 
   +
   J\max_{j\in[J]}\|E_j\|_2^2.
\end{align*}
\end{proposition}
The condition 
$s^\star\log(P/s^\star)< PJ$ is mild, since $s^\star\lesssim P$,
while the condition 
$\max_{j\in[J]}\|E_j\|_2\lesssim \iniC$
is mild, since
$\|E_j\|_2<\sigma_K(\Sigma_j^\star)$, $j\in[J]$. 
Under the Gaussian distribution, 
the sample complexity is improved to
$N\gtrsim K(P+\log J)$ from $N\gtrsim KP(\log P+\log J)$ 
in Theorem~\ref{theorem:combine12}. 
Proposition~\ref{prop:gaussian_example} 
provides an explicit bound on the estimator 
that can be obtained under an assumption 
on the eigenvalue decay. 
The statistical error is comprised of two terms
that correspond to errors when estimating 
smooth temporal components and sparse spatial components. 
In our choice of $\trnerror$, the truncation error is at the same 
order as the statistical error induced by the smooth temporal components.

\subsection{Linear Convergence}
\label{sec:lin_conv}

We establish the linear rate of convergence of Algorithm~\ref{alg:main}
when it is appropriately initialized. Recall that rows of $A^\star$ belong 
to $\projA(c^\star,\gamma^\star)\subseteq\projA(c,\gamma)$,
while the projected gradient descent is implemented on the set 
$\projtrnA(c,\gamma)\subset\projA(c,\gamma)$.
Let 
$$\ptrnA=\underset{B_{k\cdot}\in\projtrnA(c,\gamma),k\in[K]}{\argmin}\norm{B-\pA}_F^2,$$
be the best approximation of $\pA$ in $\projtrnA(c,\gamma)$.
See Supplementary Material for details on the construction of $\ptrnA$.
We define $\ptrnCov_j=V^\star\diag(\ptrna_j)V^{\star T}$, $j\in[J]$,
and $\tilde{Z}^{\star T}=(V^{\star T},\ptrnA)$. 
With these definitions, we establish the linear rate of convergence 
of the iterates to $\ptrnCov_j$ and $\tilde{Z}^{\star}$.
The convergence rate in Theorem~\ref{theorem:combine12}
will then follow by combining the results 
with the truncation error.

Observe that the covariance factorization is not unique, 
since, for any ${R} \in \mathcal{O}(K)$, we have
${\Sigma}_j={ V}\diag({a}_j){ V}^T = {V}{R_j}{R_j}^T\diag({a}_j){R_j}{R_j}^T{ V}^T$, $j\in[J]$.
By triangle inequality, we have
\begin{align}\label{eq:cov2dist}
    \sum_{j=1}^J\|\Sigma_j-\ptrnCov_j\|_F^2\leq \sum_{j=1}^J&\alpha_{V,j}\|V-V^\star R\|_F^2+\alpha_{A}\|\diag(a_j)-R^T\diag(\ptrna_j)R\|_F^2,
\end{align}
where $\alpha_{V,j}=3\{\|V\diag(a_j)\|_2^2+\|V^\star\diag(\ptrna_j)\|_2^2\}$, $j\in[J]$,
and $\alpha_{A}=3\|V^\star\|_2^2\|V\|_2^2$. 
This implies that if $\|V-V^\star R\|_F^2+\|\diag(a_j)-R^T\diag(\ptrna_j)R\|_F^2$ is small 
for some rotation matrix $R$ and every $j\in[J]$, 
then the left hand side will also be small. 
To this end, our goal is to show that the following distance 
metric contracts at each iterate of Algorithm~\ref{alg:main}. Let
\begin{align}
&\defR,\quad{\dist}^2({Z},\ptrnZ)
    =\sum_{j=1}^J d^2(Z_j,\ptrnZ_j)\label{eq:auxi_dist};\\
    &d^2(Z_j,\ptrnZ_j)=\|{ V}-{ V}^\star\rotmat\|_F^2+\|\diag({ a}_j)-\rotmat^T\diag(\ptrna_j)\rotmat\|_F^2,\notag
\end{align}
where $Z_j^T=(V^T,\diag(a_j))$ and $\tilde{Z}_j^{\star T}=(V^{\star T},\diag(\ptrna_j))$.
The metric first finds
the rotation matrix that aligns 
two subspaces and then computes 
the transformation of $\diag(\ptrna_j)$ along the rotation $\rotmat$. 
This metric is similar to the distance metric commonly
used in matrix factorization problems~\citep{anderson1956,ten1977orthogonal},
but in our model the choice of $R$ only depends on $V$. 


To show the convergence of ${\dist}^2(Z,\ptrnZ)$, we need following assumptions.
\begin{assumption}
\label{assumption_inicondition}
Suppose that ${Z}_j^0$ 
satisfies $\text{d}^2({ Z}_j^0,{ Z}_j^\star)\leq \iniC^2$, for $j\in[J]$, where $\iniC$ is defined in~\eqref{eq:I0xi2}.
Assume that $\|V^0-V^\star \rotmat\|_F^2\leq \iniC^2/J$ and 
$\|\diag(a_j^0)-\rotmat^T\diag(a_j^\star)\rotmat\|_F^2\leq (J-1)\iniC^2/J$.
\end{assumption}
Since $d^2(Z_j^0,\ptrnZ_j)\leq d^2(Z_j^0,\pZ_j)$ for $j\in[J]$, Assumption~\ref{assumption_inicondition} 
ensures that the distance of initial estimates 
and the population parameters are bounded within
the ball of radius $\iniC$. In addition,
$\iniC^2\leq J$ ensures that $\|V-V^\star R\|_2\leq 1$, so that $V$ is full-rank. 
Intuitively, we assume the squared distance for $V$ is $1/(J-1)$ times smaller 
than the squared distance for $A$,
because we have $J$ times more samples to estimate
$V$ compared to $A$.

\begin{theorem}\label{theorem:lineardist}
 Assume that
 Assumption~\ref{assumption_stepsize}--\ref{assumption_inicondition} hold. After $I$ iterations of Algorithm~\ref{alg:main}, we have
\begin{align}\label{eq:lemma1}
     {\dist}^2({ Z}^I,\ptrnZ) \leq
    \beta^{I/2}{\dist}^2({ Z^0},\ptrnZ)+\frac{\tau\eta\staterr^2}{\beta^{1/2}-\beta}.
\end{align}
\end{theorem}
The above result obtains a linear rate of convergence in  ${\dist}^2(Z,\ptrnZ)$. 
The second term on the left hand side denotes the constant multiple of 
the statistical error, which depends on the distribution of the data and the sample size. 
Combining with \eqref{eq:cov2dist} yields a linear rate of convergence in $\sum_{j=1}^J\|\Sigma_j-\ptrnCov_j\|_F^2$.

%
\subsection{Statistical Error}
Theorem~\ref{theorem:lineardist} shows linear convergence of the algorithm to
a region around population parameters 
characterized by the statistical error. One may wonder how large the
statistical error can be? While Assumption~\ref{assumption_statcondition} provides
a condition under which convergence is guaranteed, 
this bound is loose as it does not depend on the sample size. 
We establish a tighter bound under the Gaussian distribution.

\begin{proposition} [Statistical Error of Gaussian Distributed Data] \label{prop:SSE}
Suppose that samples $x_j^{(n)}\in\RR^P$
are Gaussian with mean zero and covariance as in \eqref{eq:model}.
Then, with probability at least $1-\delta$,
\[
\staterr(2K, (2m+1)s^*, 2m'\gamma^\star, \trnerror)
\leq (\nu\vee \nu^2)+\surd{J}\max_{j\in[J]}\|E_j\|_2,
\]
where
\begin{align*}
\nu = 
\frac{\|A^\star\|_\infty}{e_0}\left[\frac{1}{N}\left\{\log\frac{1}{\delta}+K\rkG{\trnG}+Ks^\star\log\frac{P}{s^\star}\right\}\right]^{\frac{1}{2}},
\end{align*}
$m,m'$ are positive integers, 
$e_0$ is an absolute constant depending on $m$ and $m'$,
and $\rkG{\trnG}$ is the rank of the $\trnerror$-truncated kernel matrix $\trnG$.
\end{proposition}



To interpret $\staterr$, 
the first term corresponds to the error in estimating the low-rank matrix,
while 
the second term corresponds to the essential error 
incurred from approximating the covariance 
matrix by a low-rank matrix. The low-rank matrix
can be estimated with the rate that converges to zero 
as
$[K\{\rkG{\trnG} + s^\star\log P\} / N]^{-1/2}$,
which corresponds to the rate of convergence of
temporal and spatial components. 
We also highlight that truncation of $G$ simplifies the statistical analysis because we can view the projection to $\projtrnA$ as restricting rows of $A$ to a subset of a $\rkG{\trnG}$-dimensional smooth subspace with $\rkG{\trnG}$ much smaller than $J$, the original dimension. 



\subsection{Sample Complexity of Spectral Initialization}\label{Subsec:sc}

We discuss the sample complexity required to satisfy Assumption~\ref{assumption_inicondition}.
That is, we characterize the sample size needed for 
Algorithm~\ref{alg:spectral_ini} 
to give a good initial estimate,
so that Algorithm~\ref{alg:main} outputs a solution characterized in
Theorem~\ref{theorem:lineardist}.
We consider a general case of a bounded distribution. 

\begin{theorem}[Sample Complexity of Spectral Initialization]\label{Theorem:SBSO}

Let ${ x}_j^{(n)} \in \mathbb{R}^{P}$  be independent zero mean 
samples with $\|x_j^{(n)}\|_2^2\leq{P\|A^\star\|_{\infty}}$ almost surely,
$n\in[N]$, $j\in[J]$, $J\geq 4$.
Let ${M}^\star=J^{-1}\sum_{j=1}^JE(\sampleCov{j})$ and $\defeigengap$
be the eigengap. Then, with probability at least $1-\delta$,
\begin{align}
    \label{eq:samplecomplexity}
      &\dist^2(Z^0,Z^\star)\leq \phi(g,A^\star)\cbr{ \frac{KJP^2}{N^2}\rbr{\log\frac{4JP}{\delta}}^2+\frac{KJP}{N} \log\frac{4JP}{\delta}};\\
      &\phi(g,A^\star)=4\|A^\star\|_{\infty}^2\cbr{\frac{5(1+16\varphi^2\|A^\star\|_\infty^2)}{g^2J}\vee 8\varphi^2},\notag
\end{align}
where $\defvarphitwo$. 
\end{theorem}

From~\eqref{eq:samplecomplexity} 
we note that if $N\gtrsim P\log (PJ/\delta)$,
then Assumption~\ref{assumption_inicondition} will be satisfied with high probability.
The eigengap $g$ must be greater than $0$ for the
bound in~\eqref{eq:samplecomplexity} to be nontrivial.
Moreover, since $g\leq\|A^\star\|_\infty$, 
the first term of $\phi(g,A^\star)$ dominates when $J$ is small.
Combining results from ~\eqref{eq:cov2dist}, Theorem~\ref{theorem:lineardist}
and Theorem~\ref{Theorem:SBSO}, we can establish Theorem~\ref{theorem:combine12}.



\section{Simulations}

We use the metric~\eqref{eq:auxi_dist} to evaluate recovery. We also compare results to other methods using the average log-Euclidean metric~\citep{arsigny2006log}.
Unless stated otherwise, we use the Mat\'ern five-half kernel~\citep{minasny2005matern} as the smoothing kernel 
for all the simulations. 
We evaluate the algorithm with a variety of temporal dynamics and compare with methods 
stated in Table~\ref{tab:Competing methods}. As for the data generation process, we create synthetic 
samples from the Gaussian distribution: 
${ x}_{j}^{(n)}\sim\mathcal{N}({ 0}, {\Sigma}_j^\star+\sigma{I})$,
$n\in[N]$, $j\in[J]$, where ${\Sigma}_j^\star=\sum_{k=1}^Ka_{k,j}^\star{ v}_k^\star{ v}_k^{\star T}$ 
and $\sigma{ I}$ is the additive noise.  
\begin{table}[ht]
  \caption{Competing methods}
  \label{tab:Competing methods}
    \fontsize{9pt}{9pt}\selectfont
  \centering
  \begin{tabular}{*{5}l}
    Abbr. & Model  & low-rank& smooth A& sparse V\\
    M1 & Sliding window principal component analysis   & \cmark& \cmark&\xmark\\
    M2 & Hidden Markov model & \xmark &\xmark&\xmark\\
    M3 & Autoregressive hidden Markov model~\citep{poritz1982linear} &\xmark&\cmark&\xmark\\
    M4 & Sparse dictionary learning~\citep{mairal2010online} & \cmark&\xmark&\cmark\\
    M5 & Bayesian structured learning~\citep{andersen2018bayesian} & \cmark&\cmark&\cmark\\
    M6 & Slinding window shrunk covariance~\citep{ledoit2004well}& \xmark& \cmark&\xmark\\
    M* & Spectral initialization (Algorithm~\ref{alg:spectral_ini})& \cmark&\xmark&\xmark\\
    M** & Proposed model (Algorithm~\ref{alg:main})& \cmark&\cmark&\cmark\\
    MQ** & Proposed model (Algorithm~\ref{alg:main}) with QR decomposition step& \cmark&\cmark&\cmark
  \end{tabular}
\end{table}

\subsection{Simulation of different temporal dynamics}

{ Ground truth recovery and linear convergence}: We demonstrate the algorithm in the 
noiseless setting and evaluate the recovery using the distance metric $\dist^2({ Z},{ Z}^\star)$.
The objective of this experiment is to evaluate the algorithm under different smooth temporal structures. All the tuning parameters are selected based on Section~\ref{subsec:stp}.
The ground truth and results are shown in Figure~\ref{fig:simulation1}. The top row shows the first
setting of mixing temporal weights, where we have sine functions, a constant function, and a ramp function.
The bottom row shows the second setting of different sine functions. On the right side of
Figure~\ref{fig:simulation1}, we plot the distance metric ${\dist}^2({Z},{Z}^\star)$
with different number of subjects $N=\{1,5,15,200\}$. For each trial, we see the linear 
convergence of the distance metric up to some statistical error,
and the error decreases with the increase of sample size, as predicted by 
Theorem~\ref{theorem:lineardist}. Moreover, the statistical error is consistent with 
the number of subjects. More simulation results for different temporal structures
are presented in the Supplementary Material. 

\begin{figure}
    \centering
    \includegraphics[width=4cm,height=2.7cm,keepaspectratio]{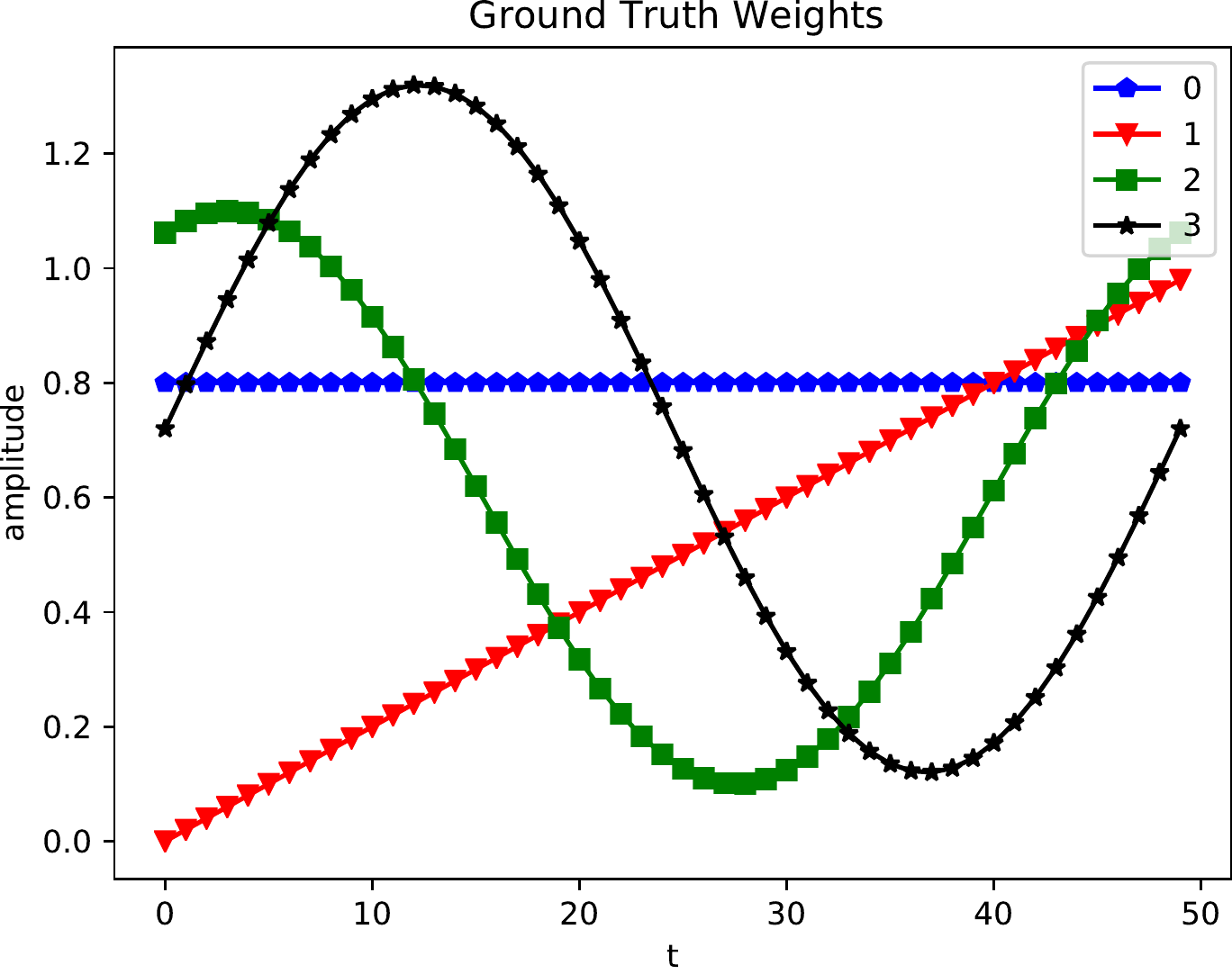}
    \includegraphics[width=4cm,height=2.7cm,keepaspectratio]{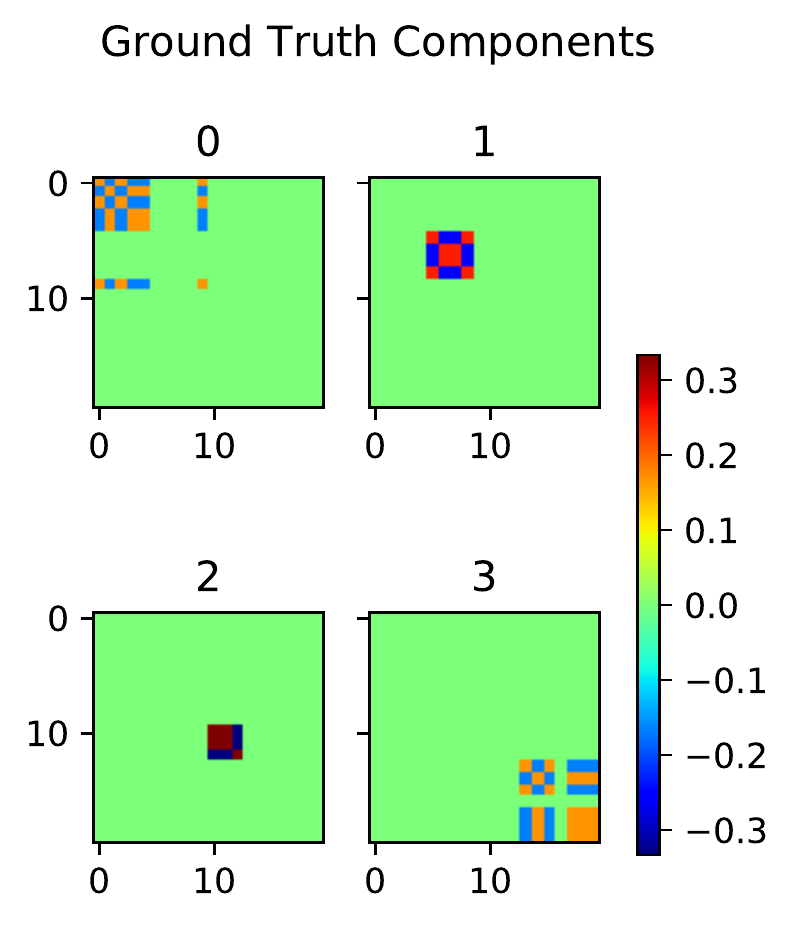}
    \includegraphics[width=4cm,height=2.7cm,keepaspectratio]{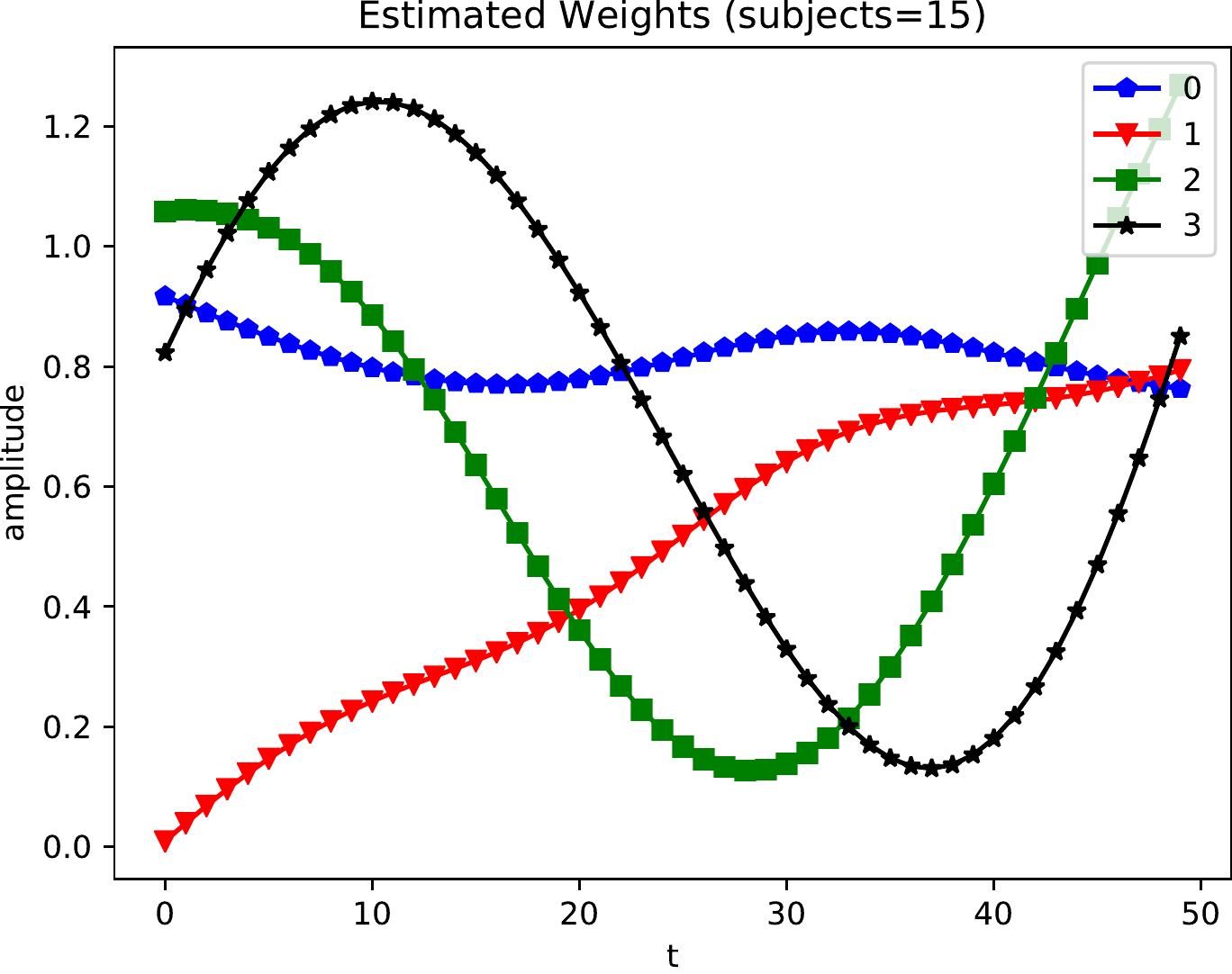}
    \includegraphics[width=4cm,height=2.7cm,keepaspectratio]{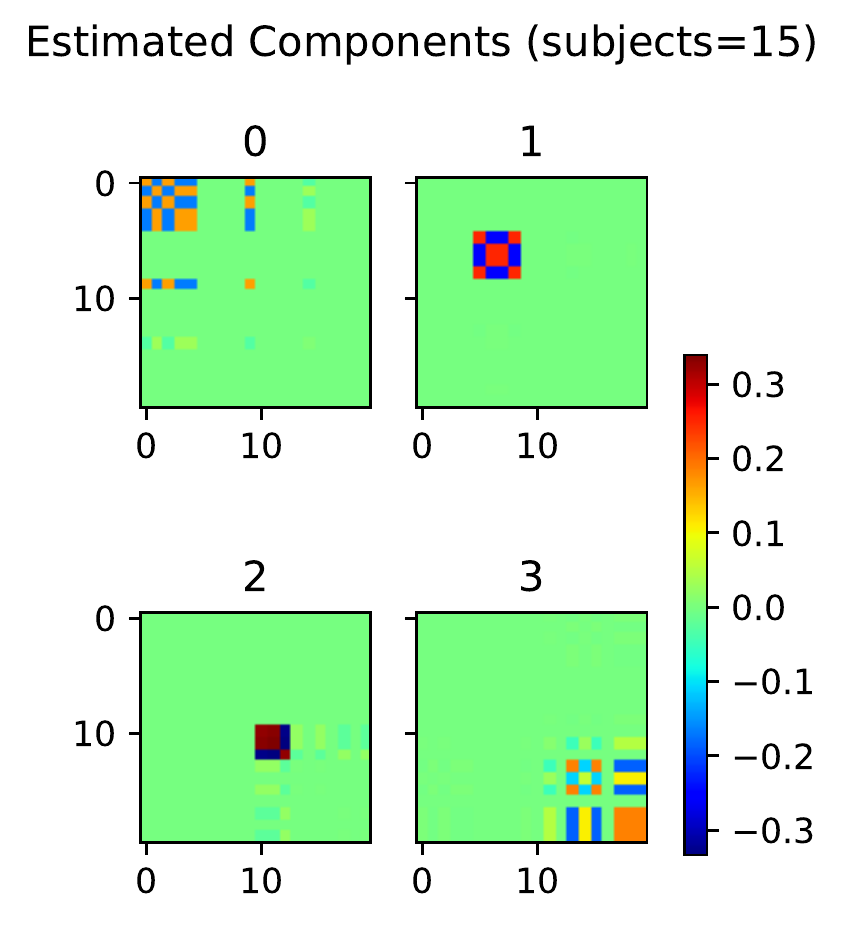}
    \includegraphics[width=4cm,height=2.7cm,keepaspectratio]{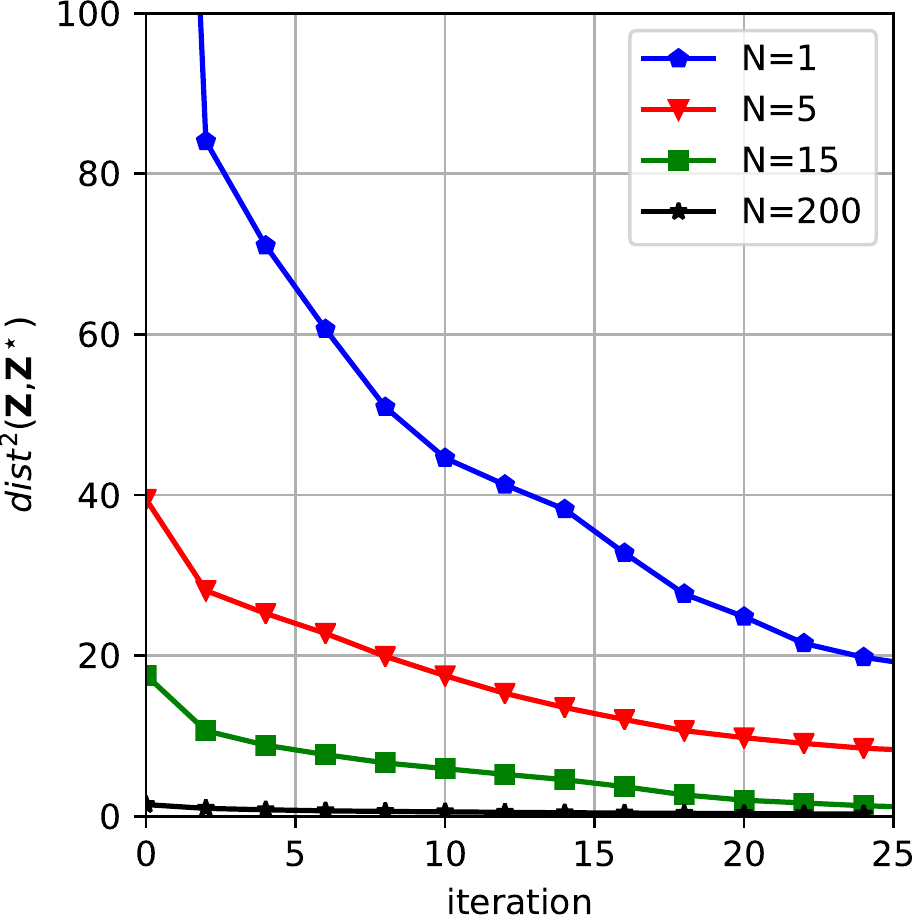}
    \centering
    \includegraphics[width=4cm,height=2.7cm,keepaspectratio]{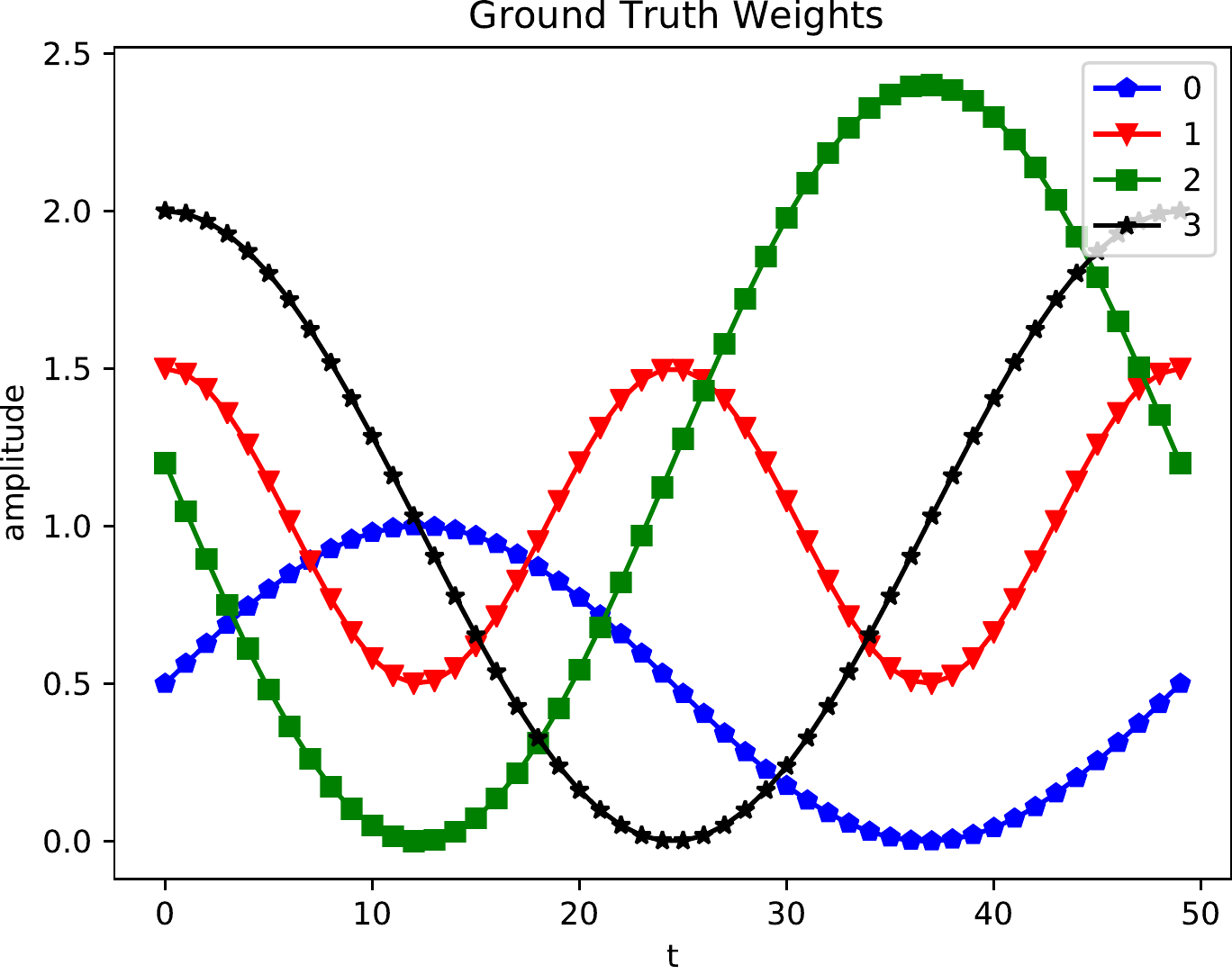}
    \includegraphics[width=4cm,height=2.7cm,keepaspectratio]{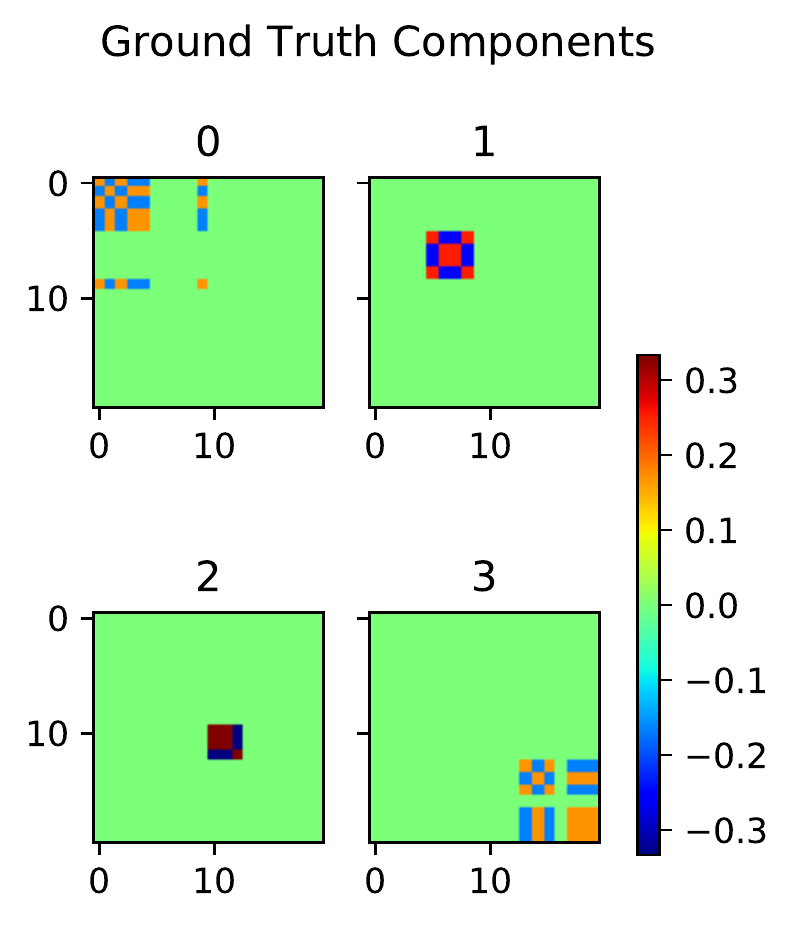}
    \includegraphics[width=4cm,height=2.7cm,keepaspectratio]{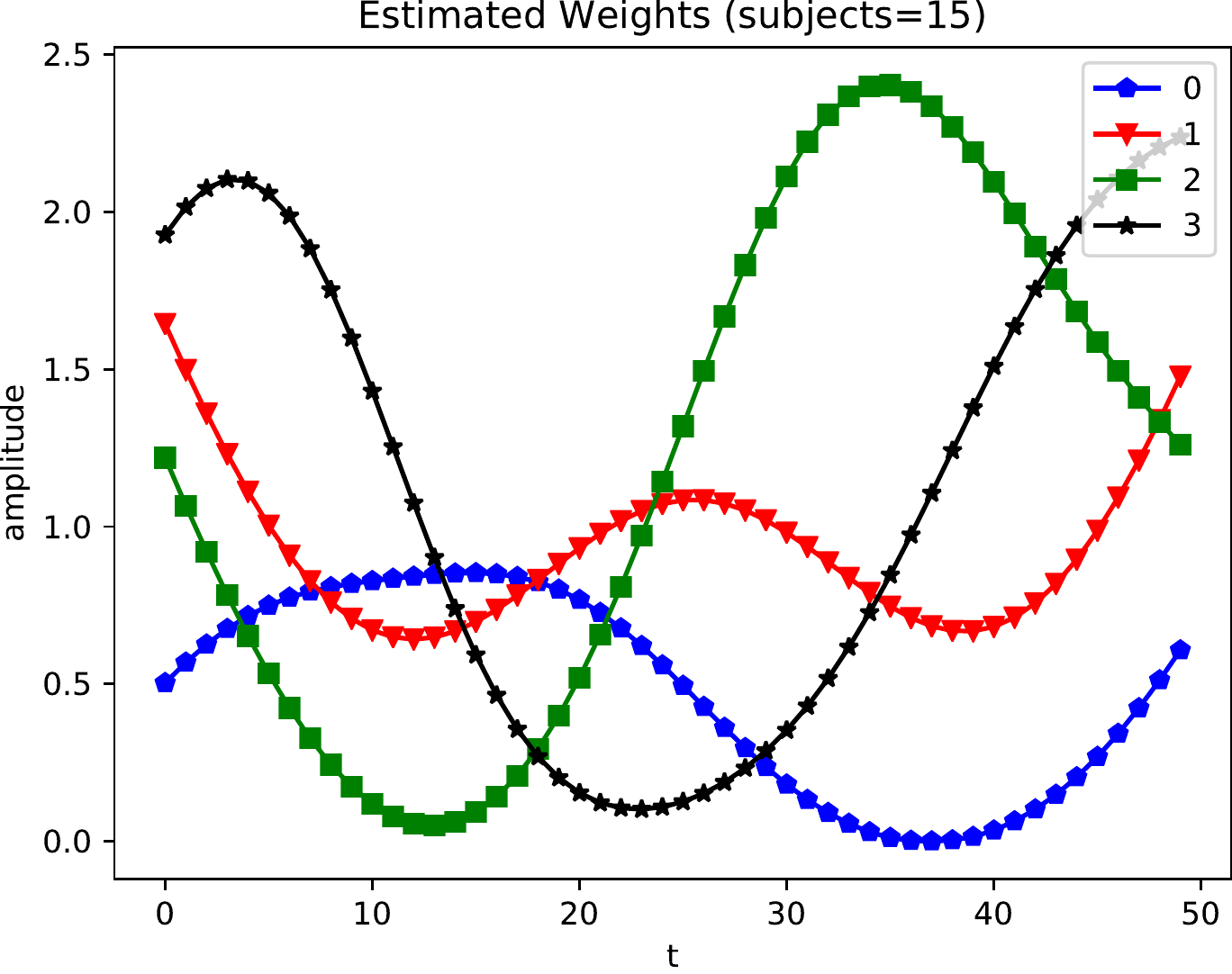}
    \includegraphics[width=4cm,height=2.7cm,keepaspectratio]{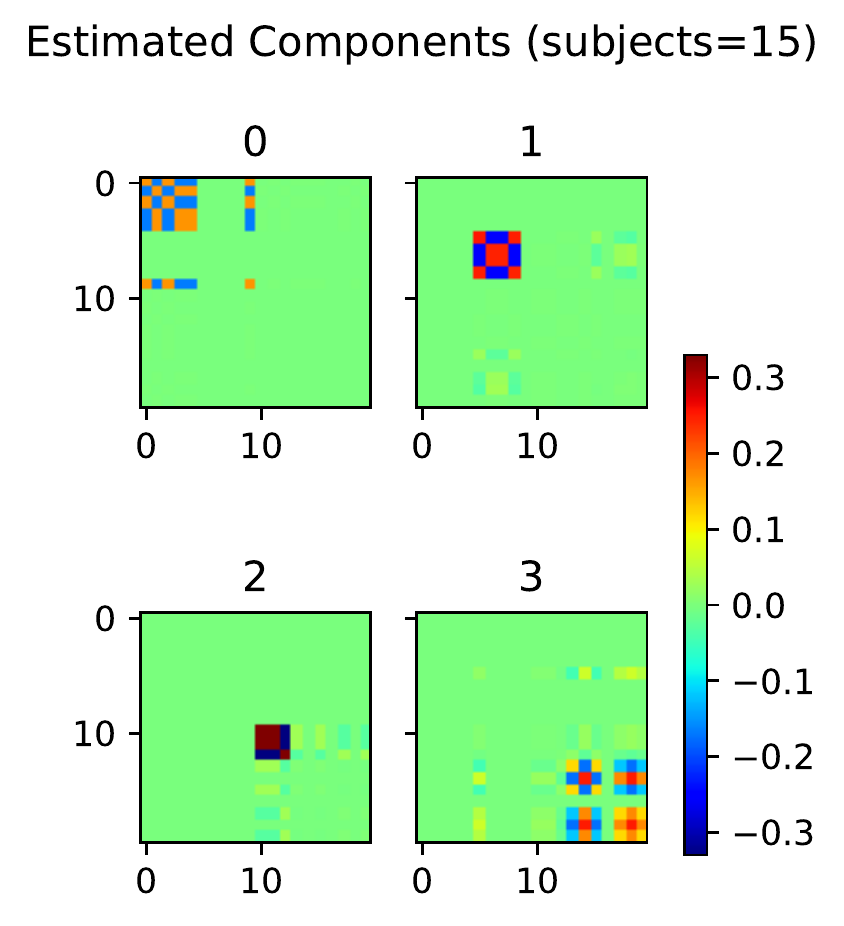}
    \includegraphics[width=4cm,height=2.7cm,keepaspectratio]{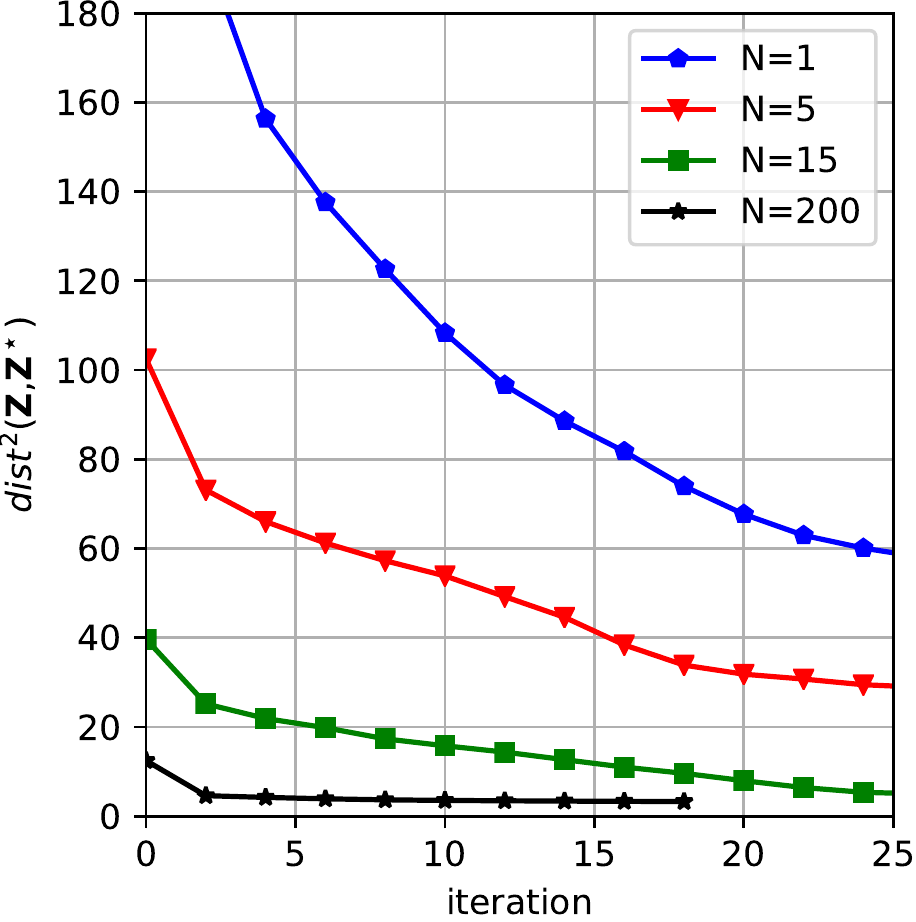}
    \caption{Covariance recovery of $K=4, P=20, J=50$. The left two columns show the simulation ground truth, the center two columns show the recovery of $N=15$, and the rightmost column shows the convergence rate under different number of subjects. The results indicate good spatial and temporal recovery and demonstrate linear rate of convergence.}
    \label{fig:simulation1}
\end{figure}

\subsection{Simulations in high dimension}
We increase both the data dimension $P$ and the number of components $K$ to demonstrate
the effectiveness of the proposed algorithm. The data generation process is described in
the Supplementary Material. Note that the spatial components are non-overlapping in the
previous simulation, whereas we generate spatial components that are partially overlapping
in the high-dimensional setting, making the task more challenging.

To compare with other methods, we use the average log-Euclidean metric~\citep{arsigny2006log}: $J^{-1}\sum_{j=1}^J\|\log({\Sigma}_j)-\log({\Sigma}_j^\star)\|_F$,
where $\log({\Sigma}_j)={ U}_j\log({\Lambda}_j){ U}_j^T$ and $ U$ is the eigenvector matrix, 
and ${\Lambda}$ is the diagonal eigenvalue matrix of ${\Sigma}_j$.  In practice, we truncate the 
eigenvalues whose magnitude is smaller than $10^{-5}$ to maintain stability of the evaluation. 

{Dimension v.s. rank}: We first test the proposed algorithm 
with following experimental settings $N=100$, $P=\{50, 100, 150, 200, 300\}$, $J=100$, $K=\{10,20,30,40,50\}$.
The result of different configurations are shown in Table~\ref{tab:dimensionandrank}.
The results indicate that with fixed $N$, the distance increases with the increment of rank, 
which can be expected as there are more parameters to estimate. 
Moreover, we also find the dimension $P$ has a small influence on the performance.
\begin{table}
  \caption{Average log-Euclidean metric of high dimensional low-rank simulated data}
  \label{tab:dimensionandrank}
    \fontsize{9pt}{9pt}\selectfont
  \centering
  \begin{tabular}{*{6}l}
    
    &$P=50$& $P=100$ &$P=150$&$P=200$&$P=300$\\
    
$K=10$&$0.35  \pm 0.01 $&$0.35  \pm 0.02 $&$0.41  \pm 0.02 $&$0.35  \pm 0.01 $&$0.34  \pm 0.03 $\\
$K=20$&$0.66  \pm 0.02 $&$0.66  \pm 0.02 $&$0.69  \pm 0.02 $&$0.66  \pm 0.03 $&$0.66  \pm 0.01 $\\
$K=30$&$0.82  \pm 0.01 $&$0.82  \pm 0.01 $&$0.82  \pm 0.02 $&$0.80  \pm 0.01 $&$0.79  \pm 0.01 $\\
$K=40$&$0.97  \pm 0.01 $&$0.93  \pm 0.01 $&$0.93  \pm 0.01 $&$0.88  \pm 0.01 $&$0.87  \pm 0.01 $\\
$K=50$&$1.11  \pm 0.01 $&$1.10  \pm 0.01 $&$1.04  \pm 0.01 $&$0.99  \pm 0.01 $&$0.99  \pm 0.01 $\\

  \end{tabular}
\end{table}

{Competing methods}: To compare with other methods, we select relatively small 
$P=100$ as some methods are not scalable to high dimensions.
The settings are $J=100$, $K=10$, and noise level $\sigma=0.5$. Moreover, to make fair comparisons,
we set the number of components for all methods to be $10$. 
We run $20$ trials for each method. For the Bayesian model (M5), we draw $30$ samples from the
posterior distribution and compute the estimated covariance.

From Table~\ref{tab:comparisonhighdimension05}, we see that the proposed model performs the best 
compared to competing methods and the log-Euclidean distance decreases as the number of samples
increases for M**. This observation matches the result of Theorem~\ref{theorem:combine12} 
and Proposition~\ref{prop:SSE}. When comparing M* and M**, we see the decrease of the log-Euclidean distance
resulting from Algorithm~\ref{alg:main}. Comparing M** and MQ**, we see that their scores are very close, implying that estimation with and without QR decomposition does not change much, supporting the theory that $V$ and $V_{ortho}$ span the same subspace.  Notably, our model yields comparable performance with M5. 
We can expect this because the model structure of M5 and the proposed model are similar, 
though M5 takes the Bayesian framework and uses variational inference~\citep{blei2017variational}. 
Moreover, we compare the running time of different methods. From Table~\ref{tab:runinghighdimension}, we see that the running time of the proposed method remains relatively stable as the number of subjects increases. On the other side, the running time of other methods, M2-M4, increases as $N$ increases. While our method remains efficient in high dimension settings, many other methods become slow as the dimension increases. Finally, even though M5 has comparable performance to ours, our method is more computationally efficient than the counterparts.

\begin{table}[ht!]
    \caption{Average log-Euclidean metric of high dimensional low-rank data ($\sigma=0.5$)}
  \label{tab:comparisonhighdimension05}
    \fontsize{9pt}{9pt}\selectfont
  \centering
  \begin{tabular}{*{8}l}
    Methods & \multicolumn{4}{c}{Number of training subjects}\\
    & 10 &20&30&40&50\\
    M1 $W=20$& $0.49\pm0.01$& $0.46\pm0.01$& $0.45\pm0.01$& $0.45\pm0.01$& $0.442\pm0.01$&\\
    M2 & $1.22\pm0.01$& $1.04\pm0.01$& $1.00\pm0.01$& $0.98\pm0.01$& $0.97\pm0.01$&\\
    M3& $71.50\pm6.46$& $1.90\pm0.26$& $1.12\pm0.01$& $1.14\pm0.01$& $1.12\pm0.01$&\\
    M4 & $0.94\pm0.03$& $0.46\pm0.01$& $0.41\pm0.01$& $0.39\pm0.01$& $0.38\pm0.01$&\\
    M5&$0.51\pm0.01$&$0.46\pm0.01$&$0.43\pm0.01$&$0.42\pm0.01$&$0.41\pm0.01$\\
    M* & $0.43\pm0.01$& $0.40\pm0.01$& $0.40\pm0.01$& $0.39\pm0.01$& $0.39\pm0.01$&\\
    M** & $0.42\pm0.03$& $0.35\pm0.05$& $0.36\pm0.04$& $0.33\pm0.04$& $0.32\pm0.02$&\\
    MQ** & $0.40\pm0.04$& $0.40\pm0.01$& $0.35\pm0.04$ & $0.32\pm0.03$ & $0.32\pm0.02$
  \end{tabular}
\end{table}
\begin{table}
    \caption{Running time $(seconds)$ of high dimensional low-rank data ($\sigma=0.5$)}
  \label{tab:runinghighdimension}
    \fontsize{9pt}{9pt}\selectfont
  \centering
  \begin{tabular}{*{7}l}
    
    Methods & \multicolumn{5}{c}{Number of training subjects}\\
    &10&20&30&40&50\\
    
    M1 $W=20$& $0.8\pm0.4$& $0.5\pm0.3$& $1.2\pm0.6$& $1.5\pm0.5$& $1.8\pm0.5$&\\
    M2 &  $151.9\pm17.9$& $222.6\pm51.0$& $376.7\pm53.6$& $498.8\pm40.0$& $635.5\pm58.2$\\
    M3 & $422.0\pm33.3$& $729.7\pm161.6$& $1154.4\pm53.6$& $1427.4\pm64.4$& $1751.7\pm52.0$\\
    M4 &   $267.0\pm112.1$& $378.4\pm148.5$& $846.2\pm358.8$& $872.0\pm440.8$& $1841.5\pm697.2$\\
M5 &$2243.8\pm33.3$&$2263.3\pm38.4$&$2273.7\pm36.0$&$2259.8\pm34.2$&$2278.9\pm35.1$\\
    M* & $0.1\pm0.0$& $0.1\pm0.1$& $0.2\pm0.1$& $0.2\pm0.1$& $0.2\pm0.1$&\\
    M** & $1.2\pm0.6$& $1.3\pm0.7$& $2.9\pm1.4$& $3.9\pm0.8$& $3.6\pm0.7$&\\
    MQ** & $2.2\pm1.3$& $1.4\pm0.7$& $2.5\pm1.0$& $3.8\pm0.7$& $3.5\pm0.6$
  \end{tabular}
\end{table}

\section{Experiment on neuroimaging data}
To investigate the proposed model on real data, we focus on (i) the interpretability of the model and (ii) the out-of-sample prediction. We use the motor task data from the Human Connectome Project functional magnetic resonance imaging (fMRI) data~\citep{van2013wu}. The data is preprocessed using the existing pipeline~\citep{van2013wu}, and an additional high-pass filter with a cutoff frequency $0.015Hz$ to remove the physiological noise  as recommended by ~\citet{smith1999investigation}. The data consists of five motor tasks: right hand tapping, left foot tapping, tongue wagging, right foot tapping, and left hand tapping. 

For the model interpretation experiment, we select $N=20$ subjects. Preprocessed time series $J=284$ for each subject were extracted from $P=375$ cortical and subcortical parcels, following~\citep{shine2019human}. The regions include 333 cortical parcels (161 and 162 regions from the left and right hemispheres, respectively) using the Gordon atlas~\citep{gordon2016generation}, 14 subcortical regions from the Harvard–Oxford subcortical atlas (bilateral thalamus, caudate, putamen, ventral striatum, globus pallidus, amygdala, and hippocampus), and 28 cerebellar regions from the SUIT atlas 54~\citep{diedrichsen2009probabilistic}. 

During the session, each task is activated twice (see the activation sequence in Supplementary Material). The goal is to analyze the corresponding dynamic connectivity. To investigate the temporal and spatial components, we compute the correlation of each weight $\rowkA$ for every $k\in[K]$ with the onset task activation, and select the component that has the highest correlation shown in Figure~\ref{fig:motorfMRI}. Our results show that the 
   temporal fluctuations of the top components coincide with the task activation. 
   
   Following the hypothesis that the neural activity are the consequence of multiple components rather than single components~\citep{posner1988localization}, for each task, we select three components with the highest correlations and plot the connectivity patterns in Figure~\ref{fig:connectome_combination}. The spatial hubs in the connectivity matrices  closely match with the expected motor regions as defined in the cortical homunculus~\citep{marieb2007human}. Thus, the results indicate that the proposed algorithm can separate and identify the components of each task, and each task has a unique connectivity pattern.
 
\begin{figure}
    \includegraphics[width=0.32\textwidth]{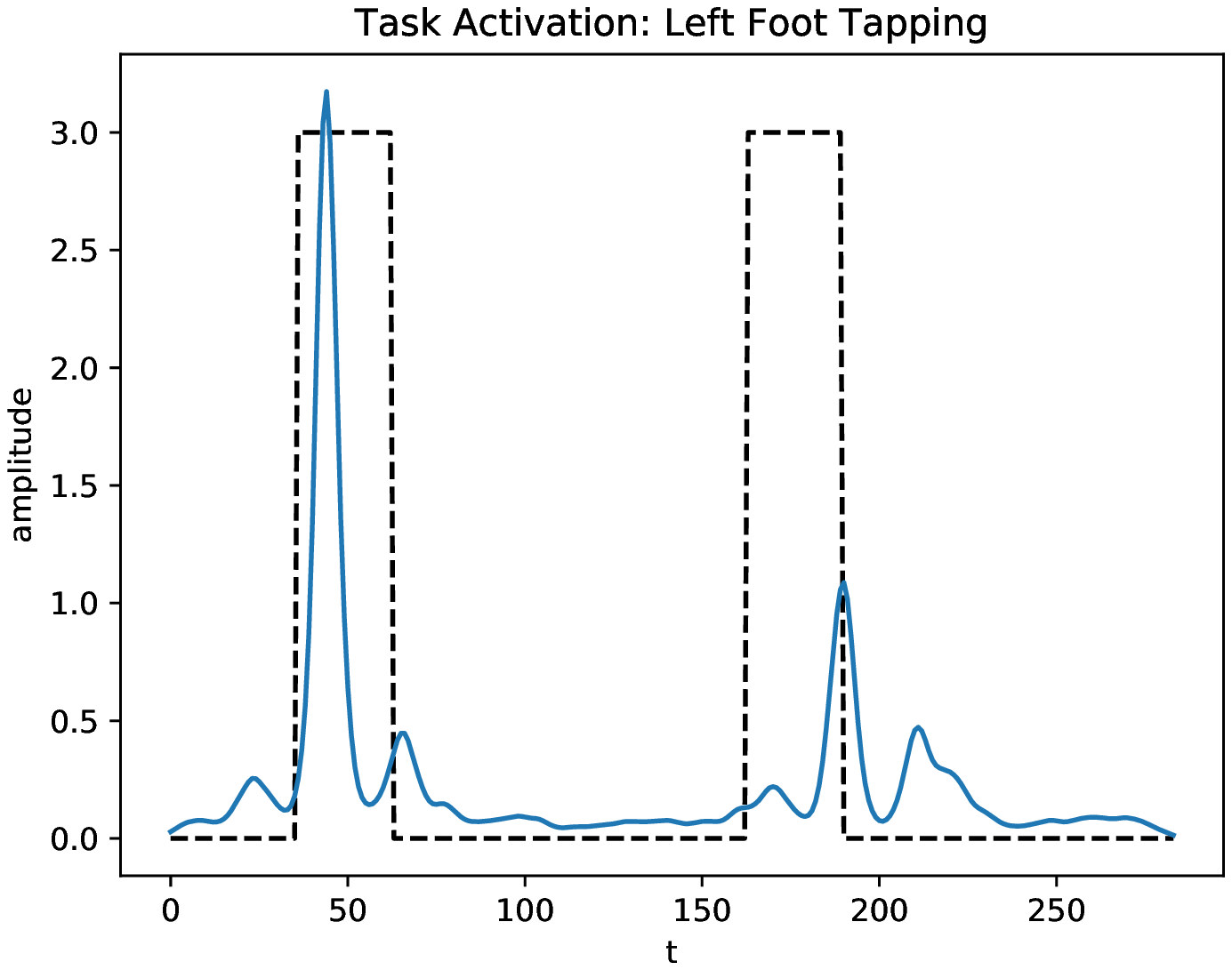}
    \includegraphics[width=0.32\textwidth]{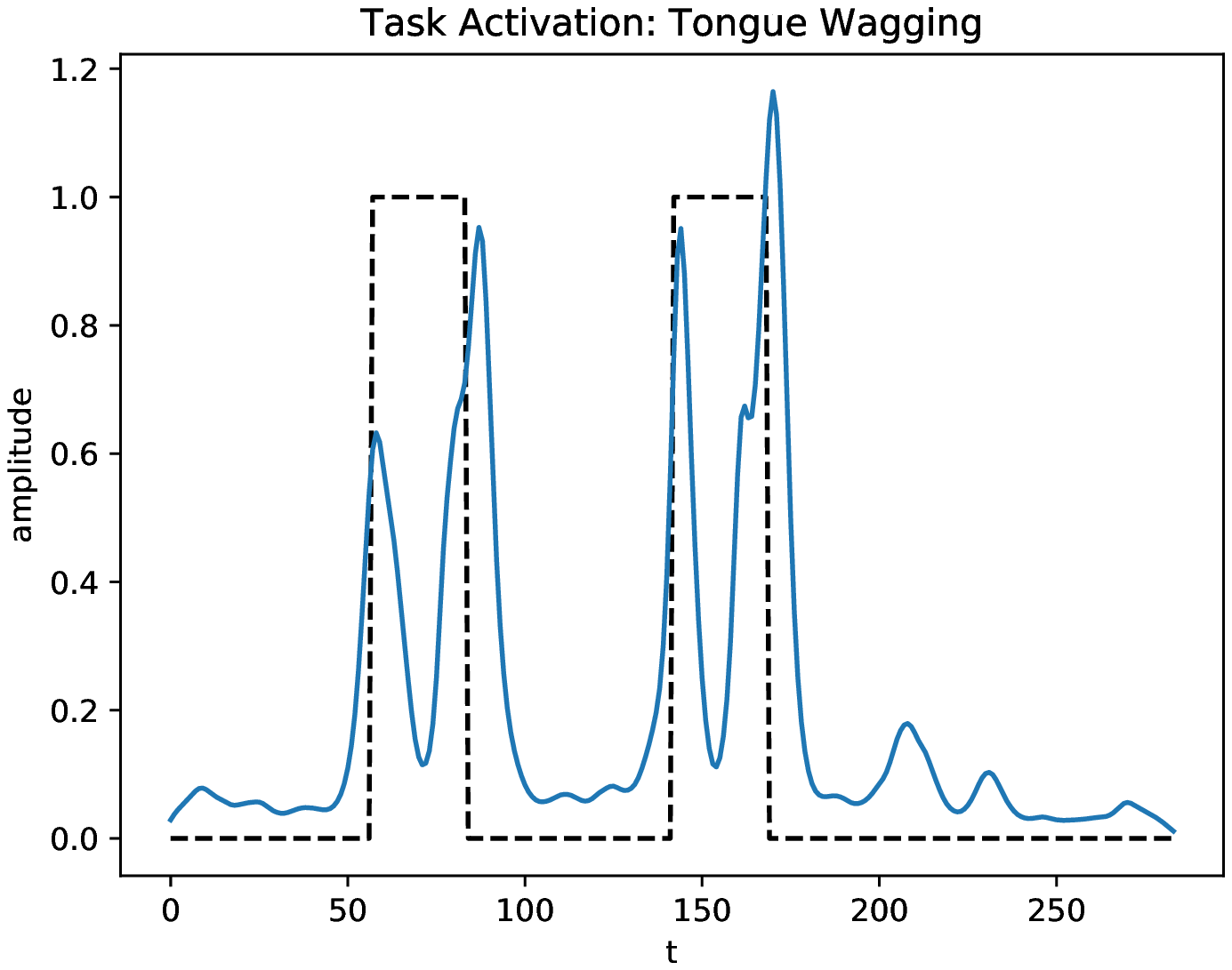}
    \includegraphics[width=0.32\textwidth]{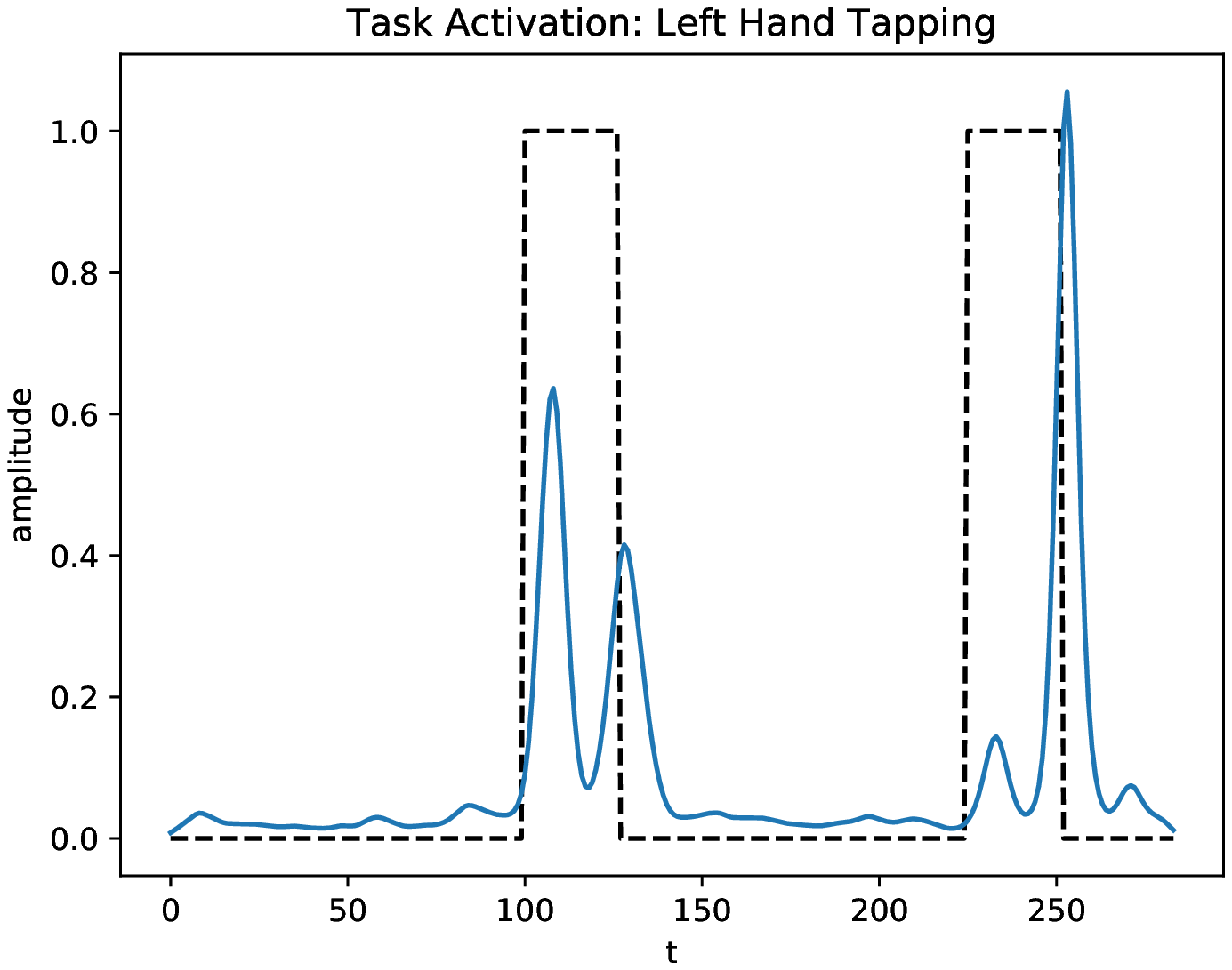}
    \includegraphics[width=0.32\textwidth]{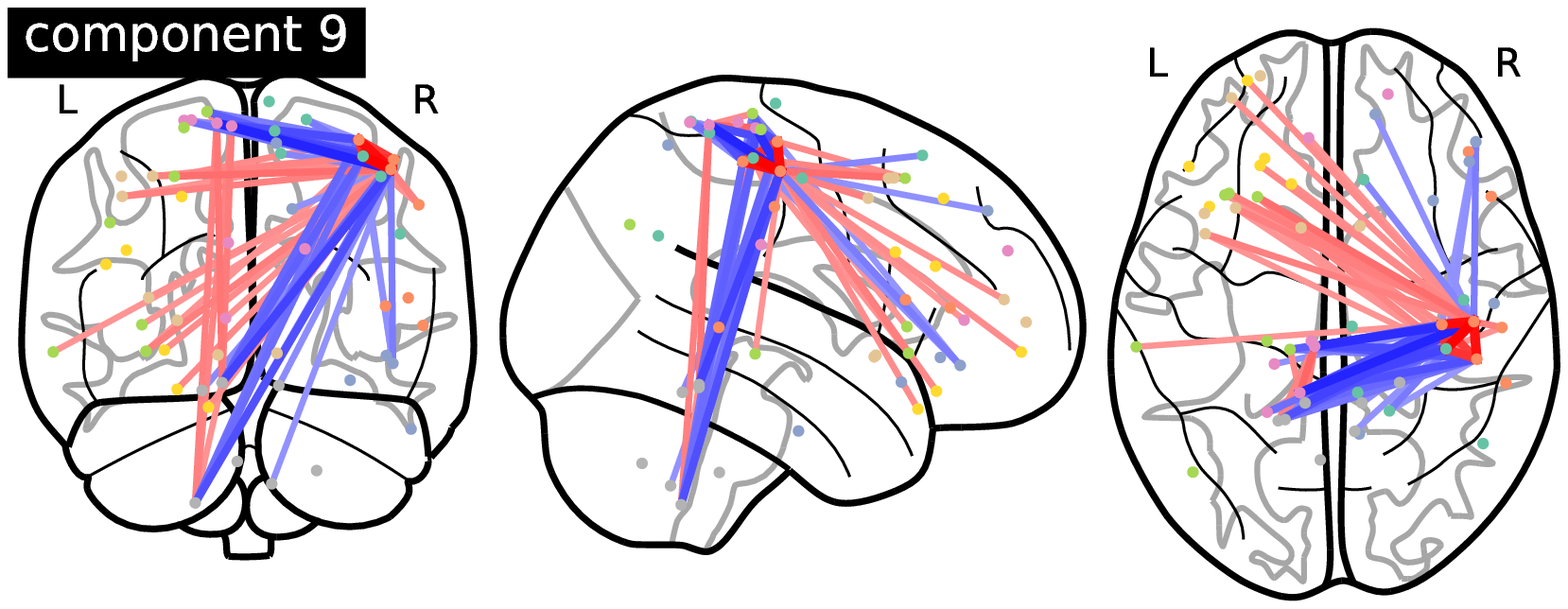}
    \includegraphics[width=0.32\textwidth]{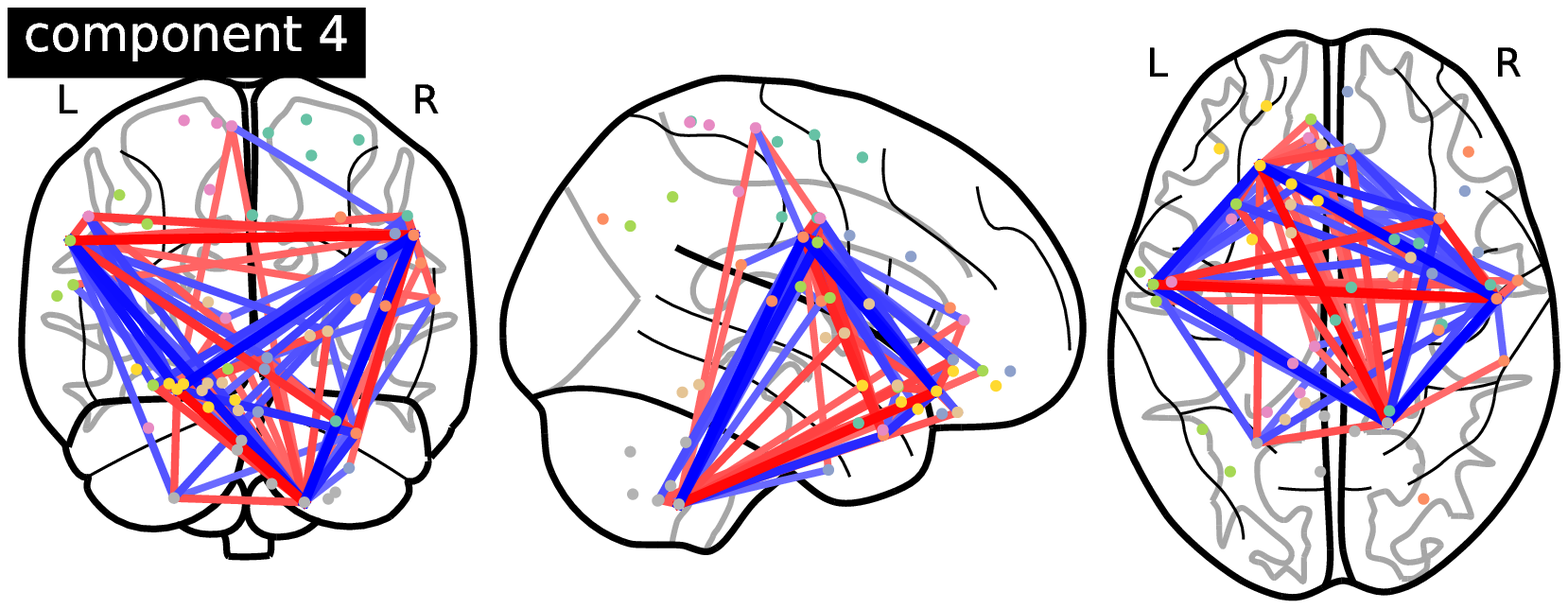}
    \includegraphics[width=0.32\textwidth]{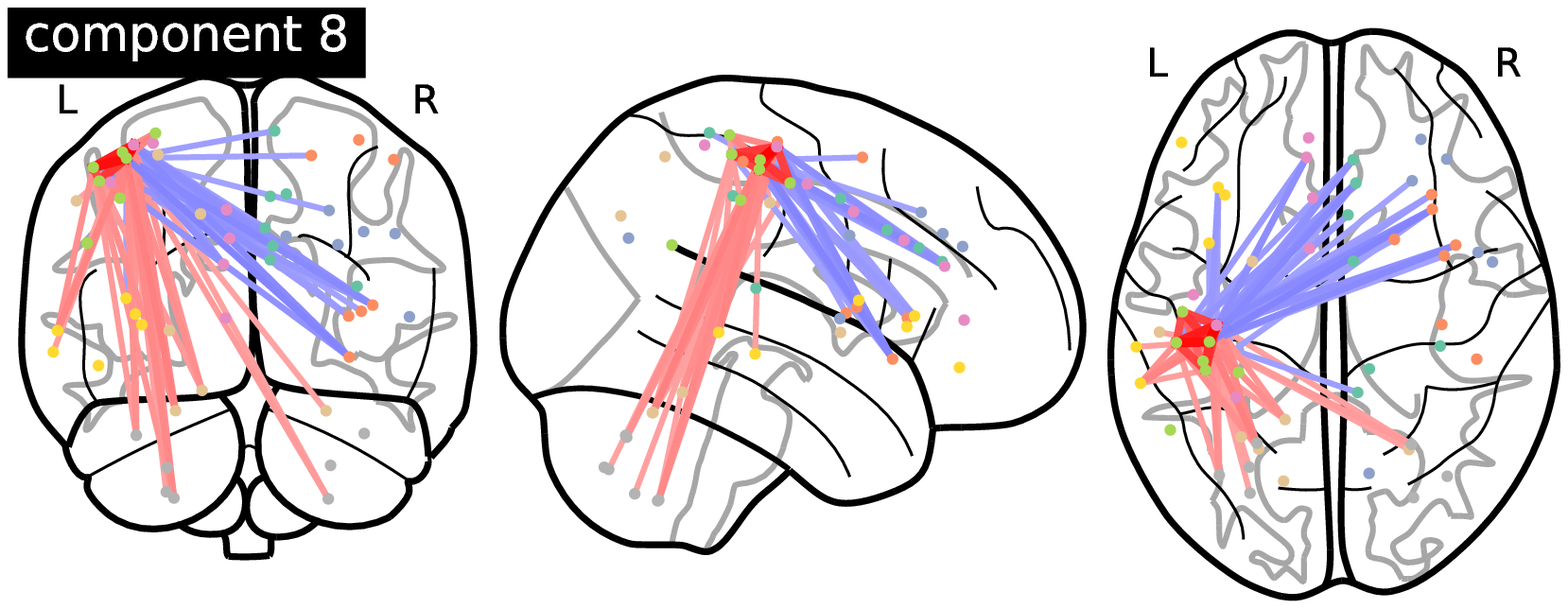}
    \caption{ The top row shows the temporal components (blue solid lines) whose correlations are the largest with respect to the task activation (black dotted lines). The bottom row shows the corresponding brain connectivity patterns (spatial components) of the tasks above. The red lines denote positive connectivity and the blue lines denote negative connectivity.}
    \label{fig:motorfMRI}
\end{figure}

\begin{figure}
    \centering
    \includegraphics[width=0.32\textwidth]{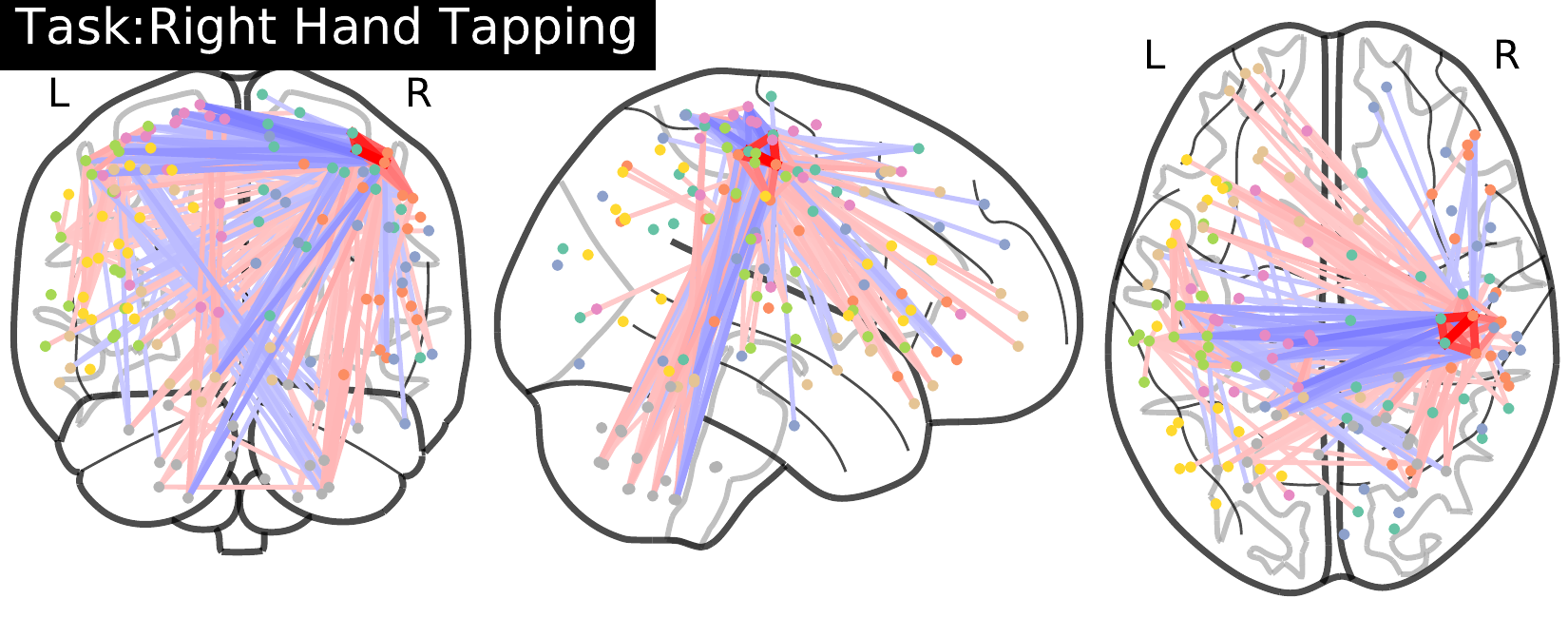}
    \includegraphics[width=0.32\textwidth]{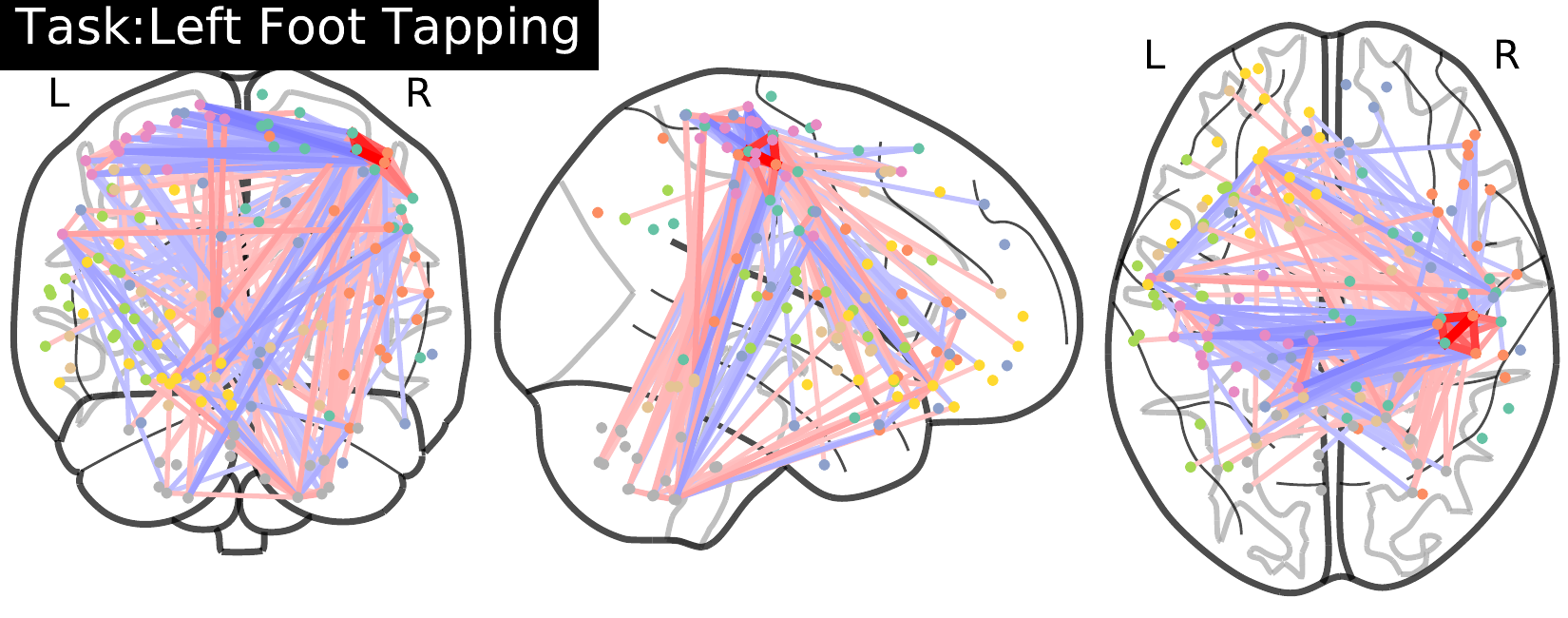}
    \includegraphics[width=0.32\textwidth]{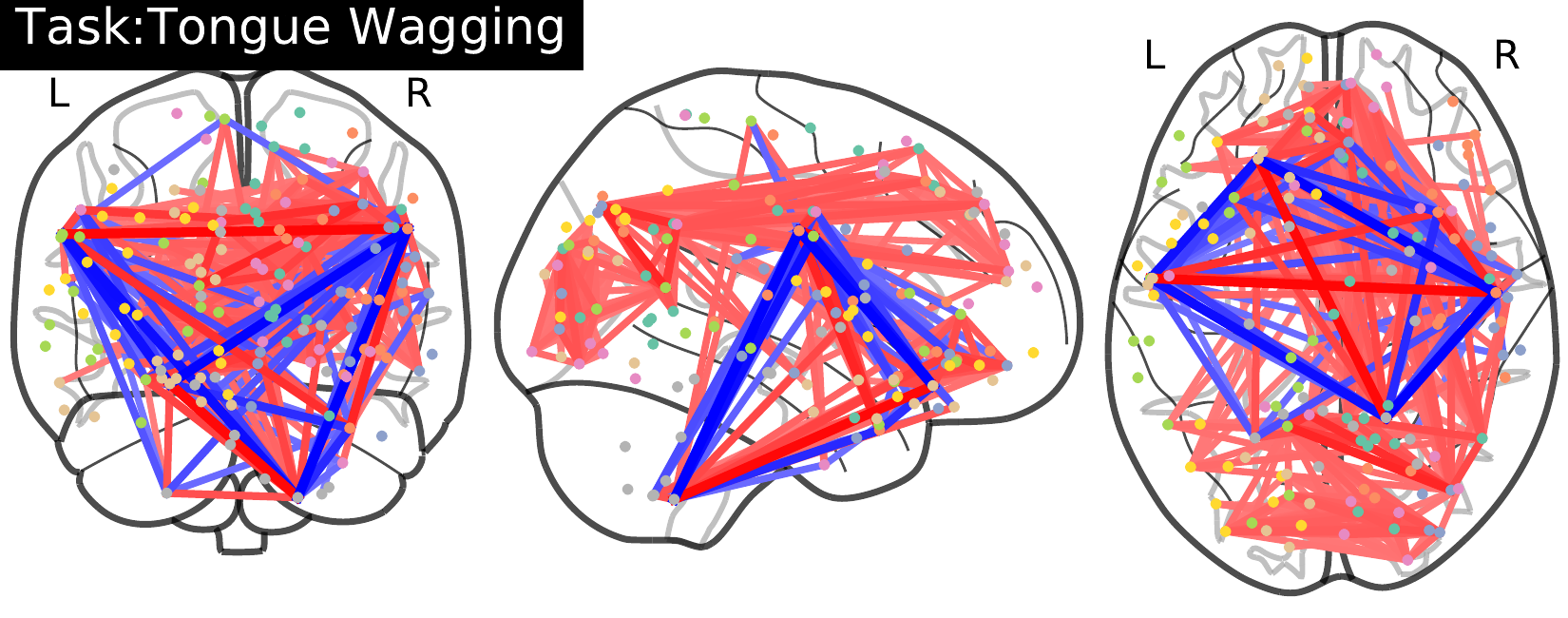}
    \includegraphics[width=0.32\textwidth]{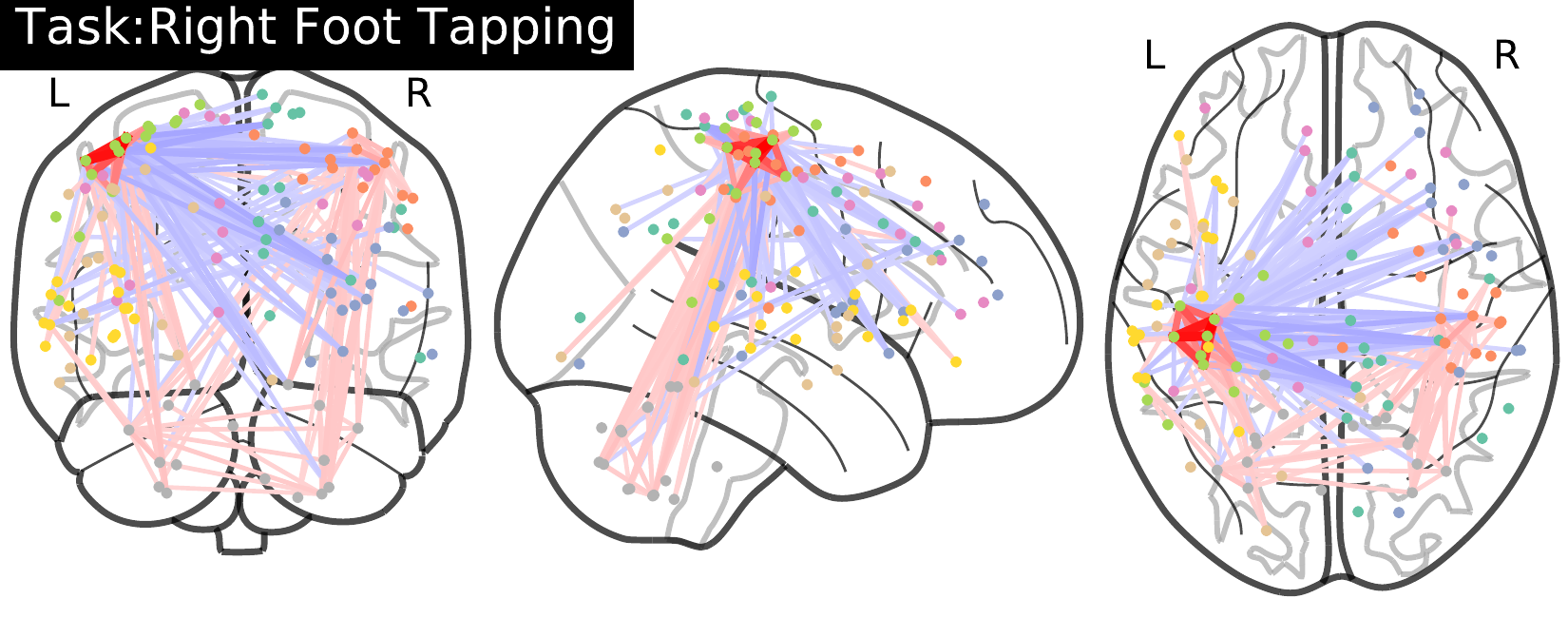}
    \includegraphics[width=0.32\textwidth]{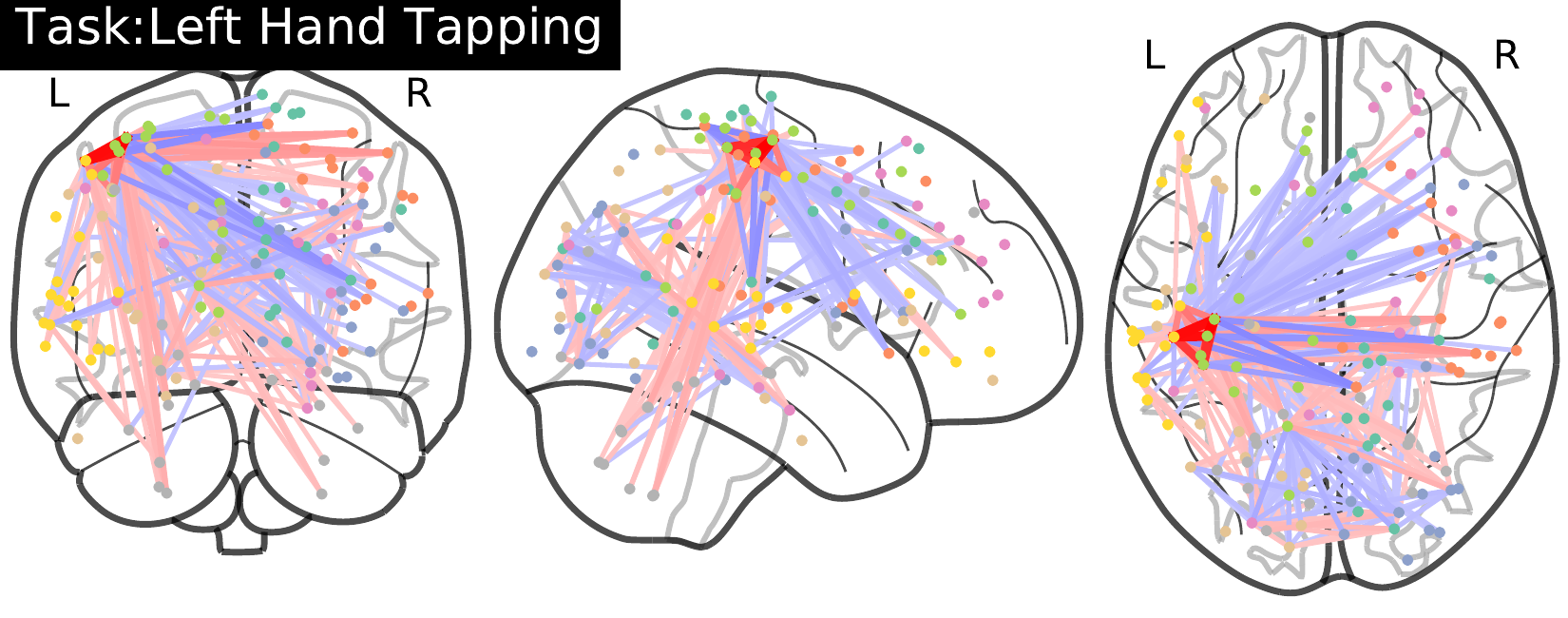}
    \caption{ Each connectome is the superposition of top three spatial components. The spatial hubs in the connectivity matrices closely match with the expected motor regions (hands, feet, tongue) as defined in the cortical homunculus~\citep{marieb2007human}.}
    \label{fig:connectome_combination}
\end{figure}

As the ground truth is unknown, and motivated by the hypothesis that each task has a separate activation pattern, we design a classification task as a surrogate experiment to evaluate the algorithm. Prior observations~\citep{zalesky2012use,calhoun2014chronnectome} also indicate that task fMRI data share similar connectivity patterns across test subjects. Thus, if we can recover the functional connectivity patterns of the training subjects, then similar patterns exist in testing subjects. We partition the Human Connectome Project motor task dataset~\citep{van2013wu}, which contains $103$ subjects, into a training set and testing set. The length of each task is identical, $27$ time points for each activation, and $2$ activations in each session. Since each task partially overlaps with others (see the activation map in Supplementary Material), we predict the task based on the activation blocks rather than single time points. We group the estimated covariances $\{\Sigma_j\}_{j\in[J]}$ and the testing data based on the task activation map and perform a nearest neighbor search. The clustered covariance is denoted as ${\Sigma}_{task,i}$, where $task\in\{\text{right hand tapping}, \text{left foot tapping}, \text{tongue wagging}, \text{right foot tapping},\text{left hand tapping}\}$ and $i\in[54]$. The task score of each block of testing data is defined as
\begin{equation*}
    score_{task}(\{{ x}_i\}_{i\in[54]})=\sum_{i=1}^{54}\|{ x}_i{ x}_i^T-{\Sigma}_{task,i}\|_F^2.
\end{equation*}
where $\{{ x}_i\}_{i\in[54]}$ is a block of testing data. We predict the task of the block data by choosing the task with the minimum score. We repeat the experiment $10$ times and the result is shown in Table~\ref{tab:classfication error}. Note that the Markov model (M2) performs the worst even if we increase the number of states to $60$. The dictionary learning model (M4) has comparable performance to our model when the sample size is large, and yet our model performs better with small sample sizes.

\begin{table}
  \caption{The test classification accuracy of using covariance parameters to predict tasks in the Human Connectome Project motor dataset ($\%$)}
  \label{tab:classfication error}
    \fontsize{9pt}{9pt}\selectfont
  \centering
  \begin{tabular}{*{8}l}
    Methods & \multicolumn{5}{c}{Number of training subjects}\\
    & 10 &20&30&40&50\\

    M1 $(W=10, K=15)$ &$49.6\pm6.4$& $70.8\pm3.4$& $77.3\pm3.2$& $79.7\pm2.8$&$80.8\pm1.7$\\
    M1 $(W=50, K=15)$ &$26.5\pm4.7$&  $34.9\pm3.3$&  $36.5\pm3.5$& $38.4\pm2.0$& $39.6\pm2.4$  \\
    M2 ($K=15$) & $27.7\pm3.5$&$27.1\pm4.3$&$25.1\pm2.0$&$24.3\pm4.4$&$24.8\pm2.0$\\
    M2 ($K=60$)&$37.6\pm5.3$&$44.5\pm4.7$&$40.9\pm5.8$&$38.2\pm3.9$&$44.1\pm6.7$\\
    M4 ($K=15$)&$52.6\pm8.3$& $77.4\pm6.8$&$85.0\pm3.5$&$89.0\pm3.0$&$90.5\pm2.5$\\
    M6 ($W=10$)& $70.9\pm2.7$&$78.3\pm3.3$& $81.0\pm1.6$& $79.8\pm3.0$&$81.2\pm2.6$\\
    M** ($K=15$) & $61.5\pm7.4$  & $81.5\pm5.4$  &$87.4\pm2.5$  &$90.0\pm1.7$  &$90.5\pm2.5$ \\
  \end{tabular}
\end{table}
\section{Discussion}
Several directions are worthy of further investigation.  We plan to explore more flexible spatial structure. Prior work~\citep{gibberd2017regularized, hallac2017network} applied fused graphical lasso and group graphical lasso to encourage similar sparse structures for time-varying graphical models. 
These approaches did not restrict the spatial components to be identical,
but only similar, and thus are more flexible compared to
the proposed model. To this end, one idea is to build factor models 
that encourage similar, but not identical spatial structures,
while retaining low-rank. Finally, while our work has focused on fixed sampling intervals, another direction is to explore models with samples obtained at irregular time intervals~\citep{tank2019identifiability,qiao2020doubly}, as this setting is common in multimodal data.



\section*{Acknowledgement}
We thank James M. Shine for help with data pre-processing and helpful modeling discussions. 
The research project is partially funded by the U.S.A. National Institutes of Health 1R01MH116226-01A, 
National Science Foundation Graduate Research Fellowships Program, and Strategic Research Initiatives 
Grainger College of Engineering, the University of Illinois at Urbana-Champaign. 
Data were provided by the Human Connectome Project, WU-Minn Consortium (Principal Investigators:
David Van Essen and Kamil Ugurbil; 1U54MH091657) funded by the 16 NIH Institutes and Centers that
support the NIH Blueprint for Neuroscience Research; and by the McDonnell Center for Systems 
Neuroscience at Washington University.

\section*{Supplementary material}
The supplementary material contains technical proofs and additional experimental results. 

\newpage

\putbib[bu1]
\end{bibunit}

\appendix
\pagebreak
\begin{bibunit}[my-plainnat]

\input{supplement2}

\putbib[bu2]
\end{bibunit}
\end{document}

%% file: supplement2.tex
\section*{Supplementary material for A Nonconvex Framework for Structured Dynamic
Covariance Recovery}

We propose a nonconvex framework to estimate structured covariance matrices. 
In the present setting, the dynamic covariance matrices are decomposed into
low-rank sparse spatial components and smooth temporal components. 
We employ a two-stage approach to minimize the proposed nonconvex objective function: 
we design a spectral initialization method to obtain a good initial guess followed by iterative refinements via projected gradient descent. This approach converges linearly to a statistically useful solution, which can be quantified by the statistical error rate.
In this supplementary material, we provide technical proofs and additional experimental results. 

\section{Projection to constraint sets}

We describe the algorithms used to project iterates 
to the constraints $\projV$, $\projqrV$, and $\projtrnA$.
Next, we characterize the expansion coefficient
induced by projecting to nonconvex sets.

\subsection{Projections to {$\projV$ and $\projqrV$}%
}\label{sec:proj_v}

Recall that $\defprojV{s}$.
To project a vector $v$ onto $\projV(s)$,
we want to solve the following problem
\begin{align}
    \arg\min_{{x}\in\mathcal{C}_{V}}\|{v}-{x}\|_2^2\label{eq:projD}.
\end{align}
Let $\mathcal{S}({x})=\{i:{x}_i\neq 0\}$ be the support of ${x}$. 
Given a support $E\subset[P]$, 
let $[{x}]_E \in \mathbb{R}^P$ be a vector whose $i$th entry is 
equal to $x_i$ if $i \in E$ and $0$ otherwise.
Let
\begin{align*}
    &d(E)=\min_{x}\|{v}-{x}\|_2^2\notag\quad\text{subject to}\quad\mathcal{S}({x})\subseteq     E,\;\|{x}\|_2=1,
\end{align*}
and observe that
\begin{align*}
    d(E)&=\min_{x}\|{v}\|_2^2+\|{x}\|_2^2-2\langle{x},{v}\rangle\notag\\
    &=\|{v}\|_2^2+1-2\max_{{x}}\langle{x},{v}\rangle\notag\\
    &=\|{v}\|_2^2+1-2\|[{v}]_{E}\|_2.
\end{align*}
Then we can conclude that
\begin{align}\label{eq:projD2}
    \hat{E}=\arg\min_{E:|E|\leq s}d(E)=\arg\max_{E:|E|\leq s}\|[{v}]_{E}\|_2,
\end{align}
which can be solved by finding the top-$s$ entries of ${v}$ in magnitude.
This can be done with computational complexity $O(P\log P)$. 
After finding the support in~\eqref{eq:projD2}, 
we can obtain \eqref{eq:projD} by projecting $[x]_{\hat{E}}$ to the unit sphere.
Algorithm~\ref{alg:proj_v} summarizes the procedure.
\begin{algorithm}
\caption{Projection to $\projV(s)$} \label{alg:proj_v}
\begin{tabbing}
    \qquad \enspace Input: $v\in\mathbb{R}^P$\\
   \qquad \enspace ${v}_S\leftarrow$ Pick the top-$s$ entries of $v$ in magnitude and set the rest of entries to $0$\\
   \qquad \enspace $\hat{v}\leftarrow$ Project $v_S$ to the unit sphere $S^{P-1}$\\
    \qquad \enspace Output $\hat{v}$
\end{tabbing}
\end{algorithm}

Next, we discuss the projection procedure when we additionally orthogonalize the estimate via QR decomposition. Let $\hat{V}=\Pi_{\projV}(V)$ and $\hat{V}=BL$ be the QR decomposition of $\hat{V}$, where $B$ has orthonormal columns and $L$ is upper triangular matrix. We define
\begin{align}\label{eq:projqrv}
    \Pi_{\projqrV}(V)=B.
\end{align}

\subsection{Projection to {$\projtrnA$}}

Recall that
$\defprojtrnA{c}{\gamma}$, 
where $\deftrnGinv$ is
the eigendecomposition of $\trnG^\dagger$
and $\lambda_j$ denotes
the $j$th diagonal entry of $\trnLam$.
To project to the convex set $\projtrnA$, 
we use an alternating projection method. 
While the convergence rate of the alternating projection method is not our focus, 
in the experiments, we observe 
that often one iteration of the alternating projection
results in an iterate that satisfies both constraints. 
Algorithm~\ref{alg:proj_a} summarizes the alternating projection procedure.
\begin{algorithm}
\caption{Projection to $\projtrnA(c,\gamma)$} \label{alg:proj_a}
\begin{tabbing}
    \qquad \enspace Input: $\alpha\in\mathbb{R}^J$\\
   \qquad \enspace While $\alpha\not\in\projtrnA(c,\gamma)$\\
   \qquad \qquad $\hat{\alpha}\leftarrow$ Project $\alpha$ to the hypercube $[0,c]^J$\\
   \qquad \qquad $\alpha\leftarrow$ Project $\hat\alpha$ to the set  $\ellipsoid{\trnQ}{\trnLam}{\gamma}$ using Algorithm~\ref{alg:proj_G}\\
    \qquad \enspace Output $\alpha$
\end{tabbing}
\end{algorithm}

Next, we provide an algorithm for projecting to the ellipsoid 
$\ellipsoid{\trnQ}{\trnLam}{\gamma}$, which is one of the steps
in Algorithm~\ref{alg:proj_a}. It is easy to see that a vector $y$ 
belonging to $\ellipsoid{\trnQ}{\trnLam}{\gamma}$ lies in the range 
space of $\trnQ$, which has dimension $\rkG{\trnG}$. Let $Q_1$ be the 
matrix whose orthonormal columns form the subspace orthogonal to 
columns of $\trnQ$, which has dimension $J-\rkG{\trnG}$. In this case, we have $Q_1^Ty=0$. 

\begin{algorithm}
\caption{Projection to $\ellipsoid{\trnQ}{\trnLam}{\gamma}$} \label{alg:proj_G}
\begin{tabbing}
    \qquad \enspace Input: $\hat\alpha\in\mathbb{R}^J$\\
   \qquad \enspace If  $\hat\alpha^T\trnG^\dagger\hat\alpha\leq\gamma$\\
   \qquad \qquad $a\leftarrow \trnQ\trnQ^T\hat\alpha$\\
   \qquad \enspace Else\\
   \qquad \qquad $u\leftarrow \trnQ^T\hat \alpha$\\
   \qquad\qquad $D\leftarrow$ Find the roots of $x$: 
    $3x^2\sum_{j=1}^{\rkG{\trnG}}\lambda_j^3u_j^2-2x\sum_{j=1}^{\rkG{\trnG}}\lambda_j^2u_j^2+\sum_{j=1}^{\rkG{\trnG}}\lambda_ju_j^2-\gamma=0$
    \\
    \qquad \qquad $\hat{x}\leftarrow$ Pick the largest nonnegative value in the set $D$\\
   \qquad \qquad $a\leftarrow \trnQ\trnLam^{-1}(\hat{x} I+\trnLam^{-1})^{-1}\trnQ^T\hat\alpha$ \\
    \qquad \enspace Output $a$
\end{tabbing}
\end{algorithm}

Projection of $\hat{\alpha}$ to the ellipsoid 
$\ellipsoid{\trnQ}{\trnLam}{\gamma}$ is conducted by solving 
the following constrained optimization problem
\begin{align}\label{eq:projA}
    \arg\min_{{y}} &\|\hat{\alpha}-{y}\|_2^2,\quad\text{subject to }{y}^T\trnG^\dagger{y}\leq\gamma,\;y\in\mathcal{R}(\trnQ),
\end{align}
where $\mathcal{R}(\trnQ)$ denote the range space of $\trnQ$. We can find the solution by finding the Karush-Kuhn-Tucker condition of the Lagrangian function.
The following proposition characterizes the solution, which justifies 
Algorithm~\ref{alg:proj_G}.
\begin{proposition} The solution to~\eqref{eq:projA} is 
\begin{align*}
    \left\{\begin{array}{cl}
         \trnQ\trnQ^T\hat{\alpha}, &  \text{if }\hat{\alpha}^{T}\trnG^\dagger\hat\alpha\leq\gamma;\\
         \trnQ\trnLam^{-1}(\hat{x} I+\trnLam^{-1})^{-1}\trnQ^T\hat{\alpha},&\text{otherwise},
    \end{array}\right.
\end{align*}
where $\deftrnGinv$ is
the eigendecomposition of $\trnG^\dagger$
and $\lambda_j$ denotes
the $j$th diagonal entry of $\trnLam$,
$\hat{x}$ is the largest nonnegative solution to 
$$
  3x^2\sum_{j=1}^{\rkG{\trnG}}\lambda_j^3u_j^2-2x\sum_{j=1}^{\rkG{\trnG}}\lambda_j^2u_j^2+\sum_{j=1}^{\rkG{\trnG}}\lambda_ju_j^2-\gamma=0,  
$$
and $u=\trnQ^{T}\hat{\alpha}$.
\end{proposition}
\begin{proof}
Let $(\trnQ,\;Q_1)$ be a unitary matrix 
with $\trnQ^TQ_1=0$. 
Let $\hat u=(\trnQ,\;Q_1)^T\hat{\alpha}$,
and ${\hat z}=(\trnQ,\;Q_1)^T{y}$. Since $(\trnQ,\;Q_1)$ is unitary, we have
\[\|\hat{\alpha}-{y}\|_2^2=\|(\trnQ\;Q_1)^T(\hat{\alpha}-{y})\|_2^2=\|\hat u-\hat{z}\|_2^2.\]
Let $u=\trnQ^T\hat\alpha$ and $z=\trnQ^Ty$.
Since $Q_1^T\trnQ=0$ and we must have $Q_1^Ty=0$, the problem~\eqref{eq:projA} is equivalent to the following
\begin{align}\label{eq:projA2}
    \arg\min_{{z}} \|u-{z}\|_2^2,\quad\text{subject to }{z}^T{\trnLam}{z}\leq\gamma.
\end{align}
Letting ${w}={\trnLam}^{{1}/{2}}{z}$, we can
rewrite the objective function~\eqref{eq:projA2} as follows
\begin{align*}
    \arg\min_{{w}}\|u-\trnLam^{-\frac{1}{2}}{w}\|_2^2,\quad\text{subject to }{w}^T{w}\leq\gamma.    
\end{align*}
Let the corresponding Lagrangian function be
\begin{align*}
    \mathcal{L}({w}, x)=\|u-{\trnLam}^{-{1}/{2}}{w}\|_2^2+x({w}^T{w}-\gamma).
\end{align*}
The condition $\nabla_{w}\mathcal{L}=0$ implies that
\begin{align*}
    {w}=(x {I}+{\trnLam}^{-1})^{-1}{\trnLam}^{-{1}/{2}}u.
\end{align*} 
By the Karush–Kuhn–Tucker condition, 
if $w^Tw<\gamma$, then ${y}^*=\trnQ\trnQ^T\hat{\alpha}$.
Otherwise, ${w}^T{w}=\gamma$. This implies that 
\begin{align}\label{eq:sum_lambda}
    \sum_{j=1}^{\rkG{\trnG}}\frac{(
{u}_{j}^2\lambda_j)}{(1+x \lambda_j)^2}=\gamma.
\end{align}
Using the second order Taylor expansion, we write~\eqref{eq:sum_lambda} as
\begin{align}\label{eq:poly}
    3x^2\sum_{j=1}^{\rkG{\trnG}}\lambda_j^3{u}_{j}^2-2x\sum_{j=1}^{\rkG{\trnG}} \lambda_j^2{u}_{j}^2+\sum_{j=1}^{\rkG{\trnG}}\lambda_j{u}_{j}^2-\gamma=0.
\end{align}
Then, finding $x$ is equivalent as finding the roots of the above polynomial function. Finally, we plug $\hat{x}$, the largest nonnegative solution to~\eqref{eq:poly}, into ${y}^*={\trnQ}{\trnLam}^{-1}(\hat{x}{I}+{\trnLam}^{-1})^{-1}{\trnQ}^T\hat{\alpha}$ and complete the proof.
\end{proof}

\subsection{Expansion Coefficients of Projections to $\projV$ and $\projqrV$}

Let $v$ be a column of $V$,
$v^\star$ be a column of $V^\star$,
and let $\hat{v}$ denote projection of
$v$ to $\projV$.
Since $\projV$ is a nonconvex set, 
$\hat{v}$ may be further away from $v^\star$ compared to $v$.
We denote
$\Pi_{\projV}(V)$ as the projection operator 
that projects columns of $V$ to $\projV$.
We characterize $\rho$ such that 
\[\|\Pi_{\projV}(V)-V^\star R\|_F^2\leq\rho\|V-V^\star R\|_F^2.\]
Lemma~\ref{lemma:rho_ub} characterizes $\rho$
by combining results from Lemma~\ref{lemma:CP} and Lemma~\ref{lemma:vr}. 
Lemma~\ref{lemma:QR_V} provides a bound
on $\rho$ when an additional step to orthogonalize $V$ via QR decomposition is
performed. 

The following lemma shows the expansion coefficient of
the hard thresholding operator, 
which corresponds to the first step in Algorithm~\ref{alg:proj_v}.
\begin{lemma}[Lemma 4.1 in~\citet{li2016stochastic}]\label{lemma:CP}
Suppose that $u\in\mathbb{R}^P$ is a sparse vector such that $\|u\|_0\leq s^\star$. Let 
$\Pi_s(\cdot):\mathbb{R}^P\rightarrow\mathbb{R}^P$ 
be the hard thresholding operator,
which outputs a vector by selecting the top-$s$ entries of the
input vector in absolute value and setting the rest of the entries to $0$.
Given $s>s^\star$, for any vector ${v}\in\mathbb{R}^P$, we have
\begin{align*}
    \|\Pi_s({v})-u\|_2^2\leq \left\{1+\frac{2\surd{s^\star}}{\surd{(s-s^\star})}\right\}\|{v}-u\|_2^2.
\end{align*}

\end{lemma}

The following results characterizes 
the expansion coefficient for
the second step in Algorithm~\ref{alg:proj_v}.
\begin{lemma}\label{lemma:vr}
    Assume that ${v}^Tu\geq0$, $\|u\|_2=1$, and $\|{v}\|_2\leq 1$. 
    Then
    \[2\|{v}-u\|_2^2\geq\left\|\frac{v}{\|{v}\|_2}-u\right\|_2^2.\]
\end{lemma}
\begin{proof}[of Lemma~\ref{lemma:vr}]
    {Showing $2\|{v}-u\|_2^2\geq\|{v}/{\|{v}\|_2}-u\|_2^2$ is equivalent to
    showing}
    \begin{align*}
        {2\|{v}\|_2^2+2{v}^Tu\left(\frac{1}{\|{v}\|_2}-2\right)\geq0.}
    \end{align*}
    {Let $\cos\theta=(v^T u)/(\|v\|_2\|u\|_2)$. Then we need to show that}
    \begin{align}\label{eq:expansion2}
        {2\|v\|_2^2+2\cos\theta-4\|v\|_2\cos\theta \geq0.}
    \end{align}
{Since $(a+b)\geq2\surd(ab)$, for $a\geq 0$ and $b\geq 0$,
and $\cos^{1/2}\theta\geq\cos\theta$ for $\cos\theta\geq0$,
we have established~\eqref{eq:expansion2}.}
\end{proof}

Combining Lemma~\ref{lemma:CP} and Lemma~\ref{lemma:vr}, we obtain the following
result.
\begin{lemma} \label{lemma:rho_ub}
Consider two matrices $U, V \in \mathbb{R}^{P\times K}$, 
and assume that $v_k^Tu_k\geq 0$  and $\|{u}_k\|_0\leq s^\star$ for $k\in[K]$. Assume that $s>s^*$.
Let $\Pi_{\projV}:\mathbb{R}^{P\times K}\rightarrow\mathbb{R}^{P\times K}$ be
the projection operator that projects columns of the matrix onto the set $\projV$,
defined in Section~\ref{sec:proj_v}. Then 
    \begin{align}\label{eq:projcoef_nonortho}
        \|\Pi_{\projV}(V)-U\|_F^2\leq2\left\{1+\frac{2\surd{s^\star}}{\surd{(s-s^\star)}}\right\}\|{V}-U\|_F^2.
    \end{align}
\end{lemma}
\begin{proof}[of Lemma~\ref{lemma:rho_ub}]
    Lemma~\ref{lemma:CP} states the expansion coefficient of the first projection in Algorithm~\ref{alg:proj_v}. Similarly, Lemma~\ref{lemma:vr} states the expansion coefficient of the second projection in Algorithm~\ref{alg:proj_v} when the vector before projection has norm smaller than $1$. If the vector before projection has norm greater or equal to $1$, then the projection to the unit sphere is equivalent as the projection to the unit ball, which is a convex set. Then, the resulting projection is a contraction. By multiplying the results of two lemmas, we can obtain the expansion coefficient of projection to $\projV$ for each column vector. Stacking all the column vectors together, we obtain the result~\eqref{eq:projcoef_nonortho}.
\end{proof}

The following lemmas characterize the expansion
coefficient $\rho$ when an additional QR decomposition step is used. 
We first state a result from the perturbation theory of QR decomposition~\citep{stewart1977perturbation}. 
\begin{lemma}[Adapter from Theorem 1 in~\citep{stewart1977perturbation}]\label{lemma:QR_fac}
Let $A^\dagger$ be the pseudo inverse of a rank $K$ matrix
$A\in \mathbb{R}^{P\times K}$.
Suppose that $E\in\mathbb{R}^{P\times K}$ and $\|E\|_2\|A^\dagger\|_2<1$.
Then, given a QR decomposition of $(A+E)=BL$,
there exists a decomposition of $A=B^\star L^\star$, such that $B^\star$ has orthonormal columns and $L^\star$ is a nonsingular upper triangular matrix and
\begin{align}
\|B-B^\star\|_F\leq\frac{\surd{2}\|A^\dagger\|_2\|E\|_F}{1-\|E\|_2\|A^\dagger\|_2}\, .
\end{align}
\end{lemma}

Next, we apply Lemma~\ref{lemma:QR_fac} to our setting and establish the following lemma.
\begin{lemma}\label{lemma:QR_V} 
Let $U$ be a matrix with orthonormal columns. Let
$V\in\mathbb{R}^{P\times K}$ be a rank $K$ matrix with unit norm columns,
and $\|V-U\|_2\leq r'<1$. 
Let $V=BL$ be the QR decomposition of $V$, 
where $B\in\mathbb{R}^{P\times K}$ has orthonormal columns and 
$L\in\mathbb{R}^{K\times K}$ is an upper triangular matrix. Then
\begin{align*}
    \|B-B^\star\|_F^2\leq\frac{2}{(1-r')^2}\|V-U\|_F^2\,.
\end{align*}
\end{lemma}
\begin{proof}
 Let $E=V-U$ and $A=U$. We have $\|E\|_2\leq r'$, $\|A^\dagger\|_2=1$. The result then follows from 
 Lemma~\ref{lemma:QR_fac} as 
 \begin{align*}
     \|B-B^\star\|_F\leq\frac{\surd{2}\|V-U\|_F}{1-r'}.
 \end{align*}
\end{proof}


\section{Linear Convergence and Statistical Error}

 \subsection{Upper bound for the distance metric}\label{sssec:UBD}

We establish an upper bound on $\dist^2(Z,\ptrnZ)$ in
terms of $\{\|\Sigma_j-\ptrnCov_j\|_F^2\}_{j\in[J]}$,
which serves as an important ingredient in the analysis of 
linear convergence.

\begin{lemma}\label{lemma:UBD_ortho} 
For two matrices $V, V^\star\in\mathbb{R}^{P\times K}$ with
orthonormal columns, let
\[\defR.\]
Let $\Sigma_j=V\diag(a_j)V^T$, $\ptrnCov_j=V^\star\diag(\ptrna_j)V^{\star T}$, $\pCov_j=V^\star\diag(\pa_j)V^{\star T}$ and
$c$ be a positive constant such that $\|\diag({a_j})\|_2\leq c$ for $j\in[J]$. Suppose that $\sigma_K(\pCov_j-\ptrnCov_j)\leq1/4\sigma_K(\pCov_j)$ for $j\in[J]$, then
\begin{align*}
    &\sum_{j=1}^J\|V-V^\star\rotmat\|_F^2+\|\diag(a_j)-\rotmat^T\diag(\ptrna_j)\rotmat\|_F^2\leq\xi^2\sum_{j=1}^J\|\Sigma_{j}-\ptrnCov_j\|_F^2,
\end{align*}
where
\begin{align*}
    &\xi^2=\max_{j\in[J]}\left\{\frac{16}{\sigma_K^2(\pCov_j)}+\left(1+\frac{8c}{\sigma_K(\pCov_j)}\right)^2\right\}.
\end{align*}
\end{lemma}
\begin{proof}[of Lemma~\ref{lemma:UBD_ortho}]
We establish the result for a single $j\in[J]$.
The bound can easily be extended to the sum of all $j\in[J]$.

Since $\ptrnCov_j$ is rank $K$, 
we have $\sigma_{K+1}({\ptrnCov_j})=0$ for $j\in[J]$. Consequently,
\[\sigma_K(\ptrnCov_j)-\sigma_{K+1}(\ptrnCov_j)=\sigma_K(\ptrnCov_j)>0,\]
for $j\in[J]$ and we can
use Lemma~\ref{lemma:DKSTT} to obtain
\begin{align}\label{eq:lemmab3_1}
    \|V-V^\star\rotmat\|_F^2
    \leq\frac{8}{\sigma_K^2(\ptrnCov_j)}\|\Sigma_{j}-\ptrnCov_j\|_F^2.
\end{align}
Moreover, since 
\begin{align*}
\singv{K}{\ptrnCov_j}
\geq \singv{K}{\pCov_j}-\singv{K}{\pCov_j-\ptrnCov_j}
\geq \frac{3}{4}\singv{K}{\pCov_j},
\end{align*}
we have
\[
\|V-V^\star\rotmat\|_F^2
    \leq\frac{16}{\singvtwo{K}{\pCov_j}}\|\Sigma_{j}-\ptrnCov_j\|_F^2.
\]
Next, 
by the triangular inequality, we have
\begin{align*}
   \|\Sigma_j-\ptrnCov_j\|_F
   &=\|V\diag(a_j)V^T-V^\star\rotmat\rotmat^T\diag(\ptrna_j)\rotmat\rotmat^TV^{\star T}\|_F\notag\\
   &\geq\|V^\star\rotmat\{\diag(a_j)-\rotmat^T\diag(\ptrna_j)\rotmat\}\rotmat^TV^{\star T}\|_F\notag\\
   &\quad -\|(V-V^\star\rotmat)\diag(a_j)V^T\|_F-\|V^\star\rotmat\diag(a_j)(V-V^\star\rotmat)^T\|_F.
   \end{align*}
Since
   \begin{align*}
       \|V^\star\rotmat\{\diag(a_j)-\rotmat^T\diag(\ptrna_j)\rotmat\}\rotmat^TV^{\star T}\|_F=\|\diag(a_j)-\rotmat^T\diag(\ptrna_j)\rotmat\|_F
   \end{align*}
and
\begin{align}\label{eq:aux_facts}
\|V\|_2+\|V^\star\rotmat\|_2=2,
\end{align}
we further have
\begin{align*}
   \|\Sigma_j-\ptrnCov_j\|_F&\geq  \|\diag(a_j)-\rotmat^T\diag(\ptrna_j)\rotmat\|_F\notag\\
   &\quad -\|(V-V^\star\rotmat)\|_F\left\{\|\diag(a_j)V^T\|_2+\|V^\star\rotmat\diag(a_j)\|_2\right\}\notag\\
   &\geq \|\diag(a_j)-\rotmat^T\diag(\ptrna_j)\rotmat\|_F-2\|\diag(a_j)\|_2\|V-V^\star R\|_F.
\end{align*}
Therefore,
\begin{align*}
    \|\diag(a_j)-\rotmat^T\diag(\ptrna_j)\rotmat\|_F\leq\|\Sigma_j-\ptrnCov_j\|_F+2\|\diag(a_j)\|_2\|V-V^\star R\|_F.
\end{align*}
Combining \eqref{eq:lemmab3_1} and $\|\diag(a_j)\|_2\leq c$, we have
\begin{align}\label{eq:lemmab3_2}
    \|\diag(a_j)-\rotmat^T\diag(\ptrna_j)\rotmat\|_F\leq \left(1+\frac{8c}{\singv{K}{\pCov_j}}\right)\|\Sigma_j-\ptrnCov_j\|_F.
\end{align}
The proof is complete by combining 
\eqref{eq:lemmab3_1} and \eqref{eq:lemmab3_2}.
\end{proof}
\subsection{Proof of Theorem~\ref{theorem:lineardist}}

We prove Theorem~\ref{theorem:lineardist} in several steps.
First, we show that given a current iterate $Z$, 
which satisfies suitable assumptions, the
subsequent iterate $Z^+$ obtained by Algorithm~\ref{alg:main}
with a suitable step size
satisfies
\begin{align*}
     \dist^2(Z^+,\ptrnZ)\leq\beta^{1/2}\dist^2(Z,\ptrnZ)+C_1\staterr^2,
\end{align*}
with $0<\beta^{1/2}<1$ and some constant $C_1$. 
Second, we show that the step size can be chosen in a way that does not depend on 
the specific iterate.
Finally, the lemma follows by applying the first step of the proof $I$ times 
starting from $Z^0$.

We start by introducing some additional notation for simplicity of presentation.
We define
    \begin{align*}
        {Z}=\begin{pmatrix}{V}\\{A}^T\end{pmatrix},
        \quad{Z}^+=\begin{pmatrix}{V}^+\\{A}^{+ T}\end{pmatrix},
        \quad\ptrnZ=\begin{pmatrix}{V}^\star\\\tilde{A}^{\star T}\end{pmatrix},
    \end{align*}
where $Z$ is the current iterate, 
$Z^+$ is the iterate obtained by one step
of Algorithm~\ref{alg:main} starting from $Z$,
and $\ptrnZ$ is the truncated version of the ground truth parameter $\pZ$. 
Furthermore, let
\[\defR,\quad\defRplus.\]
be the optimal rotation matrices in the current and subsequent step.   

Let $\Pi_{\projqrV}(X)$ be the projection operator
defined in~\eqref{eq:projqrv}. Let $\Pi_{\projtrnA}(Y)$ be the projection operator 
that projects rows of $Y$ to $\projtrnA$, given
in Algorithm~\ref{alg:proj_a}. 
One update of Algorithm~\ref{alg:main} can be written as
\begin{align}
 V^+=\Pi_{\projqrV}(V-\eta_V\nabla_Vf_N),\quad A^+=\Pi_{\projtrnA}(A-\eta_A\nabla_Af_N),\label{eq:grad_update}
\end{align}
where
\begin{align}\label{eq:gradientva}
    \nabla_V f_N(Z) =\frac{2}{J}\sum_{j=1}^J\graell {V}\diag({a}_j),
    \quad
    \nabla_{A} f_N(Z)=\frac{1}{J}W(V),
\end{align}
with 
\[
\defgradAwrtcov{V}{v}.
\]
Similarly, we define
\[
\defgradAwrtpcov{V}{v}.
\]

Let $\suppU = \mathcal{S}(V)\cup\mathcal{S}(V^+)\cup\mathcal{S}(V^\star)$ 
and note that $|\suppU|\leq 2s+s^\star$. 
Given the index set $\suppU$, we write 
$[X]_{\suppU}$ to denote the projection 
of $X$ to the support $\suppU$
\begin{align*}
    [X]_{\suppU}=\left\{\begin{array}{cc}
         X_{ij}&  (i,j)\in\suppU\\
         0 & (i,j)\not\in\suppU
    \end{array}\right..
\end{align*}
With some abuse of notation,
given a matrix $Y$ with the factored form $Y=XX^T$, 
we write
\[
[Y]_{\suppU,\suppU}=[X]_{\suppU}[X]_{\suppU}^T.
\]
With this notation, we have
\begin{align}
    V^+=&\Pi_{\projqrV}(V-\eta_V\nabla_Vf_N)=\Pi_{\projqrV}(V-\eta_V\left[\nabla_Vf_N\right]_{\suppU}).
    \label{eq:proj_grav}
\end{align}
Furthermore, recall that
$\trnQ$ is the matrix whose columns are eigenvectors of $\trnG$. Then $\trnQ\trnQ^T$ is the projection operator to the subspace spanned by the columns of $\trnQ$.
Since the output of Algorithm~\ref{alg:proj_G}
is in the range space of $\trnQ$, we have that   $A\trnQ\trnQ^T=A$ and
\begin{align}
    A^+=\Pi_{\projtrnA}(A-\eta_A\nabla_Af_N)
    =\Pi_{\projtrnA}\{(A-\eta_A \nabla_Af_N)\trnQ\trnQ^T\}
    =\Pi_{\projtrnA}\{A-\eta_A (\nabla_Af_N\trnQ\trnQ^T)\}
    \label{eq:proj_graa}.
\end{align}
For later convenience, we also
note that for a rotation matrix $\rotmat\in\mathcal{O}(K)$,
we have
\begin{align}\label{eq:innergradientv}
    \langle\nabla_V f_N(Z), V-V^\star\rotmat\rangle=\frac{2}{J}\sum_{j=1}^J\langle\graell ,{V}\diag({a}_j){V}^T-{V}^\star\rotmat\diag({a}_j){V}^T\rangle
\end{align}
and 
\begin{multline}
\label{eq:innergradienta}
\langle\diag\{(\nabla_{A} f_N)_j\},\diag({a}_j)-\rotmat^T\diag(\ptrna_j)\rotmat\rangle
\\
=\frac{1}{J}\langle\graell ,{V}\diag({a}_j){V}^T-{V}{R}^T\diag(\ptrna_j){R}{V}^T\rangle.
\end{multline}

With this notation, we are ready to state the result of 
the first step of the proof. 
\begin{lemma}\label{lemma:MAIN}
Suppose that $Z$ satisfies
\begin{equation}
\label{eq:assum:3:condition}
d^2(Z_j, Z_j^\star) \leq \iniC^2,\quad
\norm{V-V^\star \rotmat}_F\leq \iniC^2 / J,
\quad
\|\diag(a_j)-\rotmat^T\diag(\pa_j)\rotmat\|_F^2\leq (J-1)\iniC^2/J,
\end{equation}
where $\iniC^2$ is given in~\eqref{eq:I0xi2}.
Furthermore, suppose Assumption~\ref{assumption_stepsize}, \ref{assumption_para}, and
\ref{assumption_statcondition} hold.
Then  
 \begin{align*}
    {\dist}^2({ Z}^+,\ptrnZ) \leq
    \beta^{1/2}{\dist}^2({ Z},\ptrnZ)+\tau\beta^{-1/2}\eta\staterr^2,
\end{align*}
where  $Z^+$ is obtained with
one iteration of Algorithm~\ref{alg:main} starting from $Z$ and $\deftau$.
\end{lemma}

Note that $d^2(Z_j\ptrnZ_j)\leq d^2(Z_j,\pZ_j)$ for $j\in[J]$
and $\trnerror$ quantifies how close $\pZ$ is to $\ptrnZ$. 
Therefore we do not need additional assumptions for $\ptrnZ$.


Starting from $Z^0$,
which satisfies Assumption~\ref{assumption_inicondition}
(that is also restated in \eqref{eq:assum:3:condition}),
we show that $Z^1$ also satisfies \eqref{eq:assum:3:condition}.
Therefore, we can apply Lemma~\ref{lemma:MAIN}
over $I$ iterations to obtain Theorem~\ref{theorem:lineardist}.


\begin{proof}[of Theorem~\ref{theorem:lineardist}]
When we apply one iteration of Algorithm~\ref{alg:main}, Lemma~\ref{lemma:MAIN} gives us  \begin{align}\label{eq:mainineq}
    {\dist}^2(Z^+,\ptrnZ) \leq
        \beta^{1/2}{\dist}^2({ Z},\ptrnZ)+\tau\beta^{-1/2}\eta\staterr^2.
\end{align}
Under Assumption~\ref{assumption_statcondition}, 
the right hand side of~\eqref{eq:mainineq} is bounded by $J\iniC^2$. 
This implies that the new estimate is still in a good region
where we can apply Lemma~\ref{lemma:MAIN}.
That is, $Z^+$ satisfies \eqref{eq:assum:3:condition}. 
Consequently, 
since $Z^0$ satisfies Assumption~\ref{assumption_inicondition}
and, therefore, equation \eqref{eq:assum:3:condition}, 
we can apply the result of Lemma~\ref{lemma:MAIN} for $I$ iterations to obtain
\begin{align*}
    \dist^2(Z^I,\ptrnZ)\leq\beta^{I/2}\dist^2(Z^0,\ptrnZ)+\frac{\tau\eta\staterr^2}{\beta^{1/2}-\beta}.
\end{align*}
\end{proof}
\subsection{Proof of Lemma~\ref{lemma:MAIN}}

\begin{proof}[of Lemma~\ref{lemma:MAIN}]
 Recall that
    \begin{align}\label{eq:pre_main}
        &\dist^2({Z}^+,\ptrnZ)=
        \sum_{j=1}^J\|{V}^+-{V}^\star\rotmat^{+}\|_F^2+\|\diag({a}^+_j)-\rotmat^{+ T}\diag(\ptrna_j)\rotmat^+\|_F^2.
        \end{align}
We bound the two terms on the right hand side of~\eqref{eq:pre_main} separately. 
From the triangle inequality, we have
\begin{multline*}
    \|\diag({a}^+_j)-\rotmat^{+ T}\diag(\ptrna_j)\rotmat^+\|_F^2
    \leq
    2\|\diag({a}^+_j)-\rotmat^T\diag(\ptrna_j)\rotmat\|_F^2 \\
    +2\|\rotmat^T\diag(\ptrna_j)\rotmat-\rotmat^{+T}\diag(\ptrna_j)\rotmat^+\|_F^2.   
\end{multline*}
By Lemma~\ref{lemma:rotation_bound},  
$\|\rotmat^T\diag(\ptrna_j)\rotmat-\rotmat^{+T}\diag(\ptrna_j)\rotmat^+\|_F
\leq4\|\diag(\ptrna_j)\|_2\|V^+-V^\star\rotmat\|_F$ and 
\begin{multline}
\label{eq:midstep_a}
\sum_{j=1}^J
\|\diag({a}^+_j)-\rotmat^{+ T}\diag(\ptrna_j)\rotmat^+\|_F^2
\\
\leq
32J\|\ptrnA\|^2_\infty\|{V}^+-{V}^\star\rotmat\|_F^2
+2\sum_{j=1}^J\|\diag({a}^+_j)-\rotmat^T\diag(\ptrna_j)\rotmat\|_F^2.
\end{multline}
Combining~\eqref{eq:midstep_a} with~\eqref{eq:pre_main}
and recalling the definition of $V^+$ and $A^+$
from~\eqref{eq:proj_grav} and \eqref{eq:proj_graa}
we have
\begin{multline}
    \dist^2(Z^+,\ptrnZ)
    \leq
    J\kappa\|\Pi_{\projqrV}({V}-\eta_V[{\nabla_Vf_N}]_{\suppU})-{V}^\star\rotmat\|_F^2\\
        +2\sum_{j=1}^J\|\diag\left[\Pi_{\projtrnA}\{{A}-\eta_A(\nabla_{A}f_{N}\trnQ\trnQ^T)\}_j\right]-\rotmat^T\diag(\ptrna_j)\rotmat\|_F^2.\label{eq:mid_dist}
\end{multline}
where $\kappa=(1+32\|\ptrnA\|^2_\infty)$.
Next, we define
\[
\bar{V}^+=\Pi_{\projV}(V-\eta_V[\nabla_V f_N]_{\suppU}),\]
where we recall that $\Pi_{\projqrV}(\cdot)$ is the projection operator by first applying $\Pi_{\projV}(\cdot)$ followed by a QR decomposition step.  
By Lemma~\ref{lemma:distbarv}, we have $\|\bar{V}^+-V^\star\rotmat\|_2\leq2\iniC/J^{1/2}$. Therefore, we can apply Lemma~\ref{lemma:QR_V} with $U=V^\star\rotmat=B^\star L^\star$, where $B^\star=V^\star\rotmat$ and a nonsingular matrix $L^\star=I$. Then, 
\begin{align*}
\|\Pi_{\projqrV}({V}-\eta_V[{\nabla_Vf_N}]_{\suppU})-{V}^\star\rotmat\|_F^2
&\leq \frac{2}{(1-\frac{2\iniC}{J^{1/2}})^2}\|\bar{V}^+-{V}^\star\rotmat\|_F^2\\
&= \frac{2}{(1-\frac{2\iniC}{J^{1/2}})^2}\|\Pi_{\projV}({V}-\eta_V[{\nabla_Vf_N}]_{\suppU})-{V}^\star\rotmat\|_F^2.
\end{align*}

By the fact $\|V-V^\star\rotmat\|_F\leq \iniC/{J}^{1/2}$ and Lemma~\ref{lemma:gradbound}, we have
 \begin{equation}\label{eq:somebound}
 \|{V}-\eta_V[{\nabla_Vf_N}]_{\suppU}-{V}^\star\rotmat\|_F\leq\|V-V^\star\rotmat\|_F+\|\eta_V[{\nabla_Vf_N}]_{\suppU}\|_F\leq\frac{7\iniC}{6J^{1/2}}<1.
\end{equation}
Since columns of ${V}^\star\rotmat$ are unit norm, the result in~\eqref{eq:somebound} implies that the inner product of the $k$th column of ${V}-\eta_V[{\nabla_Vf_N}]_{\suppU}$ and the $k$th column of $V^\star\rotmat$ is nonnegative for every $k\in[K]$. Therefore, we can apply Lemma~\ref{lemma:rho_ub} and the further bound $\|\Pi_{\projqrV}({V}-\eta_V[{\nabla_Vf_N}]_{\suppU})-{V}^\star\rotmat\|_F^2$ as
\begin{align*}
\|\Pi_{\projqrV}({V}-\eta_V[{\nabla_Vf_N}]_{\suppU})-{V}^\star\rotmat\|_F^2&\leq\rho \|{V}-\eta_V[{\nabla_Vf_N}]_{\suppU}-{V}^\star\rotmat\|_F^2,
\end{align*}
where $\rho = 4(1-r')^{-2}\{1+2\surd{s^\star}/\surd(s-s^\star)\}$ and $r'=2\iniC/J^{1/2}$.
By the contraction property of projection to convex sets, we have
\begin{multline*}
\sum_{j=1}^J\left\|\diag\left[\Pi_{\projtrnA}\{A-\eta_A(\nabla_{A}f_{N}\trnQ\trnQ^T)\}_j\right]-\rotmat^T\diag(\ptrna_j)\rotmat\right\|_F^2
\\
\leq
\sum_{j=1}^J\|\diag({a}_j)-\eta_A\diag\{(\nabla_{A}f_{N}\trnQ\trnQ^T)_{j}\}-\rotmat^T\diag(\ptrna_j)\rotmat\|_F^2,
\end{multline*}
where $(\nabla_{A}f_{N}\trnQ\trnQ^T)_{j}$ denotes the $j$th column of $\nabla_{A}f_{N}\trnQ\trnQ^T\in\RR^{K\times J}$. Combining the last two displays and noting that $\rho' = \rho \kappa > 2$,
we have 
\begin{multline}
\label{eq:main}
    \dist^2(Z^+,\ptrnZ)
    \leq
    \rho' \bigg[
    J \|{V}-\eta_V[{\nabla_Vf_N}]_{\suppU}-{V}^\star\rotmat\|_F^2 \\
    +
    \sum_{j=1}^J \|\diag({a}_j)-\eta_A\diag\{(\nabla_{A}f_{N}\trnQ\trnQ^T)_{j}\}-\rotmat^T\diag(\ptrna_j)\rotmat\|_F^2 \bigg].
\end{multline}
Recall that  $\eta_V={\eta}/{J}$ and $\eta_A=\eta$.
Therefore
\begin{align}\label{eq:main2}
    \dist^2(Z^+,\ptrnZ) \leq \rho'{\dist}^2({Z}, \ptrnZ) +\eta^2\rho'( B1 + B2 )- \eta\rho'(A1 + A2),
\end{align}
where 
\begin{align*}
    A1&={2}\langle[\nabla_Vf_N]_{\suppU},{V}-{V}^\star\rotmat\rangle, \\
    A2&=2\sum_{j=1}^J\langle\diag\{(\nabla_{A}f_{N}\trnQ\trnQ^T)_{j}\},\diag({a}_j)-\rotmat^T\diag(\ptrna_j)\rotmat\rangle, \\
    B1&=J^{-1}\|[\nabla_Vf_N]_{\suppU}\|_F^2,\\
    B2&=\sum_{j=1}^J\|\diag\{(\nabla_{A}f_{N}\trnQ\trnQ^T)_{j}\}\|_F^2.
\end{align*}

Next, we upper bound $B=B1+B2$ and lower bound $A=A1+A2$
in Lemma~\ref{lemma:A} and Lemma~\ref{lemma:B}, respectively. 
With these bounds, we will be able to 
show contraction $\dist^2({Z}^+,\ptrnZ)$ with respect to $\dist^2({Z},\ptrnZ)$. 
    
\begin{lemma}\label{lemma:A}
Under same conditions of Lemma~\ref{lemma:MAIN}, we have
\begin{multline*}
    A\geq \frac{1}{J}\Bigg\{\frac{3}{4}\sum_{j=1}^J\|\Sigma_j-\ptrnCov_j\|_F^2-\frac{9}{2}\staterr^2+\frac{1}{2}\sum_{j=1}^J\left\|\left[\graell -\graellstar \right]_{\suppU,\suppU}\right\|_F^2\notag\\
    -8\left(1+\frac{\|\pA\|_\infty^2}{J}\right)\iniC^2\dist^2(Z,\ptrnZ)\Bigg\}.
\end{multline*}

\end{lemma}
\begin{lemma}\label{lemma:B}
Under same conditions of~\ref{lemma:MAIN}, we have
\[B\leq\frac{ 16}{J^2}\sum_{j=1}^J\left\{\|\graell -\graellstar \|_F^2\right\}\|{Z}_j\|_2^2+\frac{4(4\vee K)}{J^2}\staterr^2\max_{j\in[J]}\|Z_j\|_2^2.\]
\end{lemma}

Using Lemma~\ref{lemma:A} and Lemma~\ref{lemma:B}, 
we have
\begin{equation*}
\begin{aligned}
    \eta\rho'A-\eta^2\rho'B&\geq
    \frac{\rho'\eta}{J}
    \underbrace{\cbr{\frac{3}{4}\sum_{j=1}^J\|{\Sigma}_j-\ptrnCov_j\|_F^2-8\left(1+\frac{\|\pA\|_\infty^2}{J}\right)\iniC^2\dist^2({Z},\ptrnZ)}}_{C1}    \\
    &\quad-{\rho'\eta}\staterr^2
    \underbrace{
    \rbr{
    \frac{9}{2J}+\frac{4(4\vee K)\eta}{J^2}\max_{j\in[J]}\|{Z}_j\|_2^2
    }}_{C2}\\
    &\quad+\frac{\eta\rho'}{J}\underbrace{\left(\frac{1}{2}-16\frac{\eta}{J}\max_{j\in[j]}\|{Z}_j\|_2^2\right)}_{C3}\sum_{j=1}^J\|\graell -\graellstar \|_F^2.
\end{aligned}
\end{equation*}
Under the assumption that $\deftrnerrorinq$ and the inequality in~\eqref{eq:errortrnA},
we have
\begin{align*}
    \sigma_K(\pCov_j-\ptrnCov_j)=\sigma_K(V^\star\diag(\pa_j-\ptrna_j)V^{\star T}) &= \min_{k\in[K]}\abr{\pa_{jk}-\ptrna_{jk}}\\
    &\leq \max_{k\in[K]}\|\ptrnA_{k\cdot}-\pA_{k\cdot}\|_2\leq(\trnerror\gamma^\star)^{1/2}\leq\frac{1}{4}\sigma_K(\pCov_j)
\end{align*}
for $j\in[J]$. Therefore, using 
Lemma~\ref{lemma:UBD_ortho}, we have
\begin{align*}
    \sum_{j=1}^J\|{\Sigma}_j-\ptrnCov_j\|_F^2\geq\frac{1}{\xi^2}\dist^2({Z},\ptrnZ).
\end{align*}
Furthermore, from the definition of $\iniC^2$ in Assumption~\ref{assumption_inicondition},
we have that 
\[
8\iniC^2\left(1+\frac{\|\pA\|_\infty^2}{J}\right)\leq\frac{1}{2\xi^2}.
\]
Therefore,
combining the last two displays, we arrive at
\begin{align*}
    C1\geq\frac{1}{4\xi^2}\dist^2({Z},\ptrnZ).
\end{align*}
In fact, we can verify that $C3$ is nonnegative because the step size $\eta$ that satisfies Assumption~\ref{assumption_stepsize} is small enough such that the following inequality holds.
\begin{lemma}\label{lemma:CSS}
Under the conditions of Lemma~\ref{lemma:MAIN}
we have $\|{Z}_j\|_2^2\leq 2 \|{Z}_j^0\|_2^2$.
\end{lemma}
Therefore,
\[
\eta\leq\min_{j\in[J]}\frac{J}{32\|Z_j\|^2_2}
\]
and $C3\geq 0$ can be omitted,
while $C2\leq\tau$. Then

\begin{align*}
    \eta\rho'A-\eta^2\rho'B
    \geq\frac{\rho'\eta}{J}\frac{1}{4\xi^2}\dist^2({Z},\ptrnZ)
    -\staterr^2\tau\rho'\eta.
\end{align*}
Under Assumption~\ref{assumption_para}, it is easy to verify that $\rho'\leq\beta^{-1/2}$, where $\defbeta$. Plugging into~\eqref{eq:main2}, we have
\begin{align*}
    \dist^2({Z}^+,\ptrnZ)
    \leq 
    \beta^{1/2}\dist^2({Z},\ptrnZ)
    +\tau\beta^{-1/2}\eta\staterr^2,
\end{align*}
which completes the proof. 
\end{proof}

\subsection{Proofs of Lemma~\ref{lemma:A}--\ref{lemma:CSS}}
\begin{proof}[of Lemma~\ref{lemma:A}]
Using~\eqref{eq:innergradientv} and
$\langle[\graell V\diag({a}_j)]_{\suppU^c},[{V}-{V}^\star\rotmat]_{\suppU}\rangle\notag=0$,
we have
\begin{align*}
A1 &= \frac{4}{J}\sum_{j=1}^J \langle [\graell V\diag({a}_j)]_{\suppU}, [{V}-{V}^\star\rotmat]_{\suppU} \rangle\notag\\
&= \frac{4}{J}\sum_{j=1}^J \langle [\graell V\diag({a}_j)]_{\suppU}+[\graell V\diag({a}_j)]_{\suppU^c}, [{V}-{V}^\star\rotmat]_{\suppU} \rangle\notag\\
&= \frac{4}{J}\sum_{j=1}^J \langle \graell V\diag({a}_j), [{V}-{V}^\star\rotmat]_{\suppU} \rangle\notag\\
&= \frac{4}{J}\sum_{j=1}^J \langle \graell, [{V}-{V}^\star\rotmat]_{\suppU}\diag({a}_j)V^T \rangle\notag\\
&= \frac{4}{J}\sum_{j=1}^J \langle\graell , [{V}-{V}^\star\rotmat]_{\suppU} [\diag({a}_j)V^T]_{\suppU} \rangle\notag\\
&= \frac{4}{J}\sum_{j=1}^J \langle\graell , [{V}\diag({a}_j){V}^T-{V}^\star\rotmat\diag({a}_j){V}^T]_{\suppU,\suppU}\rangle.
\end{align*}
Furthermore, we can write $A1$ as
\begin{align}
    &A1=A13+
    \frac{4}{J} \sum_{j=1}^J \langle \graell , [\overbrace{(V-V^\star\rotmat}^{\Delta V})\rotmat^T\diag(\ptrna_j)\rotmat V^T]_{\suppU,\suppU} \rangle;\label{eq:A1}\\
    &A13 = \frac{4}{J}\sum_{j=1}^J \langle\graell , [(\overbrace{{V}-{V}^\star\rotmat}^{\Delta{V}}) \overbrace{\{\diag({a}_j)-\rotmat^T\diag(\ptrna_j)\rotmat\}}^{\Delta{a}_j}{V}^T]_{\suppU,\suppU} \rangle.\notag
\end{align}
We also write $A2$ in a suitable way.
Note that 
$\diag\{\rotmat^T\diag(\ptrna_j)\rotmat\}=\diag(H{a}_j^\star)$,
where $H = (h_{ij})_{i \in [K], j \in[K]}$ with
$h_{ij} = r_{ji}^2$ and $r_{ij}$ is the 
$ij$th entry of $\rotmat$.
Then 
\begin{align*}
    A2 
    &=
    2\sum_{j=1}^J\langle\diag\{(\nabla_{A}f_{N}\trnQ\trnQ^T)_{j}\},\diag({a}_j)-\rotmat^T\diag(\ptrna_j)\rotmat\rangle \\
    & = 2\langle \nabla_Af_N\trnQ\trnQ^T,A-H\ptrnA\rangle \\
    & = 2\langle \nabla_Af_N,A-H\ptrnA\rangle\\
    & = 2\sum_{j=1}^J\langle\diag\cbr{(\nabla_{A}f_{N})_j},\diag({a}_j)-\rotmat^T\diag(\ptrna_j)\rotmat\rangle,
\end{align*}
since rows of $A$ and $\ptrnA$ belong to the subspace spanned by eigenvectors of $\trnG$
and $A\trnQ\trnQ^T=A$ and $\ptrnA \trnQ\trnQ^T=\ptrnA$.
Finally, using~\eqref{eq:innergradienta}, we have
\begin{align}
A2 &= \frac{2}{J}\sum_{j=1}^J \langle \graell , [{V}\diag({a}_j){V}^T - {V}\rotmat^T\diag(\ptrna_j)\rotmat{V}^T]_{\suppU,\suppU} \rangle\notag\\
& = A11 + A12 + \frac{2}{J}\sum_{j=1}^J \langle \graell , [\ptrnCov_j - V\rotmat^T\diag(\ptrna_j)\rotmat V^T]_{\suppU,\suppU}\rangle \label{eq:last_A2}\\
A11 &= \frac{2}{J}\sum_{j=1}^J \langle \graell -\graellstar ,  [{\Sigma}_j-\ptrnCov_j]_{\suppU,\suppU}\rangle\notag;\\
A12 &=\frac{2}{J}\sum_{j=1}^J \langle\graellstar , [{\Sigma}_j-\ptrnCov_j]_{\suppU,\suppU}\rangle.\notag
\end{align}
Combining~\eqref{eq:A1} and~\eqref{eq:last_A2}, we obtain
\begin{align*}
    &A = A11 + A12 + A13 + A14;\\
    &A14 = \frac{2}{J}\sum_{j=1}^J \langle\graell , [ \overbrace{({V}-{V}^\star\rotmat)}^{\Delta{V}} \rotmat^T\diag(\ptrna_j)\rotmat \overbrace{({V}-{V}^\star\rotmat)^T}^{\Delta{V}^T}]_{\suppU,\suppU}\rangle,
\end{align*}
where we have used that 
\begin{multline*}
\frac{4}{J}\sum_{j=1}^J \langle\graell , [\Delta V \rotmat^T\diag(\ptrna_j)\rotmat V^T]_{\suppU,\suppU} \rangle \\
= \frac{2}{J}\sum_{j=1}^J \langle\graell , [\Delta V \rotmat^T\diag(\ptrna_j)\rotmat V^T]_{\suppU,\suppU}\rangle\\
    + \frac{2}{J}\sum_{j=1}^J \langle \graell , [V \rotmat^T\diag(\ptrna_j)\rotmat \Delta V^T]_{\suppU,\suppU}\rangle,
\end{multline*}
since $\graell$ for $j\in[J]$ is symmetric. 
Next, we lower bound $A11$, $A12$, $A13$, and $A14$ separately.

Recall that $\Sigma_j=V\diag(a_j)V^T$ and $\ptrnCov_j=V^{\star}\diag(\ptrna_j)V^{\star T}$. 
Additionally, since $[V]_{\suppU^c}=[V^\star]_{\suppU^c}=0$, 
\begin{align*}
    &[\Sigma_j - \ptrnCov_j ]_{\suppU,\suppU^c} = [ \Sigma_j - \ptrnCov_j ]_{\suppU,\suppU^c} = [ \Sigma_j - \ptrnCov_j ]_{\suppU^c,\suppU^c}=0,
\end{align*}
for every $j\in[J]$, 
and therefore $[\Sigma_j-\ptrnCov_j]_{\suppU,\suppU}=\Sigma_j-\ptrnCov_j$. 
Then
\begin{multline*}
A11
=\frac{2}{J}\sum_{j=1}^J\langle\graell -\graellstar , {\Sigma}_j-\ptrnCov_j\rangle\\
\geq\frac{1}{J}\sum_{j=1}^J\left\{\|{\Sigma}_j-\ptrnCov_j\|_F^2+\|\graell -\graellstar \|_F^2\right\},
\end{multline*}
where we applied Lemma~\ref{lemma:convexbound} 
with $m=L=1$.
For $A12$, we have
\begin{align*}
        \text{A12}& \geq - \frac{2}{J}\left|\sum_{j=1}^J \langle\graellstar , [{\Sigma}_j-\ptrnCov_j]_{\suppU,\suppU}\rangle\right|.
\end{align*}
Since
$\{[\Sigma_j-\ptrnCov_j]_{\suppU,\suppU}\}_{j\in[J]}\in\dset{2K}{2s+s^\star}{2\gamma}{\trnerror}$, 
we have 
\begin{align*}
    A12 & \geq - \frac{2}{J}\staterr \left( \sum_{j=1}^J \|[{\Sigma}_j - \ptrnCov_j]_{\suppU,\suppU}\|_F^2 \right)^{1/2}
     \geq - \frac{2}{J} \left(\frac{\staterr^2}{2e_1}+\frac{e_1}{2}\sum_{j=1}^J \|{\Sigma}_j-\ptrnCov_j\|_F^2\right),
\end{align*}
where the last inequality follows by Young's inequality 
$ab\leq{a^2}/{(2\varepsilon)}+{(\varepsilon b^2)}/{2}$ for every $\varepsilon>0$.
We will use this bound with $e_1=1/4$.
For $A13$, we have
\begin{multline}
A13 \geq -\frac{4}{J} \left| \sum_{j=1}^J\langle\graellstar , [\Delta{V}{\Delta{a}_j}V^T]_{\suppU,\suppU}\rangle\right|
\\
    -\frac{4}{J}\sum_{j=1}^J\left|\langle\graell -\graellstar , [{\Delta{V}{\Delta{a}_j}V^T}]_{\suppU,\suppU}\rangle\right|\label{eq:A13fi}.
\end{multline}
We first bound the second term on the right hand side of~\eqref{eq:A13fi}.
Applying the Cauchy-Schwarz inequality and 
using $\|V\|_2=1$, we have 
\begin{multline*}
    -\frac{4}{J}\sum_{j=1}^J|\langle\graell -\graellstar , [{\Delta{V}{\Delta{a}_j}V^T}]_{\suppU,\suppU}\rangle|
    \\
    \geq -\frac{4}{J}\sum_{j=1}^J\|\graell -\graellstar \|_F\|{\Delta{a}_j}\|_F\|\Delta{V}\|_F.
\end{multline*}    
Using the fact that 
$\|\Delta{V}\|_F\|\Delta{a}_j\|_F\leq \frac{1}{2}d^2({Z}_j, \ptrnZ_j)$, 
the above display can be further lower bounded as
\begin{align}
\label{eq:midmidA13}
    \geq -\frac{2}{J}\sum_{j=1}^J\|\graell -\graellstar \|_F\text{d}^2({Z}_j,\ptrnZ_j).
\end{align}
Since $\{[\Delta{V}{\Delta{a}_j}V^T]_{\suppU,\suppU}\}_{j\in[J]}\in\dset{2K}{2s+s^\star}{2\gamma}{\trnerror}$, 
we can bound the first term of~\eqref{eq:A13fi} as
\begin{align}
    -\frac{4}{J}\left|\sum_{j=1}^J\langle\graellstar ,[\Delta{V}{\Delta{a}_j}V^T]_{\suppU,\suppU}\rangle\right|&\geq-\frac{4}{J}\staterr\left(\sum_{j=1}^J\|V\|_2^2\|\Delta a_j\|_F^2\|\Delta V\|_F^2\right)^{1/2}\notag\\
    &=-\frac{4}{J}\staterr\left(\sum_{j=1}^J\|\Delta a_j\|_F^2\|\Delta V\|_F^2\right)^{1/2}\notag\\
    &\geq -\frac{4}{J}\staterr\left(\frac{1}{4}\sum_{j=1}^Jd^4(Z_j,\ptrnZ_j)\right)^{1/2}\label{eq:mid_A13}, 
\end{align}
where the last inequality uses that
$\|\Delta{V}\|_F\|\Delta{a}_j\|_F\leq \frac{1}{2}d^2({Z}_j, \ptrnZ_j)$. 
Combining~\eqref{eq:midmidA13} and~\eqref{eq:mid_A13},
we have
\begin{align*}
    A13&\geq-\frac{2}{J}\sum_{j=1}^J\|\graell -\graellstar \|_F\cdot\text{d}^2({Z}_j,\ptrnZ_j)-\frac{4}{J}\staterr\left(\frac{1}{4}\sum_{j=1}^Jd^4(Z_j,\ptrnZ_j)\right)^{1/2}.
    \intertext{Applying Young's inequality with $e_2=4$, the above display can be bounded as}
    A13&\geq-\frac{1}{J}\sum_{j=1}^J\left\{\frac{1}{e_2}\|\graell -\graellstar \|_F^2+e_2\text{d}^4({Z}_j,\ptrnZ_j)\right\}\\
    &\quad-\frac{1}{J}\left\{\frac{1}{e_2}\staterr^2+e_2\sum_{j=1}^Jd^4(Z_j,\ptrnZ_j)\right\}\\
   &\geq-\frac{1}{e_2J}\left\{\staterr^2+\sum_{j=1}^J\|\graell -\graellstar \|_F^2\right\}-\frac{2 e_2}{J}\sum_{j=1}^J\iniC^2\text{d}^2({Z}_j,\ptrnZ_j),
\end{align*}
where the last inequality follows by $d^2({Z}_j,\ptrnZ_j)\leq d^2({Z}_j,\pZ_j)\leq\iniC^2$. 

A lower bound for $A14$ 
can be obtained in a similar way to the one for $A13$.
We have
\begin{multline*}
   A14\geq
   -\frac{2}{J}\sum_{j=1}^J\|\graell -\graellstar \|_F\|\Delta V\|_F^2\|\diag(\ptrna_j)\|_F \\
   -\frac{2}{J}\staterr\left(\sum_{j=1}^J\|\Delta V\|_F^4\|\diag(\ptrna_j)\|_2^2\right)^{1/2}.
\end{multline*}   
Applying Young's inequality with $e_3=4>0$, 
the above display can be bounded as   
   \begin{align*}
   A14&\geq -\frac{1}{e_3J}\left\{\staterr^2+\sum_{j=1}^J\|\graell -\graellstar \|_F^2\right\}-\frac{2 e_3}{J}\sum_{j=1}^J\|\Delta V\|_F^4\|\diag(\ptrna_j)\|_2^2\\
    &\geq -\frac{1}{e_3J}\left\{\staterr^2+\sum_{j=1}^J\|\graell -\graellstar \|_F^2\right\}-\frac{2 e_3}{J^2}\sum_{j=1}^J\|\pA\|_\infty^2\iniC^2\text{d}^2({Z}_j, \ptrnZ_j),
\end{align*}
where the last inequality is followed by $\|\diag{(\pa_j)}\|_2\leq \|\pA\|_\infty$ and $\|\Delta{V}\|_F^4\leq J^{-1}\iniC^2\text{d}^2({Z}_j, \ptrnZ_j)$.

Putting everything together, we have
\begin{multline}
    A
    \geq \frac{1}{J}\Bigg\{\frac{3}{4}\sum_{j=1}^J\|\Sigma_j-\ptrnCov_j\|_F^2-\frac{9}{2}\staterr^2+\frac{1}{2}\sum_{j=1}^J\|\graell -\graellstar \|_F^2 \\
    -8\left(1+\frac{\|\pA\|_\infty^2}{J}\right)\iniC^2\dist^2(Z,\ptrnZ)\Bigg\}.\label{eq:A}
\end{multline}
\end{proof}

\begin{proof}[of Lemma~\ref{lemma:B}]
We separately bound $B1$ and $B2$.
Since $\Sigma_j=V\diag(a_j)V^T$, recalling~\eqref{eq:gradientva}, we have
\begin{align}
    B1 &= \frac{1}{J} \left\| \frac{2}{J}\sum_{j=1}^J [\graell {V}\diag({a}_{j})]_{\suppU} \right\|_F^2\notag\\
    &= \frac{4}{J^3} \left\| \sum_{j=1}^J [\graellstar {V}\diag({a}_{j})]_{\suppU} + \left[ \left\{\graell - \graellstar \right\}{V}\diag({a}_{j}) \right]_{\suppU}\right\|_F^2.\notag \\
    \intertext{Using  the  Cauchy–Schwarz inequality and along with the fact
    that $(a+b)^2\leq2(a^2+b^2)$, the above display can be bounded as}
        & \leq \frac{8}{J^2}\sum_{j=1}^J \|[\{\graell - \graellstar \}{V}\diag({a}_j)]_{\suppU} \|_F^2\notag\\
     &\quad + \frac{8}{J^3} \left\| \sum_{j=1}^J[\graellstar {V}\diag({a}_{j})]_{\suppU} \right\|_F^2.\label{eq:b1_midstep}
\end{align}
For any $X_{\suppU}$ with $\|X_{\suppU}\|_F=1$,
we have $\{[V\diag(a_j)X^T]_{\suppU,\suppU}\}_{j\in[J]}\in\dset{2K}{2s+s^\star}{2\gamma}{\trnerror}$.
We can bound the second term in~\eqref{eq:b1_midstep} as
\begin{align}
    \left\| \sum_{j=1}^J [\graellstar V\diag(a_j)]_{\suppU} \right\|_F &= \sup_{\substack{\|X_{\suppU}\|_F=1}} \sum_{j=1}^J \tr([\graellstar  V\diag(a_j) ]_{\suppU}X^T_{\suppU} ) \notag\\
    &= \sup_{\|X_{\suppU}\|_F=1} \sum_{j=1}^J\langle\graellstar , [V\diag(a_j)]_{\suppU} X^T_{\suppU} \rangle\notag\\
    & \leq \staterr \cdot \sup_{\|X_{\suppU}\|_F=1}  \rbr{\sum_{j=1}^J\|[V\diag(a_j)X^T]_{\suppU, \suppU}\|_F^2}^{1/2} \notag\\
    &\leq \staterr J^{1/2}\|V\|^2_2\|A\|_\infty\label{eq:upper_stat}.
\end{align}
Plugging~\eqref{eq:upper_stat} back into~\eqref{eq:b1_midstep}, we arrive at
\begin{align*}
B1& \leq \frac{8}{J^2} \sum_{j=1}^J \|\graell -\graellstar \|_F^2 \|{V}\|_2^2 \|\diag({a}_j)\|_2^2 + \frac{8}{J^2}\staterr^2 \|V\|_2^4\|A\|_\infty^2\\
& \leq \frac{8}{J^2}\sum_{j=1}^J \|\graell -\graellstar \|_F^2 \|\diag({a}_j)\|_2^2 + \frac{8}{J^2}\staterr^2\|A\|_\infty^2,
\end{align*}
where the last inequality follows since $\|V\|_2=1$.

For $B2$, we have
\begin{align*}
    B2 &= \frac{1}{J^2} \|\gradAwrtcov{V}\trnQ\trnQ^T\|_F^2 = \frac{1}{J^2} \| \{\gradAwrtcov{V} - \gradAwrtpcov{V} + \gradAwrtpcov{V} \} \trnQ\trnQ^T \|_F^2 \\
    & \leq \frac{2}{J^2} \| \{\gradAwrtcov{V} - \gradAwrtpcov{V}\} \trnQ\trnQ^T\|_F^2 + \frac{2}{J^2}\|\gradAwrtpcov{V}\trnQ\trnQ^T\|_F^2.
\end{align*}
Since $\trnQ\trnQ^T$ is an orthogonal projection operator, 
we have $\|X\trnQ\trnQ^T\|_F\leq\|X\|_F$ for a matrix $X$. 
Then
\begin{align}
\|\{\gradAwrtcov{V} - \gradAwrtpcov{V}\} \trnQ\trnQ^T\|_F^2 
\leq \|\gradAwrtcov{V} - \gradAwrtpcov{V} \|_F^2 
& \leq
\sum_{j=1}^J\|V^T\cbr{\graell -\graellstar }V\|_F^2\notag\\
&\leq\left\{\sum_{j=1}^J \|\graell - \graellstar \|_F^2\right\} \label{eq:B2_1},
\end{align}
since $\|V\|_2=1$. Furthermore, we have
 \begin{align}
     \|\gradAwrtpcov{V}\trnQ\trnQ^T\|_F^2=\sum_{k=1}^K\|\trnQ\trnQ^TW^\star_{k\cdot}(v_k)\|_2^2&\leq\cbr{\sum_{k=1}^K\|\trnQ\trnQ^TW^\star_{k\cdot}(v_k)\|_2}^2\notag\\
         &=\cbr{\sum_{k=1}^K\sup_{\|x_k\|_2\leq 1}\left\langle x_k,\trnQ\trnQ^TW^\star_{k\cdot}(v_k)\right\rangle}^2\notag\\
         &\leq\cbr{\sum_{k=1}^K\sup_{\|\trnQ\trnQ^Tx_k\|_2\leq 1}\left\langle \trnQ\trnQ^Tx_k,W_{k\cdot}^\star(v_k)\right\rangle}^2\notag\\
         &\leq\cbr{\sum_{k=1}^K\rbr{\frac{1}{2\trnerror\gamma}}^{1/2}\sup_{\|\trnQ\trnQ^Tx_k'\|^2_2\leq 2\trnerror\gamma}\left\langle \trnQ\trnQ^Tx_k',W_{k\cdot}^\star(v_k)\right\rangle}^2\notag.\\
         \intertext{Let $y_k=\trnQ\trnQ^Tx_k'$ for $k\in[K]$. Since a ball 
         of radius $(2\trnerror\gamma)^{1/2}$ is contained in the truncated ellipsoid $\ellipsoid{\trnQ}{\trnLam}{2\gamma}$,
         $y_k$, for $k\in[K]$, lies in the ellipsoid. Therefore,
         $\{\sum_{k=1}^Ky_{kj}v_kv_k^T\}_{j\in[J]}$ 
         is in $\dset{2K}{2s+s^\star}{2\gamma}{\trnerror}$
         and we can bound the above display as}
         &\leq\staterr^2 \cdot \frac{1}{2\trnerror\gamma}\rbr{\sum_{k=1}^K\sup_{\|\trnQ\trnQ^Tx_k'\|_2^2\leq 2\trnerror\gamma}\|\trnQ\trnQ^Tx_{k}'\|_2^2}\notag\\
         &= \staterr^2K\label{eq:B2_2}.
\end{align}
Combining~\eqref{eq:B2_1}--\eqref{eq:B2_2} and noting that $\|V\|_2=1$, we have
\[
B2 \leq \frac{2}{J^2} \left\{\sum_{j=1}^J \|\graell -\graellstar \|_F^2\right\} \|{V}\|_2^2 + \frac{2}{J^2}K\staterr^2\|V\|_2^2.
\]
     
Finally, combining $B1$ and $B2$ we arrive at the following
\begin{align}
B 
&\leq \frac{8}{J^2}\sum_{j=1}^J \{\|\graell -\graellstar \|_F^2 \left(\|\diag({a}_j)\|_2^2 + \|{V}\|_2^2\right)\} + \frac{2(4\vee K)}{J^2}\staterr^2(\|V\|_2^2+\|A\|_\infty^2)\notag\\
&\leq\frac{ 16}{J^2}\sum_{j=1}^J\left\{\|\graell -\graellstar \|_F^2\right\}\|{Z}_j\|_2^2+\frac{4(4\vee K)}{J^2}\staterr^2\max_{j\in[J]}\|Z_j\|_2^2\label{eq:B},
\end{align}
where the second inequality comes from $(\|\diag(a_j)\|_2^2\vee\|V\|_2^2)\leq\|Z_j\|_2^2$.
\end{proof}

\begin{proof}[of Lemma~\ref{lemma:CSS}]
We first
upper bound $\|{Z}_j\|_2$ with $\|Z_j^\star\|_2$ for every $j\in[J]$.
We have
\begin{align}\label{eq:ztub}
    \|{Z}_j\|_2
    &=\left\|\begin{pmatrix}{V}\\\diag({a}_j)\end{pmatrix}-\begin{pmatrix}{V}^\star\rotmat\\\rotmat^T\diag(\pa_j)\rotmat\end{pmatrix}+\begin{pmatrix}{V}^\star\rotmat\\\rotmat^T\diag(\pa_j)\rotmat\end{pmatrix}\right\|_2\notag\\
    &\leq \left\|\begin{pmatrix}{V}\\\diag({a}_j)\end{pmatrix}-\begin{pmatrix}{V}^\star\rotmat\\\rotmat^T\diag(\pa_j)\rotmat\end{pmatrix}\right\|_2+\|Z_j^\star\|_2\notag\\
    &\leq \iniC+\|Z_j^\star\|_2\notag\\
    &\leq\frac{\sigma_K(\pCov_j)}{16}+\|Z_j^\star\|_2\notag\\
    &\leq\frac{17}{16}\|Z_j^\star\|_2,
\end{align}
where the third inequality follows from~\eqref{eq:I0xi2},
in particular that $\iniC\leq{\sigma_K(\pCov_j)}/{16}$,
and 
the last inequality follows
because $\|\pZ_j\|_2\geq \sigma_K(\pCov_j)$.

Next, we
lower bound $\|{Z}_j^0\|_2$ in terms of $\|\pZ_j\|_2$ for every $j\in[J]$.
Similar to \eqref{eq:ztub}, we have
\begin{align}\label{eq:z0lb}
    \|{Z}^0_j\|_2
    &=\left\|\begin{pmatrix}{V}^0\\\diag({a}^0_j)\end{pmatrix}-\begin{pmatrix}{V}^\star\rotmat\\\rotmat^T\diag(\pa_j)\rotmat\end{pmatrix}+\begin{pmatrix}{V}^\star\rotmat\\\rotmat^T\diag(\pa_j)\rotmat\end{pmatrix}\right\|_2\notag\\
    &\geq-\left\|\begin{pmatrix}{V}^0\\\diag({a}^0_j)\end{pmatrix}-\begin{pmatrix}{V}^\star\rotmat\\\rotmat^T\diag(\pa_j)\rotmat\end{pmatrix}\right\|_2+\|Z_j^\star\|_2\notag\\
    & \geq -\iniC+\|Z_j^\star\|_2\notag\\
        &\geq-\frac{\sigma_K(\pCov_j)}{16}+\|\pZ_j\|_2\notag\\
        &\geq\frac{15}{16}\|\pZ_j\|_2.
\end{align}  
Combining~\eqref{eq:ztub} and~\eqref{eq:z0lb}, we  obtain
$\|{Z}_j\|_2^2\leq 2 \|{Z}_j^0\|_2^2$.
\end{proof}

\subsection{Proofs of Auxiliary Lemmas}

\begin{lemma}\label{lemma:rotation_bound}
Suppose $V^\star$ has orthonormal columns.
Let 
\[
\defR, \quad
\defRplus. 
\]
Then
\[
\|\rotmat^+-\rotmat\|_F\leq 2\|V^+-V^\star \rotmat\|_F
\]
and
\[
\|\rotmat^T\diag(a^\star)\rotmat-\rotmat^{+T}\diag(a^\star)\rotmat^+\|_F
\leq4\|\diag(a^\star)\|_2\|V^+-V^\star\rotmat\|_F.
\]
\end{lemma}
\begin{proof}
Recall that $V^\star$ has orthonormal columns. This implies that for a matrix $X$, we have
\begin{align*}
    \|V^\star X\|_F^2=\tr(X^TV^{\star T}V^{\star}X)=\tr(X^TX)=\|X\|_F^2.
\end{align*}
Using the above property, we have
\begin{equation}
\begin{aligned}
    \|\rotmat^+-\rotmat\|_F
    =\|V^\star\rotmat^+-V^\star\rotmat\|_F
    &=\|(V^+-V^\star\rotmat)+(V^\star\rotmat^+-V^+)\|_F\\
    &\leq\|V^+-V^\star\rotmat\|_F+\|V^+-V^\star\rotmat^+\|_F\\
    &\leq 2\|V^+-V^\star\rotmat\|_F,
\end{aligned}
\label{eq:lemma:rotation_bound:proof:1}
\end{equation}
where the last inequality follows as $\rotmat^+$ minimizes the distance $\|V^+-V^\star\rotmat\|_F$.
This completes the proof for the first statement.

For the second statement, we have
\begin{align*}
    \|\rotmat^T\diag(a^\star)\rotmat-\rotmat^{+T}\diag(a^\star)\rotmat^+\|_F
    &=\|(\rotmat-\rotmat^+)^T\diag(a^\star)\rotmat+\rotmat^{+T}\diag(a^\star)(\rotmat-\rotmat^+)\|_F\\
    &\leq \|\diag(a^\star)\|_2(\|\rotmat\|_2+\|\rotmat^+\|_2)\|\rotmat^+-\rotmat\|_F\\
    &\leq 4\|\diag(a^\star)\|_2\|V^+-V^\star\rotmat\|_F,
\end{align*}
where the last inequality follows from~\eqref{eq:lemma:rotation_bound:proof:1}
and $\|\rotmat\|_2=\|\rotmat^+\|_2=1$.
\end{proof}
\begin{lemma}\label{lemma:gradbound}
Assume that $V$ has orthonormal columns and 
\[
\norm{V-V^\star \rotmat}_F\leq \iniC^2 / J,
\quad
\|\diag(a_j)-\rotmat^T\diag(\pa_j)\rotmat\|_F^2\leq (J-1)\iniC^2/J,
\]
and $Z=(V^T, A)^T$. Under Assumption~\ref{assumption_stepsize}--\ref{assumption_statcondition}, we have
\[
\frac{\eta}{J}\|[\nabla_Vf_N(Z)]_{\suppU}\|_F\leq\frac{\iniC}{6\surd{J}},
\]
where $\iniC$ is defined in~\eqref{eq:I0xi2}.
\end{lemma}
\begin{proof}[of Lemma~\ref{lemma:gradbound}]
By~\eqref{eq:gradientva}, we have
\begin{align*}
   \frac{\eta}{J}\norm{[\nabla_Vf_N]_{\suppU}}_F
    &=\frac{2\eta}{J^2}\left\|\sum_{j=1}^J[\graell V\diag(a_j)]_{\suppU}\right\|_F\\
    &\leq \frac{2\eta}{J^2}\bigg\{\sum_{j=1}^J\bignorm{[\cbr{\graell -\graellstar} V\diag(a_j)]_{\suppU}}_F\\
    &\quad+\bignorm{\sum_{j=1}^J[\graellstar V\diag(a_j)]_{\suppU}}_F\bigg\}.\\
    \intertext{We can bound the second term of the above display using~\eqref{eq:upper_stat} and obtain}
    &\leq\frac{2\eta}{J^2}\cbr{\|A\|_\infty\sum_{j=1}^J\|\graell -\graellstar \|_F+\staterr J^{1/2}\|A\|_\infty}\\
    &=\frac{2\eta}{J^2}\cbr{\|A\|_\infty\sum_{j=1}^J\|\Sigma_j-\pCov_j\|_F+\staterr J^{1/2}\|A\|_\infty}.
    \end{align*}
From~\eqref{eq:cov2dist}, we have 
\begin{align*}
\|\Sigma_j-\pCov_j\|_F^2&\leq3\cbr{(\|A\|_\infty^2+\|\pA\|_\infty^2)\|V-V^\star \rotmat\|_F^2+\|\diag(a_j)-\rotmat^T\diag(\pa_j)R\|_F^2}\\
&\leq 3\rbr{ \max_{j\in[J]}\|Z_j\|_\infty^2 + \max_{j'\in[J]}\|\pZ_{j'}\|_\infty^2 }\iniC^2,
\end{align*}
where the last inequality uses that $\|A\|_\infty\leq\max_{j\in[J]}\|Z_j\|_2$, $\|\pA\|_\infty\leq\max_{j\in[J]}\|\pZ_j\|_2$, and $1=\|V\|_2\leq\|Z_j\|_2$ for $j\in[J]$.
Then
\begin{align*}
\frac{\eta}{J}\|[\nabla_Vf_N]_{\suppU}\|_2&\leq \frac{2\eta}{J^2}\max_{j\in[J]}\cbr{2J\iniC\|Z_j\|_2(\|Z_j\|_2+\max_{j'\in[J]}\|\pZ_{j'}\|_2)+\staterr J^{1/2}\|Z_j\|_2}.
    \intertext{Applying~\eqref{eq:ztub}, the above display can be bounded by}
    &\leq\frac{2\eta}{J^2}\max_{j\in[J]}\cbr{\frac{11}{5}J\iniC\|\pZ_j\|_2^2+\frac{17}{16}\staterr J^{1/2}\|\pZ_j\|_2}\\
    &\leq\frac{9 \iniC\eta}{J}\max_{j\in[J]}\|\pZ_j\|_2^2\\
    &\leq\frac{32 \iniC\eta}{3J}\max_{j\in[J]}\|Z_j^0\|_2^2,
\end{align*}
where the second to last inequality 
follows by Assumption~\ref{assumption_statcondition}
and the last inequality follows by~\eqref{eq:z0lb}.
Then, by Assumption~\ref{assumption_stepsize}, we have
\begin{align}\label{eq:ubgradV}
\frac{\eta}{J}\|[\nabla_Vf_N]_{\suppU}\|_F\leq\frac{\iniC}{6\surd{J}},
\end{align}
which completes the proof.
\end{proof}
\begin{lemma}\label{lemma:distbarv}
Assume that $V$ has orthonormal columns and let 
\[
\bar{V}^+=\Pi_{\projV}\{V-\eta/J[\nabla_Vf_N(Z)]_{\suppU}\}, 
\quad 
Z=(V^T, A)^T, 
\quad 
\defR.
\]
Assume that
\[
\norm{V-V^\star \rotmat}_F\leq \iniC^2 / J,
\quad
\|\diag(a_j)-\rotmat^T\diag(\pa_j)\rotmat\|_F^2\leq (J-1)\iniC^2/J,
\]
then under Assumption~\ref{assumption_stepsize}, \ref{assumption_para}, and~\ref{assumption_statcondition}, we have
\[
\|\bar{V}^+-V^\star\rotmat\|_2\leq \frac{2\iniC}{J^{1/2}}<1.
\]
\end{lemma}
\begin{proof}[of Lemma~\ref{lemma:distbarv}]
By definition, we have
\[
\|\bar{V}^+-V^\star{\rotmat}\|_2\leq\|\bar{V}^+-V\|_2+\|V-V^\star\rotmat\|_2.
\]
By Assumption~\ref{assumption_inicondition}, we have
$
\|V-V^\star\rotmat\|_2\leq{\iniC}/{J^{1/2}}. 
$
Then, it remains to show that 
\[
\|\bar{V}^+-V\|_2\leq\|\bar{V}^+-V\|_F\leq \iniC/J^{1/2}.
\] 

The first step is to apply Lemma~\ref{lemma:rho_ub} on $\|\bar{V}^+-V\|_F$.
We first verify that the inner product of 
the $k$th column of $V-(\eta/J)[\nabla_Vf_N]_{\suppU}$ and $v_k$ are nonnegative for $k\in[K]$. 
By Lemma~\ref{lemma:gradbound}, we have 
\[
\left\langle v_k-\rbr{\frac{\eta}{J}[\nabla_Vf_N]_{\suppU}}_k,v_k\right\rangle\geq 1-\left\|\rbr{\frac{\eta}{J}[\nabla_Vf_N]_{\suppU}}_k\right\|_2\geq 1-\frac{\iniC}{6\surd{J}}>0,
\]
for every $k\in[K]$,
which allows us to apply Lemma~\ref{lemma:rho_ub}.
Note that the term $1+2\surd{s^\star}/\surd{(s-s^\star)}$ 
decreases as $s$ increases.
Under Assumption~\ref{assumption_para}, 
$s\geq2s^\star$ and 
we have $1+2\surd{s^\star}/\surd{(s-s^\star)}\leq 3$.
Therefore applying 
Lemma~\ref{lemma:rho_ub} with $s=2s^\star$:
\begin{align*}
    \|\bar{V}^+-V\|_2\leq \|\bar{V}^+-V\|_F&\leq  6\bignorm{V-\frac{\eta}{J}[\nabla_Vf_N]_{\suppU}-V}_F\leq\frac{\iniC}{\surd{J}},
\end{align*}
where the last inequality follows from Lemma~\ref{lemma:gradbound}.

\end{proof}
\section{Quantification of Statistical Error}
\subsection{Proof of Proposition~\ref{prop:SSE}}

Let $\Omega(s,P)$ be a collection of subsets of $[P]$,
each with cardinality $s$. 
Let 
\[\suppV=\defsuppV{K}{s},\] 
be the set of the supports,
where $\mathcal{S}_{V_k}$ denotes the support of component $k$ for $k\in[K]$.
We first establish a bound on the statistical error for
a  fixed support $\suppV$ and
then take the union bound to
establish a bound on the statistical error 
on the set $\dset{2K}{ms^\star}{2m'\gamma^\star}{\trnerror}$, for some constant $m,m'>0$. 
For some positive semi-definite matrix $G$, we define the sets 
\begin{align*}
    &\defspatialset{\suppV}{K};\\
    &\temporalset{G}{\gamma}=
    \ellipsoid{G}{\Lambda}{\gamma},
\end{align*}
where $G^\dagger=Q\Lambda Q^T$ is the eigendecomposition.
For a positive semi-definite kernel matrix $G$ and a positive scalar $\gamma$, 
we define the semi-norm $\Gnorm{\cdot}{G}{\gamma}$ as
\[
\Gnorm{x}{G}{\gamma}^2=\frac{1}{\gamma}(x^TG^\dagger x).
\]
Therefore, the set $\temporalset{G}{\gamma}$ is a unit 
ball in 
$\Gnorm{\cdot}{G}{\gamma}$.
We use
$\mathcal{N}_{\mathcal{V}}({\epsilon_v})$ 
to denote
the $\epsilon$-net for $\mathcal{N}(\spatialset{\suppV}{2K},\epsilon_v,\|\cdot\|_F)$
and
$\mathcal{N}_{\mathcal{T}}(\epsilon_a)$ 
to denote
the $\epsilon$-net for 
$\mathcal{N}(\temporalset{\trnG}{2m'\gamma^\star},\epsilon_a,\Gnorm{\cdot}{\trnG}{2m'\gamma^\star})$.
For a matrix $A\in\RR^{K\times J}$, we use $\rowkA$ to denote $k$th row of $A$ and $a_j$
to denote the $j$th column of $A$.
We define the following set
\begin{align*}
    \mathcal{U}(\suppV,2m'\gamma^\star)=\{\{U\diag(a_j)V^T\}_{j\in J}:U,V\in\spatialset{\suppV}{2K},\rowkA\in\temporalset{\trnG}{2m'\gamma^\star}, k\in[2K]\},
\end{align*}
and let $\{\Delta_j=U\diag(a_j)V^T\}_{j\in[J]}\in\mathcal{U}(\suppV,2m'\gamma^\star)$.
Recall that $\graellstar =\pCov_j-\sampleCov{j}$ for $j\in[J]$.
We have
\begin{multline}
    \sup_{\{\Delta_j\}\in\mathcal{U}(\suppV,2m'\gamma^\star)}
    \sum_{j=1}^J\langle\Sigma^\star_j-\sampleCov{j}, \Delta_j\rangle
    \\
    =
    \sup_{\{\Delta_j\}\in\mathcal{U}(\suppV,2m'\gamma^\star)}
    \sum_{j=1}^J\langle \sampleCov{j}-\pCov_j-E_j, \Delta_j\rangle
    +\sum_{j=1}^J\langle E_j,\Delta_j\rangle.\label{eq:stat_error1}
\end{multline}
For the second term in the above display, we have
\begin{multline}
\sum_{j=1}^J\langle E_j,\Delta_j\rangle=\frac{1}{2}\sum_{j=1}^J\langle E_j,\Delta_j+\Delta_j^T\rangle
\leq  \frac{1}{2}\sum_j \norm{E_j}_2 \norm{\Delta_j+\Delta_j^T}_F
\leq \rbr{\max_j \norm{E_j}_2} \cdot \sum_j \|\Delta_j\|_F
\label{eq:errorbound},
\end{multline}
where the first equality follows by the fact that $E_j$ is symmetric. 

Using Lemma~\ref{lemma:innerproduct_net},
we have
\begin{multline}
\sup_{\{\Delta_j\}\in\mathcal{U}(\suppV,2m'\gamma^\star)}\left|\sum_{j=1}^J\langle \sampleCov{j}-\pCov_j-E_j,\Delta_j\rangle\right|\\
\leq 
(1-2\epsilon_v-\epsilon_a)^{-1}
\max_{
\substack{U,V\in\mathcal{N}_{\mathcal{V}}(\epsilon_v)\\
\rowkA\in\mathcal{N}_{\mathcal{T}}(\epsilon_a),\;k\in[2K]}}
\left|\frac{1}{2}\sum_{j=1}^J\langle \sampleCov{j}-\pCov_j-E_j,\Delta_j+\Delta_j^T\rangle\right|.\label{eq:stat_norm}
\end{multline}
For a fixed set of $\{\Delta_j\}_{j=1}^J$ for $j\in[J]$, we let
$Y_{j}^{(n)}=(1/2)\tr\{(\Delta_j+\Delta_j^T)x_j^{(n)}x_j^{(n)T}\}$.
Note that $E\{Y_{j}^{(n)}\} = (1/2)\tr\cbr{(\Delta_j+\Delta_j^T)\rbr{\pCov_j+E_j}}$.
Consequently, we have
\begin{align}
\frac{1}{4}\sum_{j=1}^J\tr\cbr{\rbr{\pCov_j+E_j}(\Delta_j+\Delta_j^T)\rbr{\pCov_j+E_j}(\Delta_j+\Delta_j^T)}
&\leq\frac{1}{4}\sum_{j=1}^J\|\pCov_j+E_j\|_2^2\tr\cbr{(\Delta_j+\Delta_j^T)^2}\notag\\
&=\frac{1}{4}\sum_{j=1}^J\|\pCov_j+E_j\|_2^2\|\Delta_j+\Delta_j^T\|_F^2\notag\\
&\leq\max_{j\in[J]}\|\pCov_j+E_j\|_2^2\sum_{j=1}^J\|\Delta_{j}\|_F^2\notag\\
&\leq 4\|A^\star\|_\infty^2\sum_{j=1}^J\|\Delta_{j}\|_F^2,\label{eq:subexp_norm}
\end{align}
where the last step follows from $\|\pCov_j+E_j\|_2\leq 2\|\pCov_j\|_2\leq\|\pA\|_\infty$ for all $j\in[J]$.

Then, using Lemma~\ref{lemma:chisquare_bound},
\begin{align*}
    pr&\left[\frac{1}{N}\left|\sum_{n=1}^{N}\sum_{j=1}^JY_{j}^{(n)}-E\cbr{Y_{j}^{(n)}}\right|\geq{\frac{\varepsilon}{4}}\right]\\
    &\quad\quad\quad\quad\quad\quad\leq 2\exp\left\{-Ne_0\left(\frac{\varepsilon^2}{16\|\pA\|_\infty^2\sum_{j=1}^J\|\Delta_j\|_F^2}\wedge\frac{\varepsilon }{4\|\pA\|_\infty\max_{j\in[J]}\|\Delta_j\|_2}\right)\right\},
\end{align*}
where $e_0$ are some absolute constant.

We choose $\epsilon_v=\epsilon_a=1/4$.
Taking the union bound over
$\mathcal{N}_{\mathcal{V}}(\epsilon_v)$, $\mathcal{N}_{\mathcal{T}}(\epsilon_a)$
and the choice of $\Omega(ms^\star,P)$, we have
\begin{multline}\label{eq:1}
    pr\left[
    \max_{\suppV\in{\{\Omega(ms^\star,P)\}^{2K}}}
    \max_{\substack{U,V\in\mathcal{N}_{\mathcal{V}}(\epsilon_v)\\\rowkA\in\mathcal{N}_{\mathcal{T}}(\epsilon_a),\;k\in[2K]}}\frac{1}{N}\sum_{n=1}^N\left|\sum_{j=1}^JY_j^{(n)}-E\cbr{Y_j^{(n)}} \right|\geq\frac{\varepsilon}{4}\right] \\
    \leq 2\binom{P}{ms^\star}^{4K} 
    9^{4Kms^\star+2K\rkG{\trnG}} 
    \exp \left\{ -Ne_0 \left(\frac{\varepsilon^2}{16\|\pA\|_\infty^2\sum_{j=1}^J\|\Delta_j\|_F^2} \wedge \frac{\varepsilon }{4\|\pA\|_\infty\max_{j\in[J]}\|\Delta_j\|_2}\right)\right\},
\end{multline}
where we applied the metric entropy in Lemma~\ref{lemma:covering_upsilon}.

Given $0<\delta<1$, let
\begin{align*}
    \nu\leq \frac{1}{e_0'}\left[\frac{1}{N}\left\{\log\frac{1}{\delta}+K\rkG{\trnG}+Ks^\star+Ks^\star\log\frac{Pe}{s^\star}\right\}\right]^{{1}/{2}}
\end{align*}
for some constant $e_0'$.
Combining~\eqref{eq:errorbound} and~\eqref{eq:1} with~\eqref{eq:stat_error1}, 
we have
\begin{equation}
\staterr\rbr{\sum_{j=1}^J\|\Delta_j\|_F^2}^{1/2}
\leq
\|\pA\|_{\infty}\rbr{\sum_{j=1}^J\|\Delta_j\|_F^2}^{1/2}(\nu\vee\nu^2)
+\rbr{\sum_{j=1}^J\|\Delta_j\|_F}\max_{j\in[J]}\|E_j\|_2,\label{eq:ddd}
\end{equation}
with probability at least $1-\delta$. 
Using the Cauchy-Schwarz inequality, 
\[
\sum_{j=1}^J\norm{\Delta_j}_F\leq J^{1/2}\rbr{\sum_{j=1}^J\|\Delta_j\|_F^2}^{1/2}
\]
and, therefore,
\[
\staterr\leq \|\pA\|_\infty(\nu\vee \nu^2)+J^{1/2}\max_{j\in[J]}\|E_j\|_2
\]
with probability at least $1-\delta$.

\subsection{Metric Entropy of the Structured Set}

We find the metric entropy of 
$\mathcal{N}(\spatialset{\suppV}{K},\epsilon_v,\|\cdot\|_F)$ and $\mathcal{N}(\temporalset{G}{\gamma},\epsilon_a,\|\cdot\|_{G,\gamma})$.

\begin{lemma}\label{lemma:covering_upsilon} 
Given a support set $\defsuppV{K}{s}$, let 
\[
\defspatialset{\suppV}{K}.
\]
The metric entropy of $\mathcal{N}(\spatialset{\suppV}{K},\epsilon_v,\|\cdot\|_F)$ is
\begin{align*}
    \log|\mathcal{N}(\spatialset{\suppV}{K},\epsilon_v,\|\cdot\|_F)|\leq{Ks}\log\left(1+\frac{2}{\epsilon_v}\right).
\end{align*}
Given $G$ and $\gamma$, the metric entropy of $\mathcal{N}(\temporalset{G}{\gamma},\epsilon_a,\|\cdot\|_{G,\gamma})$ is
\begin{align*}
    \log|\mathcal{N}(\temporalset{G}{\gamma},\epsilon_a,\|\cdot\|_{G,\gamma})|\leq \rkG{G}\log\rbr{1+\frac{2}{\epsilon_a}},
\end{align*}
where $\rkG{G}$ is the rank of $G$.
\end{lemma}
\begin{proof}[of Lemma~\ref{lemma:covering_upsilon}]
The first result directly follows from Lemma 5.2 in \cite{vershynin2010introduction}.
For the second reults, we note that 
the set $\temporalset{G}{\gamma}$ is a $\rkG{G}$-dimensional unit ball
in the semi-norm $\|\cdot\|_{G,\gamma}$.
Therefore, we can again apply Lemma~5.2 in~\citet{vershynin2010introduction}.
\end{proof}
\subsection{Inner Product on a $\epsilon$-net}

We define the following operator similar to the definition of $W^\star(V)$:
\[
\hat{W}^\star(U,V)=[\hat{w}_{kj}^\star(u_k,v_k)]\in\mathbb{R}^{K\times J},\quad \hat{w}_{kj}^\star(u_k,v_k)=u_k^T\nabla_{N,j}(\pCov_j)v_k,\quad\graellstar =\pCov_j-\sampleCov{j}.
\]
Recall that $A=[a_1,\ldots,a_j]\in\mathbb{R}^{K\times J}$.
Then
\begin{align*}
    \sum_{j=1}^J\langle\graellstar ,U\diag(a_j)V^T\rangle=\sum_{k=1}^K\rowkA^T\hat{W}^\star_{k\cdot}(u_k,v_k),
\end{align*}
where  $\rowkA\in\mathbb{R}^J$ denote the $k$th row of $A$ and $\hat{W}^\star_{k\cdot}(u_k,v_k)$ denotes the $k$th row of $\hat{W}^\star(U,V)\in\mathbb{R}^J$.
Consequently, if every row $\rowkA\in\mathbb{R}^J$ lies in $\temporalset{G}{\gamma}$, we have
\[
\sum_{j=1}^J\langle\graellstar ,U\diag(a_j)V^T\rangle=\sum_{k=1}^K\rowkA^TQQ^T\hat{W}^\star_{k\cdot}(u_k,v_k),
\]
where $Q$ is the matrix whose columns are eigenvectors of $G$.
This is another representation of the statistical error and will
help us to  simplify the proof steps of the following lemma.

\begin{lemma}[Inner product on a net]\label{lemma:innerproduct_net}
Given a support $\suppV$, 
a matrix $G$, 
and a positive scalar $\gamma$, we have
\begin{align*}
    \max_{\substack{U,V\in\spatialset{\suppV}{K}\\\rowkA\in\temporalset{G}{\gamma},\;k\in[K]}}\sum_{k=1}^K\rowkA^TQQ^T\hat{W}_{k\cdot}^\star(u_k,v_k)
    \leq 
(1-2\epsilon_v-\epsilon_a)^{-1}\max_{\substack{U,V\in\mathcal{N}_\mathcal{V}(\epsilon_v)\\\rowkA\in\mathcal{N}_\mathcal{T}(\epsilon_a)\;k\in[K]}}\sum_{k=1}^K\rowkA^TQQ^T\hat{W}_{k\cdot}^\star(u_k,v_k),
\end{align*} 
where $Q$ denotes the matrix whose columns are eigenvectors of $G$.

\end{lemma}

\begin{proof}
Let $\hat U$, $\hat V$ and $\hat A$ be the quantities that maximize 
\[
   \tilde{\varepsilon}  = 
\max_{\substack{U,V\in\spatialset{\suppV}{K}\\\rowkA\in\temporalset{G}{\gamma},\;k\in[K]}}\sum_{k=1}^K\rowkA^TQQ^T\hat{W}_{k\cdot}^\star(u_k,v_k).
\]
Then
\begin{multline*}
\left|\sum_{k=1}^K\rowkA^TQQ^T\hat{W}_{k\cdot}^\star(u_k, v_k)\right|
=\bigg|\sum_{k=1}^K\hat{A}_{k\cdot}^TQQ^T\hat{W}_{k\cdot}^\star(\hat{u}_k,\hat{v}_k)+\sum_{k=1}^K(\rowkA-\hat{A}_{k\cdot})^TQQ^T\hat{W}_{k\cdot}^\star(u_k,v_k)\\
+\sum_{k=1}^K\hat{A}_{k\cdot}^TQQ^T\{\hat{W}_{k\cdot}^\star(u_k,v_k)-\hat{W}_{k\cdot}^\star(\hat{u}_k,\hat{v}_k)\}\bigg|.
\end{multline*}
Using the triangle inequality, we have
\[
\left|\sum_{k=1}^K\rowkA^TQQ^T\hat{W}_{k\cdot}^\star(u_k, v_k)\right|
\geq \tilde{\varepsilon}  - T_1 - T_2,
\]
where 
\[
T_1 = \left|\sum_{k=1}^K(\rowkA-\hat{A}_{k\cdot})^TQQ^T\hat{W}_{k\cdot}^\star(u_k,v_k)\right|
\]
and
\[
T_2 =  \left|\sum_{k=1}^K\hat{A}_{k\cdot}^TQQ^T\left\{\hat{W}_{k\cdot}^\star(u_k,v_k)-\hat{W}_{k\cdot}^\star(\hat{u}_k,\hat{v}_k)\right\}\right|.
\]
Let $\theta_k=(\rowkA-\hat{A}_{k\cdot})/\|\rowkA-\hat{A}_{k\cdot}\|_{G,\gamma}\in\mathbb{R}^J$.
Then $\|\theta_k\|_{G,\gamma}=1$ for every $k\in[K]$ 
and $\theta_k \in \temporalset{G}{\gamma}$. Therefore, we have
\begin{align*}
    T1&\leq\max_{k\in[K]}\|\rowkA-\hat{A}_{k\cdot}\|_{G,\gamma}
    \left|\sum_{k=1}^K\theta_k^TQQ^T\hat{W}_{k\cdot}^\star(u_k,v_k)\right| \\
    &=\max_{k\in[K]}\|\rowkA-\hat{A}_{k\cdot}\|_{G,\gamma}
    \left|\sum_{k=1}^K(QQ^T\theta_k)^TQQ^T\hat{W}_{k\cdot}^\star(u_k,v_k)\right|
    \leq\epsilon_a\tilde\varepsilon,
\end{align*}
where the inequality holds because $QQ^T\theta_k \in \temporalset{G}{\gamma}$
for $k\in[K]$.
For $T2$, we have
\begin{align*}
    T2=\left|\sum_{k=1}^K\hat{A}_{k\cdot}^{T}QQ^T\left\{\hat{W}_{k\cdot}^\star(u_k,v_k)-\hat{W}_{k\cdot}^\star(\hat{u}_k,\hat{v}_k)\right\}\right|&=\left|\sum_{j=1}^J\left\langle\graellstar , \sum_{k=1}^K\hat{a}_{kj}(u_kv_k^T-\hat{u}_k\hat{v}_k^{T})\right\rangle\right|\\
    &\leq\tilde\varepsilon\max_{k\in[K]}\left\|u_kv_k^T-\hat{u}_k\hat{v}_k^{T}\right\|_F\\
    &\leq\tilde\varepsilon\max_{k\in[K]}(\left\|v_k-\hat{v}_k\right\|_2+\left\|u_k-\hat{u}_k\right\|_2)\leq2\tilde\varepsilon\epsilon_v,
\end{align*}
where the second equality follows from $\hat{A}_{k\cdot}^{T}QQ^T=\rowkA$ for $k\in[K]$.
\end{proof}

\section{Sample Complexity of Spectral Initialization}
\subsection{Proof of Theorem~\ref{Theorem:SBSO}}\label{sssec:DKtheorem}

Let 
$\defsumsampleCov$,
$\defsumpCov$, and
$\defRo$. 
The proof proceeds in two steps. In the first step,
we establish that 
\begin{equation}
\dist^2({Z^0},{Z}^\star)
\leq e_1\|\sumsampleCov-\sumpCov\|_2^2+e_2\sum_{j=1}^J\|\sampleCov{j}-{\Sigma}_j^\star\|_2^2,\label{eq:sample_bound}
\end{equation}
where 
$e_1={5KJ}g^{-2}(1+{16\varphi^2}\|A^\star\|_\infty^2)$,
$e_2=8K\varphi^2$, and
$\defvarphitwo$
with $\defeigengap$.
In the second step, we bound
$\|\sumsampleCov-\sumpCov\|_2$
and
$\|\sampleCov{j}-{\Sigma}_j^\star\|_2$ for $j\in[J]$
using Lemma~\ref{lemma:bernsteinbound}.

\emph{Step 1.}
We write
\begin{align*}
    \dist^2({Z^0},{Z}^\star) = T1 + T2,
\end{align*}
where $T1 = \sum_{j=1}^J \|{V}^0-{V}^\star\rotmat^0\|_F^2$ and 
$T2 = \sum_{j=1}^J\|\diag({a}_j^0)-\rotmat^{0T}\diag({a}_j^\star)\rotmat^0\|_F^2$.
First, we find a bound on $\min_{Y\in\Ocal(K)}\|V^0-V^\star Y\|_F^2$
that does not depend on $R^0$. 
By Lemma~\ref{lemma:distance_bound}, we have
\begin{equation}\label{eq:vvvv}
    \min_{Y\in\mathcal{O}(K)}\|V^0-V^\star Y\|_F^2
    \leq
    \frac{1}{2(\surd{2}-1)}\|V^0V^{0T}-V^\star V^{\star T}\|_F^2.
\end{equation}
The following Lemma gives us a bound on $T2$ that does not depend on $\rotmat^0$.
\begin{lemma}\label{lemma:inibound_a}
Let $\defvarphitwo$,
we have
\begin{align}\label{eq:BPEM}
    T2 \leq 
    4K\varphi^2\sum_{j=1}^J\left(5\|\pCov_j\|_2^2\|V^0V^{0T}-V^\star V^{\star T}\|_2^2+2\|\sampleCov{j}-\pCov_j\|_2^2\right).
\end{align}
\end{lemma}

Putting~\eqref{eq:vvvv} and Lemma~\ref{lemma:inibound_a} together, we have
\begin{multline}\label{eq:dist_ub}
\dist^2({Z^0},{Z}^\star) \leq 
\frac{J}{2(\surd{2}-1)}\|V^0V^{0T}-V^\star V^{\star T}\|_F^2 \\
+ 
4K\varphi^2 \sum_{j=1}^J\left(2\|\sampleCov{j}-\Sigma_{j}^\star\|_2^2+5\|\pCov_j\|_2^2\|V^0V^{0T}-V^\star V^{\star T}\|_2^2\right).
\end{multline}
Using Lemma~\ref{lemma:DKSTT} 
to bound 
$\|V^0V^{0T}-V^\star V^{\star T}\|_F$
and
$\|V^0V^{0T}-V^\star V^{\star T}\|_2$,
and noting that
$\|\pCov_j\|_2\leq\|A^\star\|_\infty$,
we obtain \eqref{eq:sample_bound}.

\emph{Step 2.}
 We show that
 $\|\sumsampleCov-\sumpCov\|_2^2$ and $\|\sampleCov{j}-\pCov_j\|_2^2$
 are bounded with high probability
 when the eigengap $g$ is bounded away from zero. 
 We apply Lemma~\ref{lemma:bernsteinbound} with $\|\sumpCov\|_2\leq\|A^\star\|_\infty$ and obtain
 \begin{align*}
     \pr\cbr{\|\sumsampleCov-\sumpCov\|_2\geq h_M(\delta)}\leq\frac{\delta}{J},
 \end{align*}
 where 
 \[
h_M(\delta)=
    \frac{2P\|A^\star\|_{\infty}}{NJ}\log\frac{2PJ}{\delta}+
    \rbr{\frac{2P\|A^\star\|_\infty^2}{NJ}\log\frac{2PJ}{\delta}}^{\frac{1}{2}}.
 \]
 Similarly, for $J\geq4$ and for every $j\in[J]$, we have
 \[
 \pr\cbr{\|\sampleCov{j}-\pCov_j\|_2\geq h_S(\delta)}\leq\frac{J-1}{J}\frac{\delta}{J},
 \]
 where
 \[
 h_S(\delta)=\frac{2P\|A^\star\|_{\infty}}{N}\log\frac{4PJ}{\delta}+\rbr{\frac{2P\|A^\star\|_\infty^2}{N}\log\frac{4PJ}{\delta}}^{\frac{1}{2}}.
 \]
 
 Then, collecting results and applying union bound, 
 we have
 \[
 \dist^2(Z^0,Z^\star)\leq\frac{5KJ}{g^2}\rbr{1+16\varphi^2\|A^\star\|^2_\infty}h_M^2(\delta)+8KJ\varphi^2h_S^2(\delta),
 \]
 with probability at least $1-\delta$. This implies that
 \begin{align}
      &\dist^2(Z^0,Z^\star)\leq \phi(g,A^\star)\cbr{ \frac{KJP^2}{N^2}\rbr{\log\frac{4PJ}{\delta}}^2+\frac{KJP}{N} \log\frac{4PJ}{\delta}};\notag\\
      &\phi(g,A^\star)=4\|A^\star\|_{\infty}^2\cbr{\frac{5(1+16\varphi^2\|A^\star\|_\infty^2)}{g^2J}\vee 8\varphi^2}\label{eq:constantphi},
 \end{align}
with probability at least $1-\delta$ and 
$\phi(g,A^\star)$ is a constant 
that depends on $g$ and $A^\star$
. In particular, the bond holds only when the eigengap $g$ is bounded away from zero.

\subsection{Proof of  Lemma~\ref{lemma:inibound_a}}

We prove the result for $J=1$ and drop the subscript $j$ throughout the proof. 
The proof can be easily extended to the case where $J>1$. 
Recall that $\defRo$
and $\defvarphitwo$.
Similar to \eqref{eq:lemmab3_2} in the proof of Lemma~\ref{lemma:UBD_ortho},
we have
\begin{align}
    \left\|\diag(a^0)-\rotmat^{0T}\diag(a^\star)\rotmat^0\right\|_F^2&\leq\varphi^2\left\|V^0\diag(a^0)V^{0T}-V^\star\diag(a^\star)V^{\star T}\right\|_F^2\notag\\
    &\leq 2K\varphi^2\left\|\sum_{k=1}^K(v_k^{0T}S_Nv_k^{0})v_k^0v_k^{0T}-\Sigma^\star\right\|_2^2,\label{eq:d1_firststep}
\end{align}
where we can write $V^0\diag(a^0)V^T=\sum_{k=1}^K(v_k^{0T}S_Nv_k^{0})v_k^0v_k^{0T}$.
Note that we can write 
$\Sigma^\star=\sum_{k=1}^K({v}_k^{\star T}\Sigma^\star{v}_k^\star){v}_k^\star{v}_k^{\star T}$. 
Applying the triangle inequality to the right hand side of~\eqref{eq:d1_firststep}, we have
    \begin{multline}
    \left\|\diag(a^0)-\rotmat^{0T}\diag(a^\star)\rotmat^{0}\right\|_F^2\leq 4K\varphi^2\Bigg\{\left\|\sum_{k=1}^K(v_k^{\star T}\Sigma^\star v_k^{\star})(v_k^0v_k^{0T}-{v}_k^{\star}{v}_k^{\star T})\right\|_2^2\\
    + \left\|\sum_{k=1}^K(v_k^{0T}S_Nv_k^0-{v}_k^{\star T}\Sigma^\star {v}_{k}^\star ){v}_k^0 {v}_k^{0 T}\right\|_2^2\Bigg\}\label{eq:d1_midstep}.
\end{multline}
Next, we bound the two terms on the right hand side of~\eqref{eq:d1_midstep} separately. 
We have
\begin{align}
        \left\|\sum_{k=1}^K(v_k^{\star T}\Sigma^\star v_k^\star)(v_k^0v_k^{0T}-{v}_k^{\star}{v}_k^{\star T})\right\|_2^2&\leq\|\Sigma^\star\|_2^2\left\|\sum_{k=1}^Kv_k^{0}v_k^{0T}-\sum_{k=1}^K{v}_k^\star{v}_k^{\star T}\right\|_2^2.\notag\\
        \intertext{By the uniqueness of projection operators, the above display can be written as}
        &=\|\Sigma^\star\|_2^2\|V^0V^{0T}-V^\star V^{\star T}\|_2^2.\label{eq:d1_aux1}
\end{align}
Since ${v}_1^0,\ldots,{v}_K^0$ are orthonormal to each other, we have
\begin{align}
        \left\|\sum_{k=1}^K\{v_k^{0T}S_Nv_k^0-{v}_k^{\star T}\Sigma^\star {v}_{k}^\star \}{v}_k^0{v}_k^{0 T}\right\|_2^2&\leq
        \max_{k\in[K]}|v_k^{0T}S_Nv_k^0-{v}_k^{\star T}\Sigma^\star{{v}_k^\star}|^2\notag\\
        &=\max_{k\in[K]}|v_k^{0T}S_Nv_k^0-v_k^{0T}\Sigma^\star v_k^0+v_k^{0T}\Sigma^\star v_k^0-{v}_k^{\star T}\Sigma^\star{{v}_k^\star}|^2\notag\\
        &\leq2\|S_N-\Sigma^\star\|_2^2+2\max_{k\in[K]}|v_k^{0T}\Sigma^\star v_k^{0}-{v}_k^{\star T}\Sigma^\star{v}_k^\star|^2\notag\\
        &=2\|S_N-\Sigma^\star\|_2^2+2\max_{k\in[K]}|\langle \Sigma^\star, v_k^0v_k^{0T}-v_k^\star v_k^{\star T}\rangle|^2.\label{eq:d1_aux2_1}
    \end{align}
Next, we upper bound  
the second term on the right hand side of~\eqref{eq:d1_aux2_1}. We have
    \begin{align}
        \max_{k\in[K]}|\langle \Sigma^\star, v_k^0v_k^{0T}-v_k^\star v_k^{\star T}\rangle|^2&\leq\|\Sigma^\star\|_2^2\max_{k\in[K]}\|v_k^0v_k^{0T}-v_k^{\star}v_k^{\star T}\|_F^2\notag\\
        &\leq2\|\Sigma^\star\|_2^2\max_{k\in[K]}\|v_k^0v_k^{0T}-v_k^{\star}v_k^{\star T}\|_2^2\notag\\
        &\leq2\|\Sigma^\star\|_2^2\|V^0V^{0T}-V^{\star}V^{\star T}\|_2^2. \label{eq:d1_aux3}
    \end{align}
Plugging the result of~\eqref{eq:d1_aux3} into~\eqref{eq:d1_aux2_1}, we have
    \begin{align}
        \left\|\sum_{k=1}^K(v_k^{0T}S_Nv_k^0-{v}_k^{\star T}\Sigma^\star {v}_{k}^\star ){v}_k^\star {v}_k^{\star T}\right\|_2^2&\leq2\|S_N-\Sigma^\star\|_2^2+4\|\Sigma^\star\|_2^2\|V^0V^{0T}-V^\star V^{\star T}\|_2^2\label{eq:d1_aux2_2},
    \end{align}
where the last inequality follows by triangle inequality.
Combining results from~\eqref{eq:d1_aux1} and~\eqref{eq:d1_aux2_2}, we have
\begin{align*}
    \|\diag(a^0_j)-\rotmat^{0T}\diag(a_j^\star)\rotmat^{0}\|_F^2\leq 4K\varphi^2\left(5\|\pCov_j\|_2^2\|V^0V^{0T}-V^\star V^{\star T}\|_2^2+2\|\sampleCov{j}-\pCov_j\|_2^2\right),
\end{align*}
for every $j\in[J]$. Hence, we complete the proof.

\section{Proof of Theorem~\ref{theorem:combine12}}
To proof Theorem~\ref{theorem:combine12}, we need following lemma.

\begin{lemma}[Linear Convergence Rate]\label{lemma:linearcov}
 Suppose that 
 Assumptions~\ref{assumption_stepsize}--\ref{assumption_inicondition}
 are satisfied. After $I$ iterations of Algorithm~\ref{alg:main}, 
  we have
 \begin{align}\label{eq:maintheorem}
        &\sum_{j=1}^J\|\Sigma_j^I-\ptrnCov_j\|_F^2\leq\beta^{I/2}({2\mu^2\xi^2})\sum_{j=1}^J\|\Sigma_j^0-\ptrnCov_j\|_F^2+C_1\staterr^2, &C_1=\frac{2\tau\mu^2\eta}{\beta^{1/2}-\beta},
 \end{align}
 where $\defmu$ and $\deftau$.
\end{lemma}
Since $\|\pZ_j\|_2$ and 
$\mu$ are bounded, 
Lemma~\ref{lemma:linearcov} shows that 
Algorithm~\ref{alg:main} achieves error
smaller than $\delta+C_1\varepsilon_{{stat}}^2$ 
after $I\gtrsim\log(1/\delta)$ 
iterations. The second term on the left hand side denotes the
constant multiple of the statistical error,
which depends on the distribution of the data and the sample size. 
\begin{proof}[of Lemma~\ref{lemma:linearcov}]
Let $Z = (V^T, A)^T$ be an iterate
obtained by Algorithm~\ref{alg:main},
$\Sigma_j=V\diag(a_j)V^{T}$ for $j\in[J]$, and
\[
\defR.
\]
We have the following decomposition
\begin{multline*}
    {\Sigma}_j-\ptrnCov_j
    = 
    (V-V^\star\rotmat)\rotmat^T\diag(\ptrna_j)\rotmat\rotmat^TV^{\star T}
    \\
    +
    V\{\diag(a_j)-\rotmat^T\diag(\ptrna_j)\rotmat\}\rotmat^TV^{\star T}
    + 
    {V}\diag({a}_j)({V}-{V}^\star\rotmat)^T.
\end{multline*}
Then 
\[
\|{\Sigma}_j-\ptrnCov_j\|_F
\leq
\left\{\|\diag({a}_j)\|_2+\|\diag(\ptrna_j)\|_2\right\}\|{V}-{V}^\star{R}\|_F+\|\diag({a}_j)-\rotmat^T\diag(\ptrna_j)\rotmat\|_F.
\]
Since $\|\diag({a}_j)\|_2 \leq \|{Z}_j\|_2$
and 
$\|Z_j\|_2+\|\ptrnZ_j\|_2 \leq\|Z_j\|_2+\|\pZ_j\|_2 \leq  (17/8)\|\pZ_j\|_2$
from \eqref{eq:ztub}, we have
\begin{align*}
    \|{\Sigma}_j-\ptrnCov_j\|_F
    \leq
    \mu\left\{\|{V}-{V}^\star\rotmat\|_F+\|\diag({a}_j)-\rotmat^T\diag(\ptrna_j)\rotmat\|_F\right\}.
\end{align*}
Combining
Theorem~\ref{theorem:lineardist} with Lemma~\ref{lemma:UBD_ortho},
in the $I$th iteration,
we have
\begin{align*}
    \sum_{j=1}^J\|{\Sigma}_j^I-\ptrnCov_j\|_F^2&\leq2\mu^2\dist^2({Z}^I,\ptrnZ)\notag\\
    &\leq 2\mu^2\left\{\beta^{I/2}\dist^2({Z}^0,\ptrnZ)+\frac{\tau\eta\staterr^2}{\beta^{1/2}(1-\beta)^{1/2}}\right\}\\
    &\leq 2\mu^2\left\{\beta^{I/2}\xi^2\sum_{j=1}^J\|\Sigma_j^0-\ptrnCov_j\|_F^2+\frac{\tau\eta\varepsilon_{\text{stat}}^2}{\beta^{1/2}(1-\beta)^{1/2}}\right\},
\end{align*}
which completes the proof.

\end{proof}

\begin{proof}[of Theorem~\ref{theorem:combine12}]

We first note that under the assumptions, 
using
Theorem~\ref{Theorem:SBSO},
Assumption~\ref{assumption_inicondition} is satisfied 
with probability at least $1-\delta_0$.
This allows us to use 
Lemma~\ref{lemma:linearcov} 
to bound $\sum_{j=1}^J\norm{\Sigma_j^I-\ptrnCov_j}_F^2$. 
By triangle inequality, we have
\[
\sum_{j=1}^J\norm{\Sigma_j^I-\pCov_j}_F^2\leq2\underbrace{\sum_{j=1}^J\norm{\Sigma_j^I-\ptrnCov_j}_F^2}_{\text{Estimation error}}+2\underbrace{\sum_{j=1}^J\norm{\ptrnCov_j-\pCov_j}_F^2}_{\text{Approximation error}}
\]
and, therefore,
it remains only 
to bound the
approximation
error.
We have
\begin{align*}
   \sum_{j=1}^J\norm{\ptrnCov_j-\pCov_j}_F^2&=\sum_{j=1}^J\norm{V^\star\diag(\ptrna_j-a_j^\star)V^{\star T}}_F^2\\
   &=\sum_{j=1}^J\norm{\diag(\ptrna_j-a_j^\star)}_F^2\\
   &=\norm{\ptrnA-A^\star}_F^2.
\end{align*}
Recall that each row of $\ptrnA$ is the projection of 
the corresponding row of $A^\star$ to the set $\projtrnA(c,\gamma)$,
where $c\geq c^\star$ and $\gamma\geq\gamma^\star$.
Recall
that 
columns of $Q=(\trnQ,\;Q_1)$ denote
eigenvectors of $G$
For any $k\in[K]$, 
if $Q_1^T\rowkA^\star=0$, then we have no loss in
projecting $\rowkA^\star$ to $\projtrnA(c,\gamma)$.
When $Q_1^T\rowkA^\star\neq0$, 
we want to quantify the loss of using truncated ellipsoid.
Let $\rowkA^\star=Qh_k^\star$,
$h_k^\star\in\mathbb{R}^J$ for $k\in[K]$. 
Then $\ptrnA_{k\cdot}=\trnQ \tilde{h}_k^\star$, 
where $\tilde{h}_{kj}^\star=h_{kj}^\star$ for $j\leq \rkG{\trnG}$,
$k\in[K]$. Therefore, for each row $k\in[K]$, we have
\begin{equation}\label{eq:errortrnA}
\norm{\ptrnA_{k\cdot}-\rowkA^*}_2^2=\sum_{j=\rkG{\trnG}+1}^J{h_{kj}^{\star2}}\leq \lambda_{\rkG{\trnG}}\sum_{j=\rkG{\trnG}+1}^J\frac{h_{kj}^{\star2}}{\lambda_j}\leq \lambda_{\rkG{\trnG}}\gamma^\star\leq \trnerror\gamma^\star.
\end{equation}
Then $\norm{\ptrnA-A^\star}_F^2\leq K\trnerror\gamma^\star$,
which completes the proof.

\end{proof}

\section{Proof of Proposition~\ref{prop:gaussian_example}}

The proof proceeds in three steps.
First, we verify that the iterate $Z^0$
obtained by Algorithm~\ref{alg:spectral_ini}
satisfies Assumption~\ref{assumption_inicondition}.
In the second step, we bound the statistical and
approximation errors. In the final step, 
we establish a bound
on $\sum_{j=1}^J\norm{\Sigma_j^I-\pCov_j}_F^2$.

\emph{Step 1.}
The proof is similar to that of Theorem~\ref{Theorem:SBSO}.
Let $\defsumsampleCov$ and $\defsumpCov$. 
Since $\sigma_K(\pCov_j)>E_j$ for every $j\in[J]$, 
the eigengap $\defeigengap$ is nonzero
and we can apply Davis-Kahan $``\sin\theta"$ theorem,
stated in Lemma~\ref{lemma:DKSTT}. From 
\eqref{eq:sample_bound},
\begin{align*}
    \dist^2(Z^0,Z^\star)\leq e_1\|\sumsampleCov-\sumpCov\|_2^2+e_2\sum_{j=1}^J\|\sampleCov{j}-\pCov_j\|_2^2.
\end{align*}
From the definition of $\|\cdot\|_{\psi_2}$ in~\eqref{eq:orlicznorm}, 
for every $n\in[N]$ and $j\in[J]$, we have
\begin{align*}
     \|x_j^{(n)}\|_{\psi_2}^2\leq \max_{j\in[J]}\|\pCov_j+E_j\|_2
    \leq 2\max_{j\in[J]}\|\pCov_j\|_2\leq2 \|A^\star\|_\infty.
\end{align*}
Lemma~\ref{lemma:subgaussianbound} then gives us
\begin{align}\label{eq:mnrate}
\|\sumsampleCov-\sumpCov\|_2 \lesssim \|A^\star\|_\infty\sbr{\frac{1}{NJ}\rbr{2P+\log\frac{2J}{\delta_0}}+\cbr{\frac{1}{NJ}\rbr{2P+\log\frac{2 J}{\delta_0}}}^{\frac{1}{2}}},
\end{align}
with probability at least $1 - \delta_0/J$.
Similarly, for $J\geq4$, we have for every $j\in[J]$,
\begin{align}\label{eq:snrate}
\|\sampleCov{j}-\pCov_j\|_2 \lesssim \|A^\star\|_\infty\sbr{\frac{1}{N}\rbr{2P+\log\frac{4J}{\delta_0}}+\cbr{\frac{1}{N}\rbr{2P+\log\frac{4J}{\delta_0}}}^{\frac{1}{2}}}
,
\end{align}
with probability at least $1 - (J-1)\delta_0/J^2$.
A union bound,
together with~\eqref{eq:constantphi},
gives us
\[
\dist^2(Z^0,Z^\star)\leq\phi(g,A^\star)\sbr{ \frac{KJ}{N^2}\cbr{8P^2+2\rbr{\log\frac{4J}{\delta_0}}^2}+\frac{KJ}{N}\rbr{4P+2\log\frac{4J}{\delta_0}}}
,\]
with probability at least $1-\delta_0$. 

This shows that 
$\dist^2(Z^0,Z^\star)\leq J\iniC^2$
with probability at least $1-\delta_0$ 
when $N \gtrsim K\rbr{P+\log {J}/{\delta_0}}$. 
From~\eqref{eq:vvvv} and~\eqref{eq:dist_ub} 
we have that
$\|V^0-V^\star \rotmat\|_F^2\leq 5Kg^{-2}\|\sumsampleCov-\sumpCov\|_2^2$. 
From~\eqref{eq:mnrate} we have that 
$\|V-V^\star \rotmat\|_F^2\leq \iniC^2/J$ with high probability
when 
$N \gtrsim K\rbr{P+\log {J}/{\delta_0}}$.
Similarly, 
$\|\diag(a_j)-\rotmat^T\diag(a_j^\star)\rotmat\|_F^2\leq (J-1)\iniC^2/J$
with high probability. 
This shows that Assumption~\ref{assumption_inicondition} is satisfied..

\emph{Step 2.} Let $r=\rkG{\trnG}$. 
Recall that 
$G=Q\Lambda^\dagger Q^T$. Setting
$\Lambda^\dagger_{rr}=\exp(-l^2 r^2)=\trnerror$,
we have
\begin{align}\label{eq:Gaussian_rank}
    \rkG{\trnG}=\cbr{\rbr{\frac{1}{l^2}\log\frac{1}{\trnerror}}^{\frac{1}{2}}\wedge J}.
\end{align}
Proposition~\ref{prop:SSE} then yields
\[
\staterr^2\leq 2(\nu_2\vee\nu_2^2)+2J\max_{j}\|E_j\|_2^2
\]
where
\begin{align*}
    \nu_2 \lesssim \frac{\|A^\star\|_\infty^2}{e_4N}\sbr{\log\frac{1}{\delta_0}+\cbr{\rbr{\frac{1}{\alpha^2}\log\frac{1}{\trnerror}}^{\frac{1}{2}}\wedge J}+Ks^\star\log\frac{P}{s^\star}}
\end{align*}
with probability at least $1-\delta_0$.

If $s^\star\log P/s^\star<PJ$, then 
\[
\log\frac{1}{\delta_0}+K\rkG{\trnG}+Ks^\star\log \frac{P}{s^\star}
\lesssim \rbr{\log\frac{1}{\delta_0}+KJP}.
\]
Therefore $(\nu_2\vee\nu_2^2)\lesssim J\iniC^2$.
Combining
with 
the assumption
$\max_{j\in[J]}\|E_j\|_2 \lesssim \iniC$,
we establish that Assumption~\ref{assumption_statcondition} 
holds with probability at least $1-\delta_0$.

\emph{Step 3.}
Similar to the proof of Theorem~\ref{theorem:combine12},
we combine results from Step 1 and 2
to obtain
\begin{align}\label{eq:errorbound2}
\sum_{j=1}^J\norm{\Sigma_j^I-\pCov_j}_F^2
\lesssim \delta_1 
+ 
{
   \frac{K\rkG{\trnG} 
   + Ks^\star\log (P/s^\star) 
   + \log\delta_0^{-1}}{N}
   +
   J\max_{j\in[J]}\|E_j\|_2^2
} 
+ 
K\gamma^\star\trnerror
\end{align}
after $I\gtrsim\log(1/\delta_1)$ iterations.
We omit details for brevity.

In the final step, we choose a $\trnerror$ that satisfies $\deftrnerrorinq$ and obtains the optimal error of~\eqref{eq:errorbound2}. Plugging~\eqref{eq:Gaussian_rank} into~\eqref{eq:errorbound2}, the optimal choice is $\trnerror^\star\asymp(\gamma^\star l N)^{-1}(\log\gamma^\star l N)^{1/2}$. Note that $\trnerror^\star$ is smaller than $(16\gamma^\star)^{-1}\min_{j\in[J]}\sigma^2_K(\Sigma_j^\star)$, given that $N$ is not too small. Then, plugging $\trnerror^\star$ into~\eqref{eq:errorbound2}, we establish the result.

\section{Known Results}
\begin{lemma}[Theorem~$2.1.12$ in~\citet{nesterov2013introductory}]\label{lemma:convexbound}
For a $L$-smooth and $m$-strongly convex function $h$, we have
\begin{align*}
    \langle\nabla h(X)-\nabla h(Y),X-Y\rangle\geq\frac{m L}{m+L}\|X-Y\|^2+\frac{1}{m+L}\|\nabla h(X)-\nabla h(Y)\|^2.
\end{align*}
\end{lemma}
\begin{lemma} [Davis-Kahan $\sin\theta$ theorem, adapted from~\citet{yu2015useful}]\label{lemma:DKSTT}~\\ 
Let $\defsumsampleCov$ and $\defsumpCov$. ${V}^\star$ is the matrix whose columns are top-K eigenvectors of $\sumpCov$, and ${V}$ is the matrix whose columns are the top-K eigenvectors of $\sumsampleCov$. Assume that the eigengap $\defeigengap$ is bounded away from zero. Then
\begin{align*}
  \|{V}{V}^{ T}-{V}^\star{V}^{\star T}\|_F\leq\frac{2\surd{K}}{g}\|\sumsampleCov-\sumpCov\|_2,\quad  \|{V}{V}^{ T}-{V}^\star{V}^{\star T}\|_2\leq\frac{2}{g}\|\sumsampleCov-\sumpCov\|_2.
\end{align*}
Moreover, we have
\begin{align*}
    \min_{Y\in\mathcal{O}(K)}\|V-V^\star Y\|_F\leq\frac{2\surd{2}}{g}\|\sumsampleCov-\sumpCov\|_F.
\end{align*}
\end{lemma}
\begin{lemma} [Adapted from Lemma~$5.4$ in~\citet{tu2016low}]\label{lemma:distance_bound} For any $X,U\in\mathbb{R}^{P\times K}$, we have
\[
\min_{Y\in\mathcal{O}(K)}\|U-XY\|_F^2\leq\frac{1}{2(\surd{2}-1)\sigma_{K}^2(X)}\|UU^T-XX^T\|_F^2.
\]
\end{lemma}
\begin{lemma}[Matrix Bernstein, adapted from Theorem~$6.6.1$ in~\citet{tropp2015introduction}]\label{lemma:matrixbernstein} Consider a sequence of independent, random, Hermitian matrices $X_1,\ldots,X_N$ with dimension $P$. Moreover, assume that for $n\in[N]$, we have almost surely $\|X_n\|_2\leq L$. Define 
\[
Y=\sum_{n=1}^NX_n-E(X_n),\quad\nu(Y)=\left\|\sum_{n=1}^N\var(X_n)\right\|_2.
\]
Then, for every $t\geq 0$, we have
\begin{align*}
    \pr(\|Y\|_2\geq t)\leq 2P\exp\left\{\frac{-t^2}{2(\nu(Y)+Lt)}\right\}.
\end{align*}
\end{lemma}

\begin{lemma}\label{lemma:bernsteinbound}
Let $x_1, \ldots, x_N$ be independent centered random vectors in $\RR^P$
such that $\|x_n\|_2^2 \leq L$ almost surely and $S=E(N^{-1}\sum_{n=1}^Nx_nx_n^{T})$.
Then
\[
\pr\cbr{ \bignorm{
\frac{1}{N} \sum_{n=1}^N x_nx_n^T - S
  }_2 \geq t 
  } \leq 2P\exp\cbr{\frac{-Nt^2}{2L(\|S\|_2+t)}}.
\]
\end{lemma}
\begin{proof}
For each sample, we have
\[
\var\rbr{x_nx_n^{T}}=E\cbr{\rbr{x_nx_n^{T}}^2}-E\rbr{x_nx_n^T}^2\preceq E\cbr{\|x_n\|_2^2x_nx_n^{T}}\preceq L E\rbr{x_nx_n^T}.
\]
Consequently, we have
\[
\left\|\frac{1}{N}\sum_{n=1}^{N}\var\left\{x_nx_n^{T}\right\}\right\|_2\leq L\|S\|_2.
\]
The result follows from Lemma~\ref{lemma:matrixbernstein}.
\end{proof}

Let $Z$ be a sub-Gaussian random variable, and we define the sub-Gaussian norm as
\begin{align*}
    \|Z\|_{\psi_2}=\sup_{p\geq 1}\frac{1}{\surd{p}}(E|Z|^p)^{\frac{1}{p}}.
\end{align*}
Let $Z$ be a $P$ dimensional Gaussian random vector, then we define the sub-Gaussian norm as
\begin{align}\label{eq:orlicznorm}
    \|Z\|_{\psi_2}=\sup_{x\in\mathcal{S}^{P-1}}\|\langle x,Z\rangle\|_{\psi_2}.
\end{align}

\begin{lemma}[Adapted from Corollary $5.50$ in~\citet{vershynin2010introduction}] \label{lemma:subgaussianbound} Let $x_1,\ldots,x_N$ be independent centered random vectors with sub-Gaussian distribution. Let $\|x_n\|^2_{\psi_2}\leq L$ for every $n\in[N]$. Then, we have
\[
\pr\cbr{ \bignorm{
  \frac{1}{N}\sum_{n=1}^N x_nx_n^T - S
  }_2 \geq t 
  } \leq 2\exp\cbr{2P-e_0N\rbr{\frac{t^2}{L^2}\wedge\frac{t}{L}}},
\]
for some absolute constant $e_0$.
\end{lemma}


\begin{lemma}[Adapted from Proposition~$1.1$ in~\citet{hsu2012tail}]\label{lemma:chisquare_bound} Let $A\in \RR^{P\times P}$ be a matrix, and let $\Sigma=A^TA$. Let $x=(x_1,\ldots,x_P)$ be an isotropic multivariate Gaussian random vector with zero mean. For all $t>0$.
We have
\[
\pr\cbr{\abr{\|Ax\|_2^2-E(\|Ax\|_2^2)}>t}\leq2\exp\sbr{-\cbr{\frac{t^2}{4\|\Sigma\|_F^2}\wedge\frac{t}{2\|\Sigma\|_2}}}
.\]
Moreover, consider $K$ matrices $A_1,\ldots,A_K$, with $\Sigma_k=A_k^TA_k$ for $k\in[K]$ and $x_k\in\RR^{P}$ be isotropic multivariate random vectors. Then, for all $t>0$, we have

\[
\pr\cbr{\abr{\sum_{k=1}^K\|A_kx_k\|_2^2-E(\|A_kx_k\|^2_2)}>t}\leq2\exp\cbr{-\rbr{\frac{t^2}{4\sum_{k=1}^K\|\Sigma_k\|_F^2}\wedge\frac{t}{2\max_{k\in[K]}\|\Sigma_k\|_2}}}
.
\]
\end{lemma}

\begin{proof}[of Lemma~\ref{lemma:chisquare_bound}]
The first part is shown in~\citet{hsu2012tail}, we show the second result. Let $V_k\Lambda_kV_k^T$ be the eigendecomposition of $\Sigma_k$. define $z_k=V_k^Tx_k$ which follows isotropic multivariate Gaussian distribution by the rotation invariance of Gaussian distribution. Then $\|A_kx_k\|_2^2=\sum_{i=1}^P\lambda_{ki}z_{ki}^2$, where $\lambda_{ki}$ is the $i$th diagonal entry of $\Lambda_k$. Then, we apply the chi-square tail inequality~\citep{laurent2000adaptive} and obtain
\begin{align*}
    \pr\cbr{\sum_{k=1}^K\sum_{i=1}^P\lambda_{ki}z_{ki}^2-\sum_{k=1}^K\tr(\Sigma_k)>2\rbr{\varepsilon\sum_{k=1}^K\|\Sigma_k\|_F^2}^{\frac{1}{2}}+2\varepsilon\max_{k\in[K]}\|\Sigma_k\|_2}\leq e^{-\varepsilon}.
\end{align*}
Consequently, for all $t>0$, we have
\begin{align*}
    \pr\rbr{\abr{\sum_{k=1}^K\|A_kx_k\|_2^2-E\rbr{\|A_kx_k\|_2^2}}>t}\leq2\exp\cbr{-\rbr{\frac{t^2}{4\sum_{k=1}^K\|\Sigma_k\|_F^2}\wedge\frac{t}{2\max_{k\in[K]}\|\Sigma_k\|_2}}}.
\end{align*}
Here we complete the proof.
\end{proof}
\section{Additional Empirical Results}
\begingroup
\def\thetable{\ref{tab:Competing methods}}
\begin{table}[ht]
  \caption{Competing methods}
    \fontsize{9pt}{9pt}\selectfont
  \centering
  \begin{tabular}{*{5}l}
    Abbr. & Model  & low-rank& smooth A& sparse V\\
    M1 & Sliding window principal component analysis   & \cmark& \cmark&\xmark\\
    M2 & Hidden Markov model & \xmark &\xmark&\xmark\\
    M3 & Autoregressive hidden Markov model~\citep{poritz1982linear} &\xmark&\cmark&\xmark\\
    M4 & Sparse dictionary learning~\citep{mairal2010online} & \cmark&\xmark&\cmark\\
    M5 & Bayesian structured learning~\citep{andersen2018bayesian} & \cmark&\cmark&\cmark\\
    M6 & Slinding window shrunk covariance~\citep{ledoit2004well}& \xmark& \cmark&\xmark\\
    M* & Spectral initialization (Algorithm~\ref{alg:spectral_ini})& \cmark&\xmark&\xmark\\
    M** & Proposed model (Algorithm~\ref{alg:main})& \cmark&\cmark&\cmark\\
    MQ** & Proposed model (Algorithm~\ref{alg:main}) with QR decomposition step & \cmark&\cmark&\cmark
  \end{tabular}
\end{table}
\addtocounter{table}{-1}
\endgroup
Table~\ref{tab:Competing methods} repeats the list of competing methods we use in the main text. We will use the same abbreviations in the following experiments.
\subsection{More Simulations on Temporal Dynamics}\label{sssec:simulation_jruth}
In this section, we show additional simulation results with different temporal dynamics. Table~\ref{tab:comparisontoydata} and Table~\ref{tab:comparisontoydata_jime} show the average log-Euclidean metric and the running time, repectively, where the data generation process is shown in Figure~\ref{fig:simulation1} in the main text. In the second part, we evaluate the model with discrete switching temporal dynamics as shown in Figure~\ref{fig:square_waveform}. This experiment is to evaluate the performance of the discrete switching case, analogous to the assumption of the hidden Markov model. The results are shown in Table~\ref{tab:comparisontoydata_square}.
\begin{table}[ht]
    \caption{Average log-Euclidean metric of simulated data ($\sigma=0.5$)}
  \label{tab:comparisontoydata}
    \fontsize{9pt}{9pt}\selectfont
  \centering
  \begin{tabular}{*{8}l}
    & \multicolumn{3}{c}{Mixing waveform}&\multicolumn{3}{c}{Sine waveform}\\

    Methods& $N=1$&$N=5$&$N=10$&$N=1$ &$N=5$&$N=10$\\
    
    M1 $W=20$& $0.45\pm0.01$&$0.40\pm0.01$&$0.38\pm0.01$&$0.67\pm0.03$&$0.63\pm0.01$&$0.63\pm0.01$\\
    M2 & $6.58\pm0.31$&$0.68\pm0.02$&$0.62\pm0.01$&$6.22\pm1.02$&$0.82\pm0.01$&$0.80\pm0.01$\\
    M3& $6.91\pm1.39$&$0.75\pm0.02$&$0.65\pm0.01$&$7.36\pm0.24$&$0.87\pm0.02$&$0.80\pm0.01$\\
    M4& $0.46\pm0.01$&$0.43\pm0.02$&$0.39\pm0.02$&$0.64\pm0.01$&$0.58\pm0.03$&$0.54\pm0.04$\\
    M5 &$0.41\pm0.01$&$0.36\pm0.01$&$0.34\pm0.00$&$0.58\pm0.01$&$0.54\pm0.01$&$0.53\pm0.01$\\
    M* & $0.89\pm0.05$&$0.41\pm0.01$&$0.38\pm0.01$&$0.94\pm0.06$&$0.58\pm0.02$&$0.57\pm0.01$\\
    M** & $0.41\pm0.03$&$0.29\pm0.03$&$0.30\pm0.04$&$0.59\pm0.02$&$0.51\pm0.03$&$0.50\pm0.02$\\
    MQ**&$0.37\pm0.02$&$0.31\pm0.03$&$0.29\pm0.03$&$0.58\pm0.02$&$0.56\pm0.02$&$0.56\pm0.01$
  \end{tabular}
\end{table}
\begin{table}[ht]
    \caption{Average running time of simulated data  ($\sigma=0.5$) $(\times 10^{-2}s)$. $\ddagger$ denotes the scaling of $\times10^2s$}
  \label{tab:comparisontoydata_jime}
    \fontsize{9pt}{9pt}\selectfont
  \centering
  \begin{tabular}{*{8}l}
    & \multicolumn{3}{c}{Mixing waveform}&\multicolumn{3}{c}{Sine waveform}\\

    Methods& $N=1$&$N=5$&$N=10$&$N=1$ &$N=5$&$N=10$\\
M1 $W=20$&$0.5\pm0.1$&$0.4\pm0.0$&$0.6\pm0.0$&$0.5\pm0.1$&$0.4\pm0.0$&$0.5\pm0.0$\\
M2&$0.3^\ddagger\pm4.0$&$7.9^\ddagger\pm0.1^\ddagger$&$10.0^\ddagger\pm0.4^\ddagger$&$0.3^\ddagger\pm4.0$&$8.0^\ddagger\pm0.1^\ddagger$&$8.3^\ddagger\pm7.9$\\
M3&$1.9^\ddagger\pm0.3^\ddagger$&$28.4^\ddagger\pm0.1^\ddagger$&$24.5^\ddagger\pm0.7^\ddagger$&$1.8^\ddagger\pm0.2^\ddagger$&$28.4^\ddagger\pm0.2^\ddagger$&$29.0^\ddagger\pm4.1$\\
M4&$0.6^\ddagger\pm0.3^\ddagger$&$5.2^\ddagger\pm2.9^\ddagger$&$10.8^\ddagger\pm5.1^\ddagger$&$0.7^\ddagger\pm0.3^\ddagger$&$3.2^\ddagger\pm1.5^\ddagger$&$7.8^\ddagger\pm3.3^\ddagger$\\
M5&$3.5^\ddagger\pm0.8^\ddagger$&$3.6^\ddagger\pm1.3^\ddagger$&$3.8^\ddagger\pm0.6^\ddagger$&$3.5^\ddagger\pm0.8^\ddagger$&$3.6^\ddagger\pm1.1^\ddagger$&$3.8^\ddagger\pm0.9^\ddagger$\\
M*&$0.1\pm0.0$&$0.2\pm0.0$&$0.2\pm0.0$&$0.1\pm0.0$&$0.2\pm0.0$&$0.2\pm0.0$\\
M**&$8.2\pm1.6$&$4.5\pm0.2$&$5.9\pm0.6$&$6.7\pm1.7$&$4.5\pm0.8$&$4.3\pm1.2$\\
MQ**&$18.1\pm2.2$&$14.0\pm0.6$&$13.4\pm0.3$&$15.3\pm2.4$& $12.9\pm0.8$&$12.0\pm0.6$
  \end{tabular}
\end{table}

\begin{figure}[ht]
    \centering
    \includegraphics[width=0.45\textwidth]{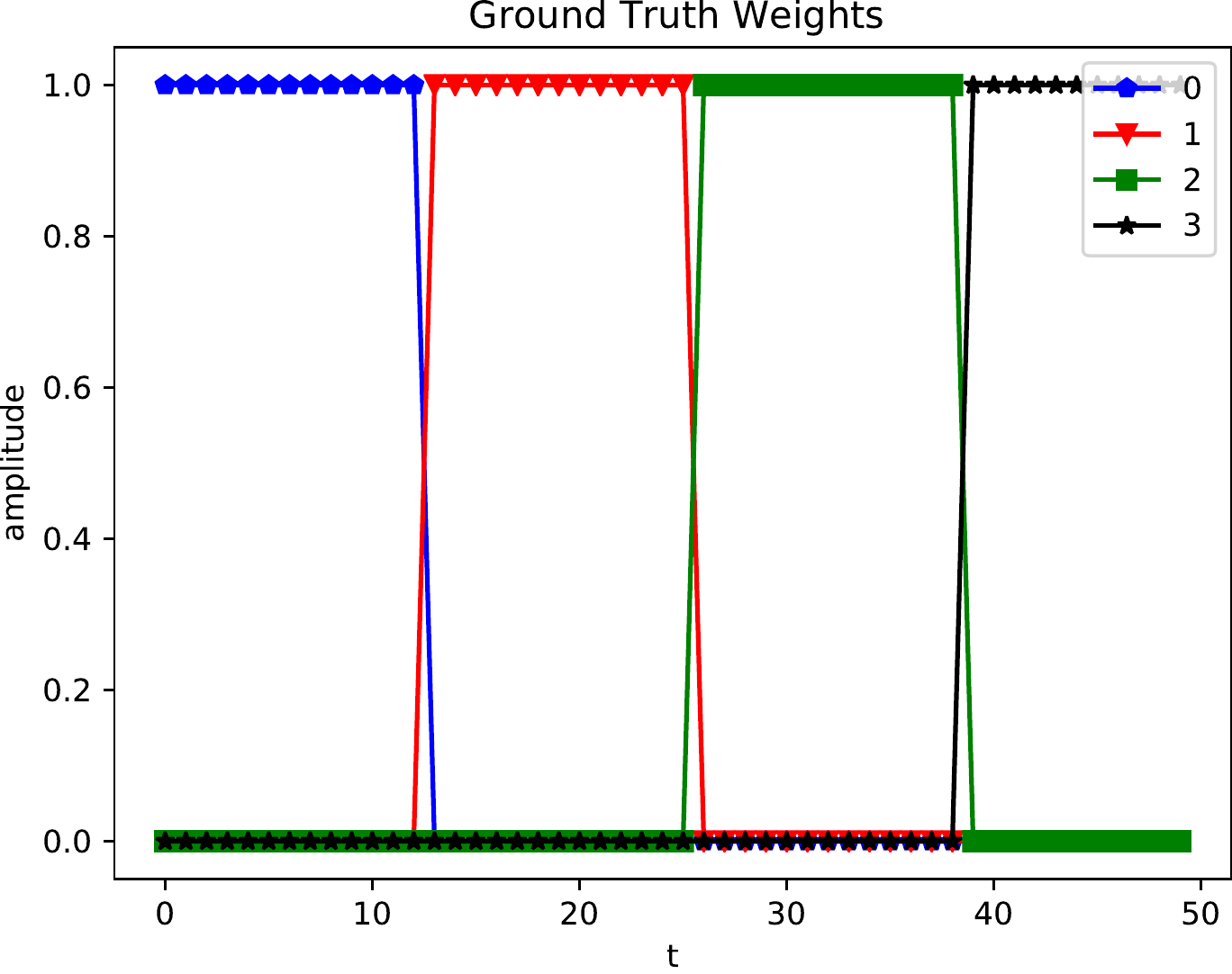}
    \includegraphics[width=0.35\textwidth]{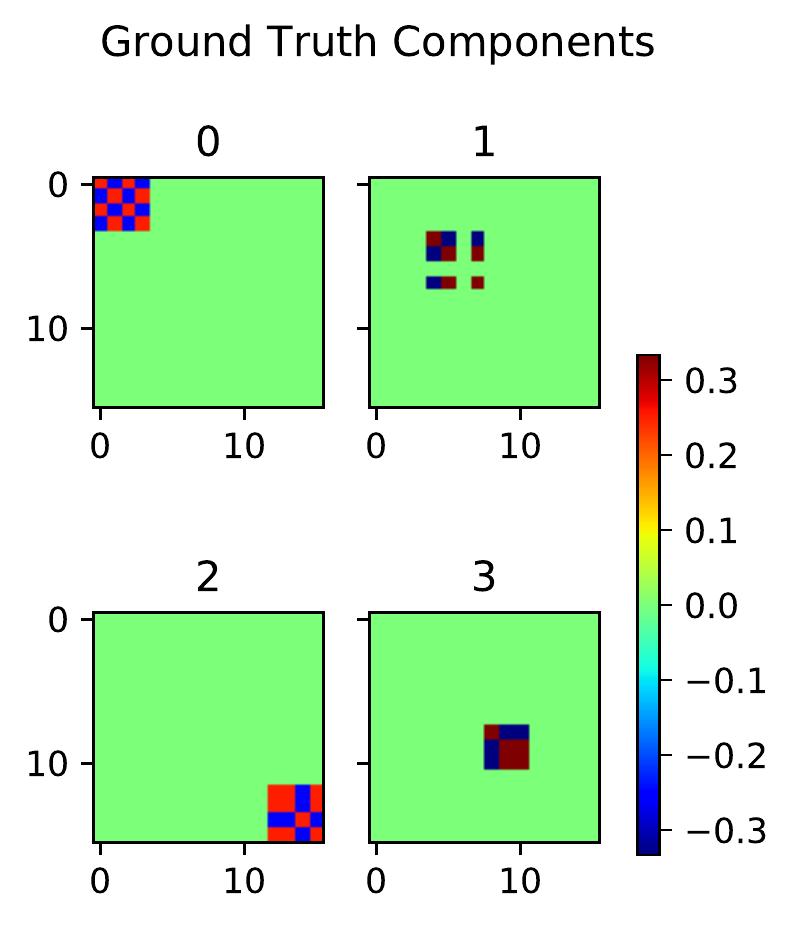}
    \caption{The left figure shows the ground truth of square temporal weights and the right figure shows the corresponding temporal components.}
    \label{fig:square_waveform}
\end{figure}
\begin{table}[ht]
    \caption{Simulation result of Figure~\ref{fig:square_waveform}}
  \label{tab:comparisontoydata_square}
    \fontsize{9pt}{9pt}\selectfont
  \centering
  \begin{tabular}{*{8}l}
    & \multicolumn{3}{c}{Average log-Euclidean metric}&\multicolumn{3}{c}{Running time $(\times 10^{-2} s)$}\\

    Methods& $N=1$&$N=5$&$N=10$&$N=1$ &$N=5$&$N=10$\\
M1&$0.58\pm0.02$&$0.52\pm0.01$&$0.50\pm0.01$&$0.4\pm0.1$&$0.4\pm0.0$&$0.5\pm0.1$\\
M2&$6.41\pm0.42$&$0.70\pm0.01$&$0.66\pm0.01$&$31.4\pm2.9$&$790.8\pm6.3$&$829.7\pm8.4$\\
M3&$8.90\pm3.33$&$0.81\pm0.04$&$0.68\pm0.01$&$173.2\pm21.0$&$2834.0\pm7.3$&$2909.1\pm16.6$\\
M4&$0.58\pm0.01$&$0.50\pm0.03$&$0.43\pm0.03$&$44.8\pm14.2$&$393.1\pm225.4$&$961.6\pm695.1$\\
M5&$0.50\pm0.01$&$0.46\pm0.01$&$0.43\pm0.00$&$3521.0\pm123.2$&$3596.6\pm98.3$&$3800.6\pm113.3$\\
M*&$0.94\pm0.11$&$0.49\pm0.01$&$0.46\pm0.01$&$0.1\pm0.0$&$0.3\pm0.4$&$0.2\pm0.0$\\
M**&$0.55\pm0.03$&$0.42\pm0.07$&$0.40\pm0.09$&$6.8\pm1.6$&$5.9\pm1.0$&$5.7\pm1.5$\\
MQ**&$0.54\pm0.04$&$0.42\pm0.04$&$0.38\pm0.06$&$4.8\pm0.1$&$3.8\pm0.2$&$3.4\pm0.3$
 \end{tabular}
\end{table}

  \subsection{More Experiments on High-Dimensional Data}
  For the data generation process, we randomly generate a sparse orthogonal matrix of ${ V}^\star\in\mathbb{R}^{P\times K}$. We first generate sparse orthogonal block diagonal matrices $\widetilde{ V}$ with dimension $P\times P$ and then compute the QR decomposition of each block. We keep the orthogonal component of the QR decomposition in each block and then randomly permute the row of the matrix. Finally, we randomly pick $K$ columns of $\widetilde{ V}$ to compose ${ V}^\star$. For the temporal components $\rowkA$ for every $k\in[K]$, unless stated otherwise, we randomly select $6$ knots and interpolate the knots with a cubic spline function. The location of each knot is uniformly distributed, with $y$-position drawn from $\Unif([0,1])$ and $x$-position uniformly drawn from $\Unif([0,T])$.
  
  { Selections of kernel functions}: In this experiment, we vary the number of knots to see how the choice of kernel length scale affects the estimations. Moreover, we choose different kernel functions to demonstrate the model generalization. We run the simulation with $N=50$, $K=10$, $P=100$, and $J=100$. The simulation results averaged by $20$ trials are shown in Table~\ref{tab:diffkernels}. The result show that as the number of knots increases, indicating that the temporal signal fluctuates more intensively, the optimal choice of length scale decreases. We observe such behavior in all three kernel functions. As for selecting the kernel function, there is no clear distinction which function is the optimal choice for all cases but may require testing all combinations.
  \begin{table}[ht!]
    \caption{${\dist}^2({ Z},{ Z}^\star)/J$ of different kernel functions and kernel length scale}
  \label{tab:diffkernels}
    \fontsize{9pt}{9pt}\selectfont
  \centering
  \begin{tabular}{*{8}l}
    
    Methods & \multicolumn{4}{l}{Number of knots in $J=100$}\\
    & 5 &10&15&20\\
    
    Radial-basis function ($l=5$)& $0.16\pm0.04$&$0.26\pm0.05$&$0.52\pm0.07$&$0.91\pm0.17$\\
    Radial-basis function ($l=10$)& $0.12\pm0.04$&$0.18\pm0.05$&{$0.33\pm0.07$}&{$0.69\pm0.18$}&\\
    Radial-basis function ($l=50$)&$0.09\pm0.04$&{$0.16\pm0.05$}&$0.76\pm0.09$&$1.17\pm0.11$\\
    Radial-basis function ($l=200$)&{$0.08\pm0.04$}&$0.30\pm0.05$&$0.84\pm0.08$&$1.27\pm0.16$\\
    \\
    Mat\'ern five-half  ($l=5$)& $0.15\pm0.04$& $0.23\pm0.05$& $0.44\pm0.06$& $0.80\pm0.18$ \\
    Mat\'ern five-half  ($l=10$)& $0.13\pm0.04$& $0.20\pm0.05$& $0.36\pm0.06$& $0.70\pm0.18$ \\
    Mat\'ern five-half  ($l=50$)&$0.10\pm0.04$&$0.15\pm0.05$&{$0.31\pm0.06$}&{$0.59\pm0.12$}\\
    Mat\'ern five-half  ($l=200$)&{$0.09\pm0.04$}&{$0.14\pm0.05$}&$0.31\pm0.07$&$0.62\pm0.19$\\
    \\
    Rational quadratic ($l=5$)&$0.14\pm0.04$&$0.24\pm0.05$&$0.51\pm0.07$&{$0.89\pm0.18$} \\
    Rational quadratic ($l=10$)&$0.17\pm0.04$&$0.29\pm0.05$ &$0.61\pm0.07$& $1.03\pm0.18$\\ 
    Rational quadratic ($l=50$)&$0.10\pm0.04$&{$0.13\pm0.05$}&{$0.43\pm0.07$}&$1.07\pm0.18$\\
    Rational quadratic ($l=200$)&{$0.08\pm0.04$}&$0.22\pm0.05$&$0.81\pm0.08$&$1.26\pm0.16$\\
     
  \end{tabular}
\end{table}

More experiments in high-dimensional setting: Table~\ref{tab:comparisonhighdimension02}--\ref{tab:comparisonhighdimension01} show experimental results of $P=100$, $J=100$, $K=10$ with different noise level $\{0.2,0.1\}$. While most methods have improved results as the noise level decrease, M2 and M3 have downgraded results. This may be because M2 and M3 are already poor estimators. 

\begin{table}[ht!]
    \caption{Average log-Euclidean metric of high dimensional low-rank data ($\sigma=0.2$)}
  \label{tab:comparisonhighdimension02}
    \fontsize{9pt}{9pt}\selectfont
  \centering
  \begin{tabular}{*{8}l}
    
    Methods & \multicolumn{4}{c}{Number of training subjects}\\
    & 10 &20&30&40&50\\
    
    M1 $W=20$& $0.40\pm0.01$& $0.37\pm0.01$& $0.36\pm0.01$& $0.36\pm0.01$& $0.35\pm0.01$&\\
    M2 & $2.21\pm0.01$& $2.07\pm0.01$& $2.04\pm0.01$& $2.02\pm0.01$& $2.02\pm0.01$&\\
    M3 & $78.42\pm7.46$& $1.66\pm0.09$& $2.15\pm0.01$& $2.20\pm0.01$& $2.19\pm0.01$&\\
    M4 & $0.96\pm0.02$& $0.43\pm0.02$& $0.36\pm0.04$& $0.36\pm0.03$& $0.35\pm0.03$&\\
    M5&$0.43\pm0.01$&$0.39\pm0.01$&$0.38\pm0.01$&$0.37\pm0.01$&$0.36\pm0.01$\\
    M* & $0.35\pm0.01$& $0.31\pm0.01$& $0.29\pm0.01$& $0.29\pm0.01$& $0.28\pm0.01$&\\
    M** & $0.32\pm0.02$& $0.29\pm0.02$& $0.27\pm0.01$& $0.27\pm0.01$& $0.26\pm0.01$&\\
    MQ** & $0.30\pm0.02$& $0.28\pm0.02$& $0.28\pm0.02$& $0.28\pm0.01$ & $0.27\pm0.01$
  \end{tabular}
\end{table}
\begin{table}[ht!]
    \caption{Average log-Euclidean metric of high dimensional low-rank data ($\sigma=0.1$)}
  \label{tab:comparisonhighdimension01}
    \fontsize{9pt}{9pt}\selectfont
  \centering
  \begin{tabular}{*{8}l}
    
    Methods & \multicolumn{4}{c}{Number of training subjects}\\
    & 10 &20&30&40&50\\
M1 $W=20$& $0.35\pm0.00$& $0.33\pm0.00$& $0.33\pm0.00$& $0.32\pm0.00$& $0.32\pm0.00$&\\
M2& $3.03\pm0.01$& $2.90\pm0.00$& $2.87\pm0.00$& $2.85\pm0.00$& $2.85\pm0.00$&\\
M3& $73.62\pm9.48$& $2.25\pm0.15$& $2.98\pm0.01$& $3.03\pm0.01$& $3.01\pm0.00$&\\
M4& $1.00\pm0.05$& $0.48\pm0.05$& $0.41\pm0.05$& $0.37\pm0.05$& $0.39\pm0.03$&\\
M5&$0.39\pm0.01$&$0.36\pm0.01$&$0.35\pm0.01$&$0.34\pm0.01$&$0.33\pm0.01$\\
 M*& $0.31\pm0.01$& $0.27\pm0.02$& $0.25\pm0.01$& $0.24\pm0.01$& $0.24\pm0.01$&\\
M**& $0.28\pm0.02$& $0.26\pm0.02$& $0.24\pm0.01$& $0.23\pm0.01$& $0.23\pm0.01$&\\
MQ**& $0.29\pm0.02$& $0.26\pm0.02$& $0.25\pm0.01$& $0.24\pm0.01$& $0.23\pm0.01$&
  \end{tabular}
\end{table}

\subsection{Experiment on fMRI Data}
\begin{table}[ht!]
  \caption{The correlation of the components with task activation.}
  \label{tab:correlation_rank}
    \fontsize{9pt}{9pt}\selectfont
  \centering
  \begin{tabular}{*{17}l}
    
    Task & \multicolumn{15}{c}{Rank of the correlation, order from largest to smallest (component index)}\\
    
    Right Hand Tapping&9  &0  &6  &5  &3  &2 &14 &13 &12 &11 &10  &7  &1  &4  &8\\
    Left Foot Tapping &9  &4  &6  &2 &14 &13 &12 &11 &10  &3  &1  &0  &5  &7  &8\\
    Tongue Wagging    &4  &1  &2  &7 &14 &13 &12 &11 &10  &3  &9  &8  &0  &5  &6\\
    Right Foot Tapping&8  &7  &6 &14 &13 &12 &11 &10  &3  &2  &4  &1  &5  &0  &9\\
    Left Hand Tapping &8  &7  &5  &3  &1  &2  &6 &14 &13 &12 &11 &10  &0  &4  &9\\
     
  \end{tabular}
\end{table}
In this section, we provide the remaining experiment result in Figure~\ref{fig:more_taskfmri}, the task correlation in Table~\ref{tab:correlation_rank}, and the task activation map in Figure~\ref{fig:activation_map}.
\begin{figure}[ht!]
\centering
\includegraphics[width=0.4\textwidth]{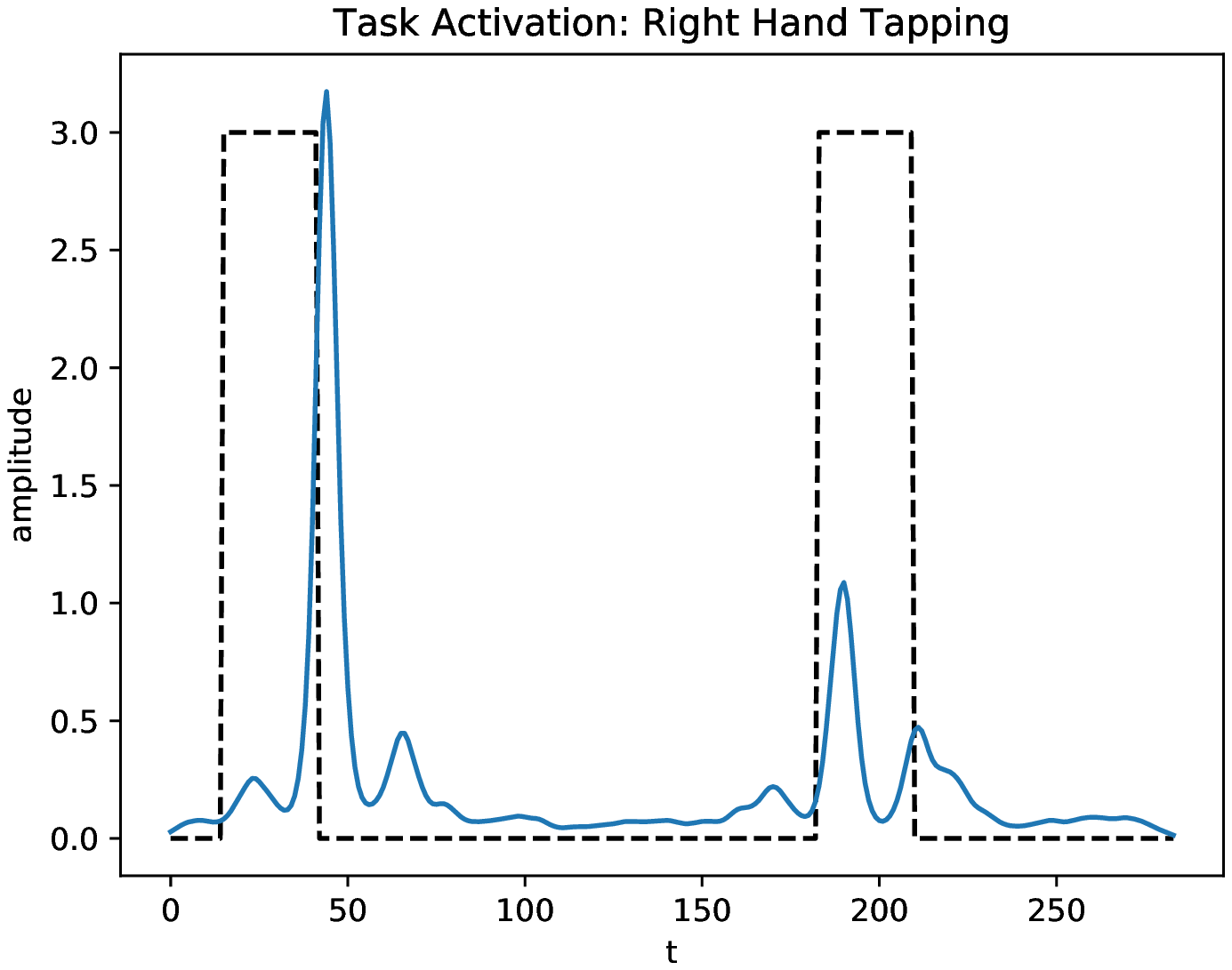}
\includegraphics[width=0.4\textwidth]{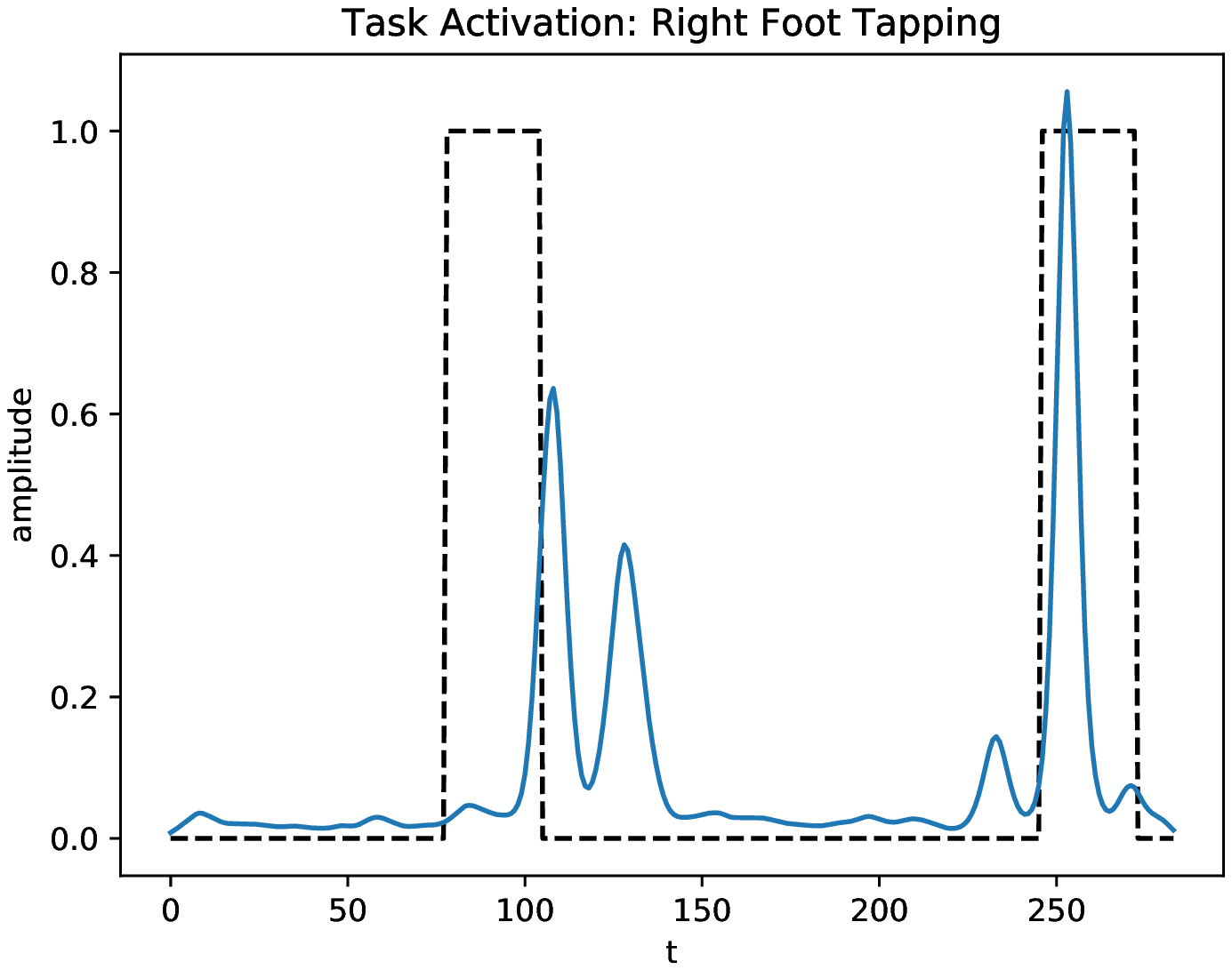}
\includegraphics[width=0.4\textwidth]{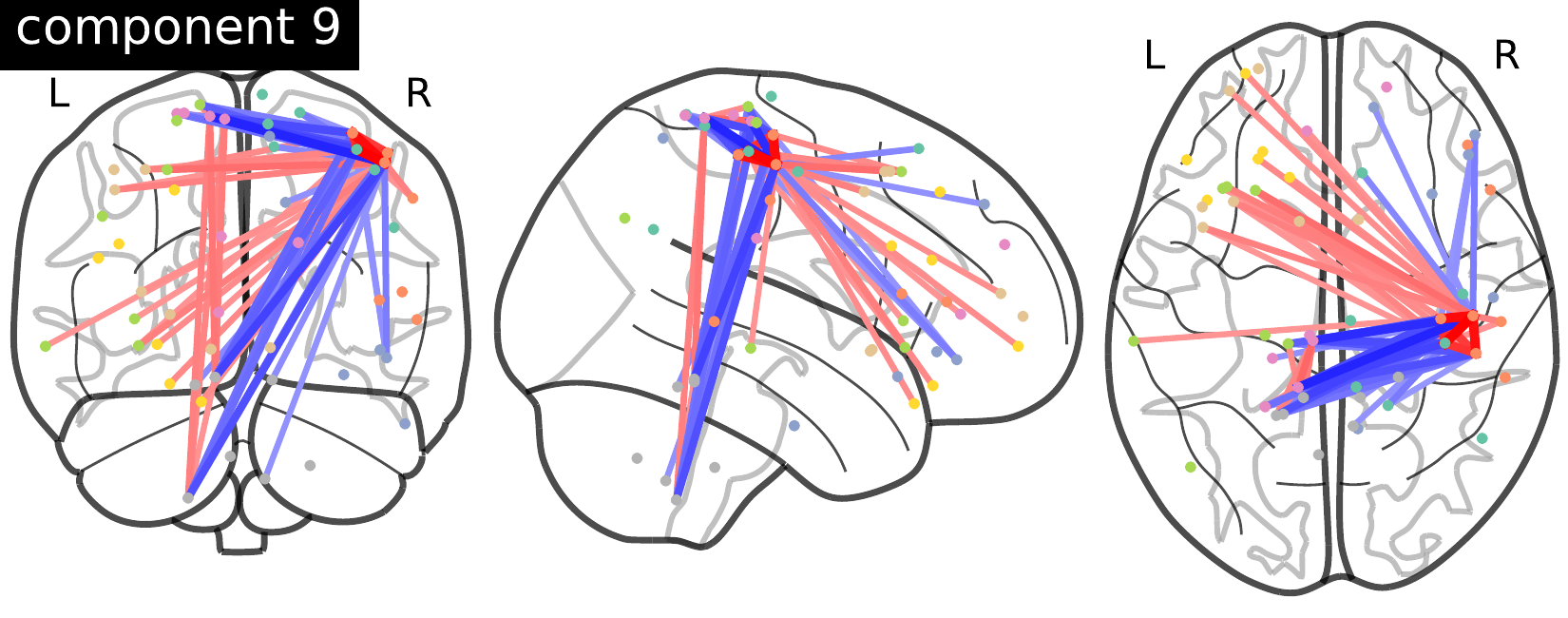}
\includegraphics[width=0.4\textwidth]{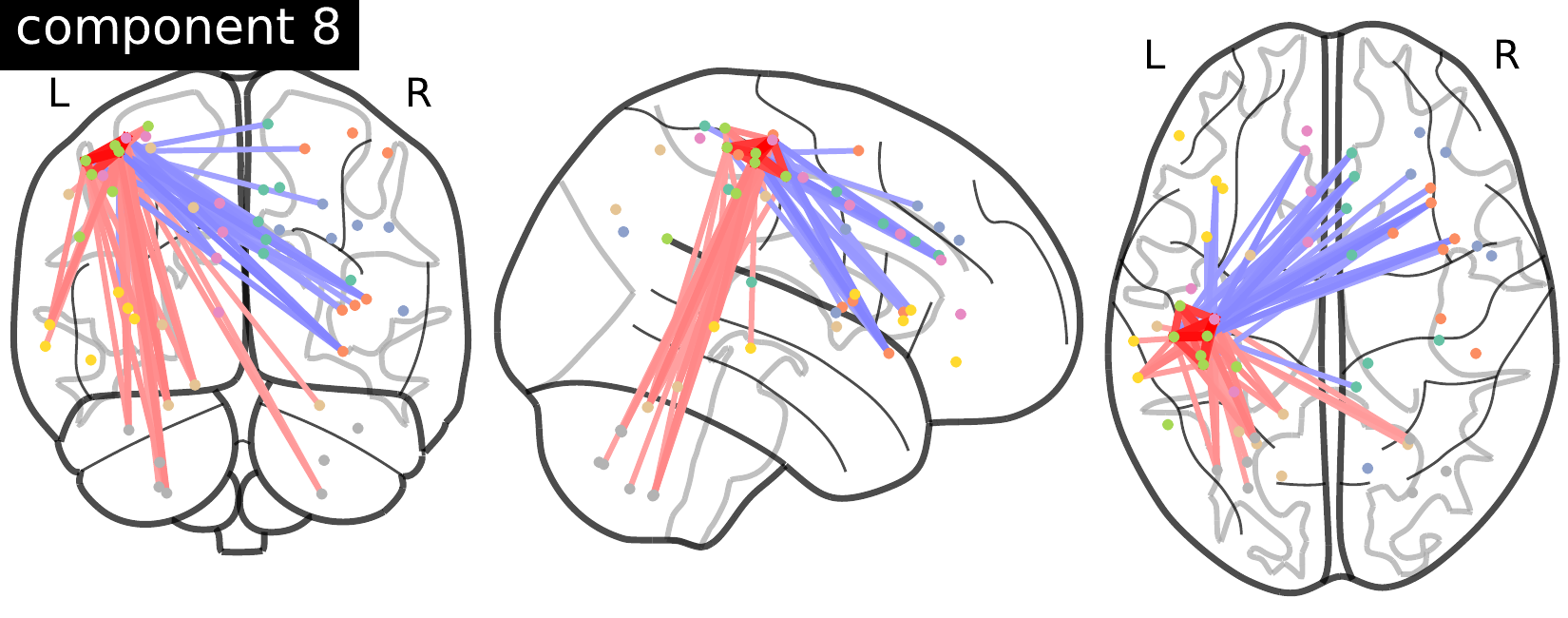}
\caption{ The top row shows the estimated temporal components (blue solid line) whose correlations are the largest with respect to the task activations (black dotted line). The bottom row shows the corresponding spatial component of the above task.}
    \label{fig:more_taskfmri}
\end{figure}
\begin{figure}[ht!]
    \centering
    \includegraphics[width=0.65\textwidth]{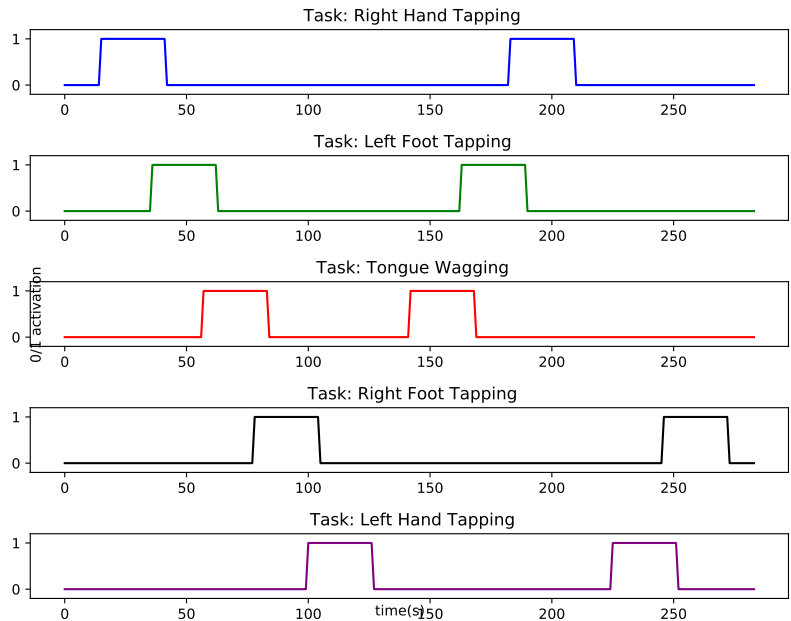}
    \caption{The activation map of the Human Connectome Project motor dataset~\citep{van2013wu}. For each task, the activation time is partially overlapping with the other tasks.}
    \label{fig:activation_map}
\end{figure}
\newpage

%% file: ms2.bbl
\begin{thebibliography}{74}
\expandafter\ifx\csname natexlab\endcsname\relax\def\natexlab#1{#1}\fi

\bibitem[{Ahelegbey et~al.(2016)Ahelegbey, Billio \&
  Casarin}]{ahelegbey2016bayesian}
\textsc{Ahelegbey, D.~F.}, \textsc{Billio, M.} \& \textsc{Casarin, R.} (2016).
\newblock Bayesian graphical models for structural vector autoregressive
  processes.
\newblock \textit{Journal of Applied Econometrics} \textbf{31}, 357--386.

\bibitem[{Andersen et~al.(2018)Andersen, Winther, Hansen, Poldrack \&
  Koyejo}]{andersen2018bayesian}
\textsc{Andersen, M.~R.}, \textsc{Winther, O.}, \textsc{Hansen, L.~K.},
  \textsc{Poldrack, R.} \& \textsc{Koyejo, O.} (2018).
\newblock Bayesian structure learning for dynamic brain connectivity.
\newblock In \textit{21st International Conference on Artificial Intelligence
  and Statistics, AISTATS 2018}.

\bibitem[{Anderson \& Rubin(1956)}]{anderson1956}
\textsc{Anderson, T.~W.} \& \textsc{Rubin, H.} (1956).
\newblock Statistical inference in factor analysis.
\newblock In \textit{Proceedings of the Third Berkeley Symposium on
  Mathematical Statistics and Probability, Volume 5: Contributions to
  Econometrics, Industrial Research, and Psychometry}. Berkeley, Calif.:
  University of California Press.

\bibitem[{Arsigny et~al.(2006)Arsigny, Fillard, Pennec \&
  Ayache}]{arsigny2006log}
\textsc{Arsigny, V.}, \textsc{Fillard, P.}, \textsc{Pennec, X.} \&
  \textsc{Ayache, N.} (2006).
\newblock Log-{E}uclidean metrics for fast and simple calculus on diffusion
  tensors.
\newblock \textit{Magnetic Resonance in Medicine: An Official Journal of the
  International Society for Magnetic Resonance in Medicine} \textbf{56},
  411--421.

\bibitem[{Ba{\'n}bura et~al.(2010)Ba{\'n}bura, Giannone \&
  Reichlin}]{banbura2010large}
\textsc{Ba{\'n}bura, M.}, \textsc{Giannone, D.} \& \textsc{Reichlin, L.}
  (2010).
\newblock Large bayesian vector auto regressions.
\newblock \textit{Journal of applied Econometrics} \textbf{25}, 71--92.

\bibitem[{Bhojanapalli et~al.(2016)Bhojanapalli, Kyrillidis \&
  Sanghavi}]{bhojanapalli2016dropping}
\textsc{Bhojanapalli, S.}, \textsc{Kyrillidis, A.} \& \textsc{Sanghavi, S.}
  (2016).
\newblock Dropping convexity for faster semi-definite optimization.
\newblock In \textit{Conference on Learning Theory}.

\bibitem[{Blei et~al.(2017)Blei, Kucukelbir \& McAuliffe}]{blei2017variational}
\textsc{Blei, D.~M.}, \textsc{Kucukelbir, A.} \& \textsc{McAuliffe, J.~D.}
  (2017).
\newblock Variational inference: A review for statisticians.
\newblock \textit{Journal of the American Statistical Association}
  \textbf{112}, 859--877.

\bibitem[{Burer \& Monteiro(2003)}]{burer2003nonlinear}
\textsc{Burer, S.} \& \textsc{Monteiro, R.~D.} (2003).
\newblock A nonlinear programming algorithm for solving semidefinite programs
  via low-rank factorization.
\newblock \textit{Mathematical Programming} \textbf{95}, 329--357.

\bibitem[{Burer \& Monteiro(2005)}]{burer2005local}
\textsc{Burer, S.} \& \textsc{Monteiro, R.~D.} (2005).
\newblock Local minima and convergence in low-rank semidefinite programming.
\newblock \textit{Mathematical Programming} \textbf{103}, 427--444.

\bibitem[{Calhoun et~al.(2014)Calhoun, Miller, Pearlson \&
  Adal{\i}}]{calhoun2014chronnectome}
\textsc{Calhoun, V.~D.}, \textsc{Miller, R.}, \textsc{Pearlson, G.} \&
  \textsc{Adal{\i}, T.} (2014).
\newblock The chronnectome: Time-varying connectivity networks as the next
  frontier in fmri data discovery.
\newblock \textit{Neuron} \textbf{84}, 262--274.

\bibitem[{Candes et~al.(2015)Candes, Li \& Soltanolkotabi}]{candes2015phase}
\textsc{Candes, E.~J.}, \textsc{Li, X.} \& \textsc{Soltanolkotabi, M.} (2015).
\newblock Phase retrieval via {W}irtinger flow: Theory and algorithms.
\newblock \textit{IEEE Transactions on Information Theory} \textbf{61},
  1985--2007.

\bibitem[{Chang et~al.(2016)Chang, Leopold, Sch{\"o}lvinck, Mandelkow,
  Picchioni, Liu, Frank, Turchi \& Duyn}]{chang2016tracking}
\textsc{Chang, C.}, \textsc{Leopold, D.~A.}, \textsc{Sch{\"o}lvinck, M.~L.},
  \textsc{Mandelkow, H.}, \textsc{Picchioni, D.}, \textsc{Liu, X.},
  \textsc{Frank, Q.~Y.}, \textsc{Turchi, J.~N.} \& \textsc{Duyn, J.~H.} (2016).
\newblock Tracking brain arousal fluctuations with fmri.
\newblock \textit{Proceedings of the National Academy of Sciences} , 201520613.

\bibitem[{Chen \& Candes(2015)}]{chen2015solving}
\textsc{Chen, Y.} \& \textsc{Candes, E.} (2015).
\newblock Solving random quadratic systems of equations is nearly as easy as
  solving linear systems.
\newblock In \textit{Advances in Neural Information Processing Systems}.

\bibitem[{Chen \& Wainwright(2015)}]{chen2015fast}
\textsc{Chen, Y.} \& \textsc{Wainwright, M.~J.} (2015).
\newblock Fast low-rank estimation by projected gradient descent: General
  statistical and algorithmic guarantees.
\newblock \textit{arXiv preprint arXiv:1509.03025} .

\bibitem[{Chi et~al.(2019)Chi, Lu \& Chen}]{chi2019nonconvex}
\textsc{Chi, Y.}, \textsc{Lu, Y.~M.} \& \textsc{Chen, Y.} (2019).
\newblock Nonconvex optimization meets low-rank matrix factorization: An
  overview.
\newblock \textit{IEEE Transactions on Signal Processing} \textbf{67}.

\bibitem[{Danaher et~al.(2014)Danaher, Wang \& Witten}]{danaher2014joint}
\textsc{Danaher, P.}, \textsc{Wang, P.} \& \textsc{Witten, D.~M.} (2014).
\newblock The joint graphical lasso for inverse covariance estimation across
  multiple classes.
\newblock \textit{Journal of the Royal Statistical Society: Series B
  (Statistical Methodology)} \textbf{76}, 373--397.

\bibitem[{Davis et~al.(2016)Davis, Zang \& Zheng}]{davis2016sparse}
\textsc{Davis, R.~A.}, \textsc{Zang, P.} \& \textsc{Zheng, T.} (2016).
\newblock Sparse vector autoregressive modeling.
\newblock \textit{Journal of Computational and Graphical Statistics}
  \textbf{25}, 1077--1096.

\bibitem[{Diedrichsen et~al.(2009)Diedrichsen, Balsters, Flavell, Cussans \&
  Ramnani}]{diedrichsen2009probabilistic}
\textsc{Diedrichsen, J.}, \textsc{Balsters, J.~H.}, \textsc{Flavell, J.},
  \textsc{Cussans, E.} \& \textsc{Ramnani, N.} (2009).
\newblock A probabilistic mr atlas of the human cerebellum.
\newblock \textit{Neuroimage} \textbf{46}, 39--46.

\bibitem[{Eavani et~al.(2012)Eavani, Filipovych, Davatzikos, Satterthwaite, Gur
  \& Gur}]{eavani2012sparse}
\textsc{Eavani, H.}, \textsc{Filipovych, R.}, \textsc{Davatzikos, C.},
  \textsc{Satterthwaite, T.~D.}, \textsc{Gur, R.~E.} \& \textsc{Gur, R.~C.}
  (2012).
\newblock Sparse dictionary learning of resting state fmri networks.
\newblock In \textit{2012 Second International Workshop on Pattern Recognition
  in NeuroImaging}. IEEE.

\bibitem[{Engle et~al.(2019)Engle, Ledoit \& Wolf}]{engle2019large}
\textsc{Engle, R.~F.}, \textsc{Ledoit, O.} \& \textsc{Wolf, M.} (2019).
\newblock Large dynamic covariance matrices.
\newblock \textit{Journal of Business \& Economic Statistics} \textbf{37},
  363--375.

\bibitem[{Escalante \& Raydan(2011)}]{escalante2011alternating}
\textsc{Escalante, R.} \& \textsc{Raydan, M.} (2011).
\newblock \textit{Alternating Projection Methods}, vol.~8.
\newblock SIAM.

\bibitem[{Foti \& Fox(2019)}]{foti2019statistical}
\textsc{Foti, N.~J.} \& \textsc{Fox, E.~B.} (2019).
\newblock Statistical model-based approaches for functional connectivity
  analysis of neuroimaging data.
\newblock \textit{Current opinion in neurobiology} \textbf{55}, 48--54.

\bibitem[{Fox \& Dunson(2015)}]{fox2015bayesian}
\textsc{Fox, E.~B.} \& \textsc{Dunson, D.~B.} (2015).
\newblock Bayesian nonparametric covariance regression.
\newblock \textit{The Journal of Machine Learning Research} \textbf{16},
  2501--2542.

\bibitem[{Fox \& Raichle(2007)}]{fox2007spontaneous}
\textsc{Fox, M.~D.} \& \textsc{Raichle, M.~E.} (2007).
\newblock Spontaneous fluctuations in brain activity observed with functional
  magnetic resonance imaging.
\newblock \textit{Nature Reviews Neuroscience} \textbf{8}, 700--711.

\bibitem[{Gibberd \& Nelson(2017)}]{gibberd2017regularized}
\textsc{Gibberd, A.~J.} \& \textsc{Nelson, J.~D.} (2017).
\newblock Regularized estimation of piecewise constant gaussian graphical
  models: The group-fused graphical lasso.
\newblock \textit{Journal of Computational and Graphical Statistics}
  \textbf{26}, 623--634.

\bibitem[{Gordon et~al.(2016)Gordon, Laumann, Adeyemo, Huckins, Kelley \&
  Petersen}]{gordon2016generation}
\textsc{Gordon, E.~M.}, \textsc{Laumann, T.~O.}, \textsc{Adeyemo, B.},
  \textsc{Huckins, J.~F.}, \textsc{Kelley, W.~M.} \& \textsc{Petersen, S.~E.}
  (2016).
\newblock Generation and evaluation of a cortical area parcellation from
  resting-state correlations.
\newblock \textit{Cerebral cortex} \textbf{26}, 288--303.

\bibitem[{Gu et~al.(2016)Gu, Wang \& Liu}]{gu2016low}
\textsc{Gu, Q.}, \textsc{Wang, Z.~W.} \& \textsc{Liu, H.} (2016).
\newblock Low-rank and sparse structure pursuit via alternating minimization.
\newblock In \textit{Proceedings of Machine Learning Research}, A.~Gretton \&
  C.~C. Robert, eds., vol.~51. PMLR.

\bibitem[{Hallac et~al.(2017)Hallac, Park, Boyd \&
  Leskovec}]{hallac2017network}
\textsc{Hallac, D.}, \textsc{Park, Y.}, \textsc{Boyd, S.} \& \textsc{Leskovec,
  J.} (2017).
\newblock Network inference via the time-varying graphical lasso.
\newblock In \textit{Proceedings of the 23rd ACM SIGKDD International
  Conference on Knowledge Discovery and Data Mining}.

\bibitem[{Hardt(2014)}]{hardt2014understanding}
\textsc{Hardt, M.} (2014).
\newblock Understanding alternating minimization for matrix completion.
\newblock In \textit{2014 IEEE 55th Annual Symposium on Foundations of Computer
  Science}. IEEE.

\bibitem[{Jain et~al.(2013)Jain, Netrapalli \& Sanghavi}]{jain2013low}
\textsc{Jain, P.}, \textsc{Netrapalli, P.} \& \textsc{Sanghavi, S.} (2013).
\newblock Low-rank matrix completion using alternating minimization.
\newblock In \textit{Proceedings of the Forty-Fifth Annual ACM Symposium on
  Theory of Computing}.

\bibitem[{Kastner et~al.(2017)Kastner, Fr{\"u}hwirth-Schnatter \&
  Lopes}]{kastner2017efficient}
\textsc{Kastner, G.}, \textsc{Fr{\"u}hwirth-Schnatter, S.} \& \textsc{Lopes,
  H.~F.} (2017).
\newblock Efficient {B}ayesian inference for multivariate factor stochastic
  volatility models.
\newblock \textit{Journal of Computational and Graphical Statistics}
  \textbf{26}, 905--917.

\bibitem[{Kolar et~al.(2010)Kolar, Song, Ahmed \& Xing}]{kolar2010estimating}
\textsc{Kolar, M.}, \textsc{Song, L.}, \textsc{Ahmed, A.} \& \textsc{Xing,
  E.~P.} (2010).
\newblock Estimating time-varying networks.
\newblock \textit{The Annals of Applied Statistics} \textbf{4}, 94--123.

\bibitem[{Kumar et~al.(2020)Kumar, Ying, de~Miranda~Cardoso \&
  Palomar}]{kumar2020unified}
\textsc{Kumar, S.}, \textsc{Ying, J.}, \textsc{de~Miranda~Cardoso, J.~V.} \&
  \textsc{Palomar, D.~P.} (2020).
\newblock A unified framework for structured graph learning via spectral
  constraints.
\newblock \textit{Journal of Machine Learning Research} \textbf{21}, 1--60.

\bibitem[{Ledoit \& Wolf(2004)}]{ledoit2004well}
\textsc{Ledoit, O.} \& \textsc{Wolf, M.} (2004).
\newblock A well-conditioned estimator for large-dimensional covariance
  matrices.
\newblock \textit{Journal of multivariate analysis} \textbf{88}, 365--411.

\bibitem[{Leonardi \& Van De~Ville(2015)}]{leonardi2015spurious}
\textsc{Leonardi, N.} \& \textsc{Van De~Ville, D.} (2015).
\newblock On spurious and real fluctuations of dynamic functional connectivity
  during rest.
\newblock \textit{Neuroimage} \textbf{104}, 430--436.

\bibitem[{Li(2019)}]{li2019multivariate}
\textsc{Li, R.} (2019).
\newblock Multivariate sparse coding of nonstationary covariances with
  {G}aussian processes.
\newblock In \textit{Advances in Neural Information Processing Systems}.

\bibitem[{Li et~al.(2016)Li, Zhao, Arora, Liu \& Haupt}]{li2016stochastic}
\textsc{Li, X.}, \textsc{Zhao, T.}, \textsc{Arora, R.}, \textsc{Liu, H.} \&
  \textsc{Haupt, J.} (2016).
\newblock Stochastic variance reduced optimization for nonconvex sparse
  learning.
\newblock In \textit{International Conference on Machine Learning}.

\bibitem[{Li{\'e}geois et~al.(2019)Li{\'e}geois, Li, Kong, Orban, Van De~Ville,
  Ge, Sabuncu \& Yeo}]{liegeois2019resting}
\textsc{Li{\'e}geois, R.}, \textsc{Li, J.}, \textsc{Kong, R.}, \textsc{Orban,
  C.}, \textsc{Van De~Ville, D.}, \textsc{Ge, T.}, \textsc{Sabuncu, M.~R.} \&
  \textsc{Yeo, B.~T.} (2019).
\newblock Resting brain dynamics at different timescales capture distinct
  aspects of human behavior.
\newblock \textit{Nature communications} \textbf{10}, 1--9.

\bibitem[{Loh \& Wainwright(2015)}]{loh2015regularized}
\textsc{Loh, P.-L.} \& \textsc{Wainwright, M.~J.} (2015).
\newblock Regularized m-estimators with nonconvexity: Statistical and
  algorithmic theory for local optima.
\newblock \textit{The Journal of Machine Learning Research} \textbf{16},
  559--616.

\bibitem[{Mairal et~al.(2010)Mairal, Bach, Ponce \& Sapiro}]{mairal2010online}
\textsc{Mairal, J.}, \textsc{Bach, F.}, \textsc{Ponce, J.} \& \textsc{Sapiro,
  G.} (2010).
\newblock Online learning for matrix factorization and sparse coding.
\newblock \textit{Journal of Machine Learning Research} \textbf{11}, 19--60.

\bibitem[{Marieb \& Hoehn(2007)}]{marieb2007human}
\textsc{Marieb, E.~N.} \& \textsc{Hoehn, K.} (2007).
\newblock \textit{Human anatomy \& physiology}.
\newblock Pearson education.

\bibitem[{Minasny \& McBratney(2005)}]{minasny2005matern}
\textsc{Minasny, B.} \& \textsc{McBratney, A.~B.} (2005).
\newblock The mat{\'e}rn function as a general model for soil variograms.
\newblock \textit{Geoderma} \textbf{128}, 192--207.

\bibitem[{Mishne \& Charles(2019)}]{mishne2019learning}
\textsc{Mishne, G.} \& \textsc{Charles, A.~S.} (2019).
\newblock Learning spatially-correlated temporal dictionaries for calcium
  imaging.
\newblock In \textit{ICASSP 2019-2019 IEEE International Conference on
  Acoustics, Speech and Signal Processing (ICASSP)}. IEEE.

\bibitem[{Nesterov(2013)}]{nesterov2013introductory}
\textsc{Nesterov, Y.} (2013).
\newblock \textit{Introductory Lectures on Convex Optimization: A Basic
  Course}, vol.~87.
\newblock Springer Science \& Business Media.

\bibitem[{Olshausen \& Field(1997)}]{olshausen1997sparse}
\textsc{Olshausen, B.~A.} \& \textsc{Field, D.~J.} (1997).
\newblock Sparse coding with an overcomplete basis set: A strategy employed by
  v1?
\newblock \textit{Vision research} \textbf{37}, 3311--3325.

\bibitem[{Paciorek(2003)}]{paciorek2003nonstationary}
\textsc{Paciorek, C.~J.} (2003).
\newblock \textit{Nonstationary Gaussian processes for regression and spatial
  modelling}.
\newblock Ph.D. thesis, Citeseer.

\bibitem[{Park et~al.(2018)Park, Kyrillidis, Caramanis \&
  Sanghavi}]{park2018finding}
\textsc{Park, D.}, \textsc{Kyrillidis, A.}, \textsc{Caramanis, C.} \&
  \textsc{Sanghavi, S.} (2018).
\newblock Finding low-rank solutions via nonconvex matrix factorization,
  efficiently and provably.
\newblock \textit{SIAM Journal on Imaging Sciences} \textbf{11}, 2165--2204.

\bibitem[{Poritz(1982)}]{poritz1982linear}
\textsc{Poritz, A.} (1982).
\newblock Linear predictive hidden markov models and the speech signal.
\newblock In \textit{ICASSP'82. IEEE International Conference on Acoustics,
  Speech, and Signal Processing}, vol.~7. IEEE.

\bibitem[{Posner et~al.(1988{\natexlab{a}})Posner, Petersen, Fox \&
  Raichle}]{posner1988a}
\textsc{Posner, M.~I.}, \textsc{Petersen, S.~E.}, \textsc{Fox, P.~T.} \&
  \textsc{Raichle, M.~E.} (1988{\natexlab{a}}).
\newblock Localization of cognitive operations in the human brain.
\newblock \textit{Science} \textbf{240}, 1627--1631.

\bibitem[{Posner et~al.(1988{\natexlab{b}})Posner, Petersen, Fox \&
  Raichle}]{posner1988localization}
\textsc{Posner, M.~I.}, \textsc{Petersen, S.~E.}, \textsc{Fox, P.~T.} \&
  \textsc{Raichle, M.~E.} (1988{\natexlab{b}}).
\newblock Localization of cognitive operations in the human brain.
\newblock \textit{Science} \textbf{240}, 1627--1631.

\bibitem[{Preti et~al.(2017)Preti, Bolton \& Van De~Ville}]{preti2017dynamic}
\textsc{Preti, M.~G.}, \textsc{Bolton, T.~A.} \& \textsc{Van De~Ville, D.}
  (2017).
\newblock The dynamic functional connectome: State-of-the-art and perspectives.
\newblock \textit{Neuroimage} \textbf{160}, 41--54.

\bibitem[{Qiao et~al.(2020)Qiao, Qian, James \& Guo}]{qiao2020doubly}
\textsc{Qiao, X.}, \textsc{Qian, C.}, \textsc{James, G.~M.} \& \textsc{Guo, S.}
  (2020).
\newblock Doubly functional graphical models in high dimensions.
\newblock \textit{Biometrika} \textbf{107}, 415--431.

\bibitem[{Qiu et~al.(2016)Qiu, Han, Liu \& Caffo}]{qiu2016joint}
\textsc{Qiu, H.}, \textsc{Han, F.}, \textsc{Liu, H.} \& \textsc{Caffo, B.}
  (2016).
\newblock Joint estimation of multiple graphical models from high dimensional
  time series.
\newblock \textit{Journal of the Royal Statistical Society: Series B
  (Statistical Methodology)} \textbf{78}, 487--504.

\bibitem[{Sako{\u{g}}lu et~al.(2010)Sako{\u{g}}lu, Pearlson, Kiehl, Wang,
  Michael \& Calhoun}]{sakouglu2010method}
\textsc{Sako{\u{g}}lu, {\"U}.}, \textsc{Pearlson, G.~D.}, \textsc{Kiehl,
  K.~A.}, \textsc{Wang, Y.~M.}, \textsc{Michael, A.~M.} \& \textsc{Calhoun,
  V.~D.} (2010).
\newblock A method for evaluating dynamic functional network connectivity and
  task-modulation: application to schizophrenia.
\newblock \textit{Magnetic Resonance Materials in Physics, Biology and
  Medicine} \textbf{23}, 351--366.

\bibitem[{Sch{\"o}lkopf et~al.(2002)Sch{\"o}lkopf, Smola, Bach
  et~al.}]{scholkopf2002learning}
\textsc{Sch{\"o}lkopf, B.}, \textsc{Smola, A.~J.}, \textsc{Bach, F.} et~al.
  (2002).
\newblock \textit{Learning with kernels: support vector machines,
  regularization, optimization, and beyond}.
\newblock MIT press.

\bibitem[{Shine et~al.(2016{\natexlab{a}})Shine, Bissett, Bell, Koyejo,
  Balsters, Gorgolewski, Moodie \& Poldrack}]{shine2016dynamics}
\textsc{Shine, J.~M.}, \textsc{Bissett, P.~G.}, \textsc{Bell, P.~T.},
  \textsc{Koyejo, O.}, \textsc{Balsters, J.~H.}, \textsc{Gorgolewski, K.~J.},
  \textsc{Moodie, C.~A.} \& \textsc{Poldrack, R.~A.} (2016{\natexlab{a}}).
\newblock The dynamics of functional brain networks: {I}ntegrated network
  states during cognitive task performance.
\newblock \textit{Neuron} \textbf{92}, 544--554.

\bibitem[{Shine et~al.(2019)Shine, Breakspear, Bell, Martens, Shine, Koyejo,
  Sporns \& Poldrack}]{shine2019human}
\textsc{Shine, J.~M.}, \textsc{Breakspear, M.}, \textsc{Bell, P.~T.},
  \textsc{Martens, K. A.~E.}, \textsc{Shine, R.}, \textsc{Koyejo, O.},
  \textsc{Sporns, O.} \& \textsc{Poldrack, R.~A.} (2019).
\newblock Human cognition involves the dynamic integration of neural activity
  and neuromodulatory systems.
\newblock \textit{Nature neuroscience} \textbf{22}, 289--296.

\bibitem[{Shine et~al.(2016{\natexlab{b}})Shine, Koyejo \&
  Poldrack}]{shine2016temporal}
\textsc{Shine, J.~M.}, \textsc{Koyejo, O.} \& \textsc{Poldrack, R.~A.}
  (2016{\natexlab{b}}).
\newblock Temporal metastates are associated with differential patterns of
  time-resolved connectivity, network topology, and attention.
\newblock \textit{Proceedings of the National Academy of Sciences}
  \textbf{113}, 9888--9891.

\bibitem[{Skripnikov \& Michailidis(2019)}]{SKRIPNIKOV2019164}
\textsc{Skripnikov, A.} \& \textsc{Michailidis, G.} (2019).
\newblock Regularized joint estimation of related vector autoregressive models.
\newblock \textit{Computational Statistics $\&$ Data Analysis} \textbf{139},
  164 -- 177.

\bibitem[{Smith et~al.(1999)Smith, Lewis, Ruttimann, Frank, Sinnwell, Yang,
  Duyn \& Frank}]{smith1999investigation}
\textsc{Smith, A.~M.}, \textsc{Lewis, B.~K.}, \textsc{Ruttimann, U.~E.},
  \textsc{Frank, Q.~Y.}, \textsc{Sinnwell, T.~M.}, \textsc{Yang, Y.},
  \textsc{Duyn, J.~H.} \& \textsc{Frank, J.~A.} (1999).
\newblock Investigation of low frequency drift in fmri signal.
\newblock \textit{Neuroimage} \textbf{9}, 526--533.

\bibitem[{Stewart(1977)}]{stewart1977perturbation}
\textsc{Stewart, G.} (1977).
\newblock Perturbation bounds for the qr factorization of a matrix.
\newblock \textit{SIAM Journal on Numerical Analysis} \textbf{14}, 509--518.

\bibitem[{Tank et~al.(2019)Tank, Fox \& Shojaie}]{tank2019identifiability}
\textsc{Tank, A.}, \textsc{Fox, E.~B.} \& \textsc{Shojaie, A.} (2019).
\newblock Identifiability and estimation of structural vector autoregressive
  models for subsampled and mixed-frequency time series.
\newblock \textit{Biometrika} \textbf{106}, 433--452.

\bibitem[{Ten~Berge(1977)}]{ten1977orthogonal}
\textsc{Ten~Berge, J.~M.} (1977).
\newblock Orthogonal procrustes rotation for two or more matrices.
\newblock \textit{Psychometrika} \textbf{42}, 267--276.

\bibitem[{Tropp(2015)}]{tropp2015introduction}
\textsc{Tropp, J.~A.} (2015).
\newblock An introduction to matrix concentration inequalities.
\newblock \textit{Foundations and Trends{\textregistered} in Machine Learning}
  \textbf{8}, 1--230.

\bibitem[{Udell et~al.(2016)Udell, Horn, Zadeh, Boyd
  et~al.}]{udell2016generalized}
\textsc{Udell, M.}, \textsc{Horn, C.}, \textsc{Zadeh, R.}, \textsc{Boyd, S.}
  et~al. (2016).
\newblock Generalized low rank models.
\newblock \textit{Foundations and Trends{\textregistered} in Machine Learning}
  \textbf{9}, 1--118.

\bibitem[{Udell \& Townsend(2019)}]{udell2019big}
\textsc{Udell, M.} \& \textsc{Townsend, A.} (2019).
\newblock Why are big data matrices approximately low rank?
\newblock \textit{SIAM Journal on Mathematics of Data Science} \textbf{1},
  144--160.

\bibitem[{Van~Essen et~al.(2013)Van~Essen, Smith, Barch, Behrens, Yacoub,
  Ugurbil, Consortium et~al.}]{van2013wu}
\textsc{Van~Essen, D.~C.}, \textsc{Smith, S.~M.}, \textsc{Barch, D.~M.},
  \textsc{Behrens, T.~E.}, \textsc{Yacoub, E.}, \textsc{Ugurbil, K.},
  \textsc{Consortium, W.-M.~H.} et~al. (2013).
\newblock The {WU}-{M}inn human connectome project: {A}n overview.
\newblock \textit{Neuroimage} \textbf{80}, 62--79.

\bibitem[{Vershynin(2010)}]{vershynin2010introduction}
\textsc{Vershynin, R.} (2010).
\newblock Introduction to the non-asymptotic analysis of random matrices.
\newblock \textit{arXiv preprint arXiv:1011.3027} .

\bibitem[{Vidaurre et~al.(2017)Vidaurre, Smith \& Woolrich}]{vidaurre2017brain}
\textsc{Vidaurre, D.}, \textsc{Smith, S.~M.} \& \textsc{Woolrich, M.~W.}
  (2017).
\newblock Brain network dynamics are hierarchically organized in time.
\newblock \textit{Proceedings of the National Academy of Sciences}
  \textbf{114}, 12827--12832.

\bibitem[{Yin et~al.(2010)Yin, Geng, Li \& Wang}]{yin2010nonparametric}
\textsc{Yin, J.}, \textsc{Geng, Z.}, \textsc{Li, R.} \& \textsc{Wang, H.}
  (2010).
\newblock Nonparametric covariance model.
\newblock \textit{Statistica Sinica} \textbf{20}, 469.

\bibitem[{Yu et~al.(2020)Yu, Gupta \& Kolar}]{yu2020recovery}
\textsc{Yu, M.}, \textsc{Gupta, V.} \& \textsc{Kolar, M.} (2020).
\newblock Recovery of simultaneous low rank and two-way sparse coefficient
  matrices, a nonconvex approach.
\newblock \textit{Electronic Journal of Statistics} \textbf{14}, 413--457.

\bibitem[{Zalesky et~al.(2012)Zalesky, Fornito \& Bullmore}]{zalesky2012use}
\textsc{Zalesky, A.}, \textsc{Fornito, A.} \& \textsc{Bullmore, E.} (2012).
\newblock On the use of correlation as a measure of network connectivity.
\newblock \textit{Neuroimage} \textbf{60}, 2096--2106.

\bibitem[{Zhang \& Li(2019)}]{kar68837}
\textsc{Zhang, J.} \& \textsc{Li, J.} (2019).
\newblock Factorized estimation of high-dimensional nonparametric covariance
  models.
\newblock \textit{Annals of Statistics} .

\bibitem[{Zhao et~al.(2015)Zhao, Wang \& Liu}]{zhao2015nonconvex}
\textsc{Zhao, T.}, \textsc{Wang, Z.} \& \textsc{Liu, H.} (2015).
\newblock A nonconvex optimization framework for low rank matrix estimation.
\newblock \textit{Advances in Neural Information Processing Systems}
  \textbf{28}, 559.

\end{thebibliography}


\begin{thebibliography}{14}
\expandafter\ifx\csname natexlab\endcsname\relax\def\natexlab#1{#1}\fi

\bibitem[{Andersen et~al.(2018)Andersen, Winther, Hansen, Poldrack \&
  Koyejo}]{andersen2018bayesian}
\textsc{Andersen, M.~R.}, \textsc{Winther, O.}, \textsc{Hansen, L.~K.},
  \textsc{Poldrack, R.} \& \textsc{Koyejo, O.} (2018).
\newblock Bayesian structure learning for dynamic brain connectivity.
\newblock In \textit{21st International Conference on Artificial Intelligence
  and Statistics, AISTATS 2018}.

\bibitem[{Hsu et~al.(2012)Hsu, Kakade, Zhang et~al.}]{hsu2012tail}
\textsc{Hsu, D.}, \textsc{Kakade, S.}, \textsc{Zhang, T.} et~al. (2012).
\newblock A tail inequality for quadratic forms of subgaussian random vectors.
\newblock \textit{Electronic Communications in Probability} \textbf{17}.

\bibitem[{Laurent \& Massart(2000)}]{laurent2000adaptive}
\textsc{Laurent, B.} \& \textsc{Massart, P.} (2000).
\newblock Adaptive estimation of a quadratic functional by model selection.
\newblock \textit{Annals of Statistics} , 1302--1338.

\bibitem[{Ledoit \& Wolf(2004)}]{ledoit2004well}
\textsc{Ledoit, O.} \& \textsc{Wolf, M.} (2004).
\newblock A well-conditioned estimator for large-dimensional covariance
  matrices.
\newblock \textit{Journal of multivariate analysis} \textbf{88}, 365--411.

\bibitem[{Li et~al.(2016)Li, Zhao, Arora, Liu \& Haupt}]{li2016stochastic}
\textsc{Li, X.}, \textsc{Zhao, T.}, \textsc{Arora, R.}, \textsc{Liu, H.} \&
  \textsc{Haupt, J.} (2016).
\newblock Stochastic variance reduced optimization for nonconvex sparse
  learning.
\newblock In \textit{International Conference on Machine Learning}.

\bibitem[{Mairal et~al.(2010)Mairal, Bach, Ponce \& Sapiro}]{mairal2010online}
\textsc{Mairal, J.}, \textsc{Bach, F.}, \textsc{Ponce, J.} \& \textsc{Sapiro,
  G.} (2010).
\newblock Online learning for matrix factorization and sparse coding.
\newblock \textit{Journal of Machine Learning Research} \textbf{11}, 19--60.

\bibitem[{Nesterov(2013)}]{nesterov2013introductory}
\textsc{Nesterov, Y.} (2013).
\newblock \textit{Introductory Lectures on Convex Optimization: A Basic
  Course}, vol.~87.
\newblock Springer Science \& Business Media.

\bibitem[{Poritz(1982)}]{poritz1982linear}
\textsc{Poritz, A.} (1982).
\newblock Linear predictive hidden markov models and the speech signal.
\newblock In \textit{ICASSP'82. IEEE International Conference on Acoustics,
  Speech, and Signal Processing}, vol.~7. IEEE.

\bibitem[{Stewart(1977)}]{stewart1977perturbation}
\textsc{Stewart, G.} (1977).
\newblock Perturbation bounds for the qr factorization of a matrix.
\newblock \textit{SIAM Journal on Numerical Analysis} \textbf{14}, 509--518.

\bibitem[{Tropp(2015)}]{tropp2015introduction}
\textsc{Tropp, J.~A.} (2015).
\newblock An introduction to matrix concentration inequalities.
\newblock \textit{Foundations and Trends{\textregistered} in Machine Learning}
  \textbf{8}, 1--230.

\bibitem[{Tu et~al.(2016)Tu, Boczar, Simchowitz, Soltanolkotabi \&
  Recht}]{tu2016low}
\textsc{Tu, S.}, \textsc{Boczar, R.}, \textsc{Simchowitz, M.},
  \textsc{Soltanolkotabi, M.} \& \textsc{Recht, B.} (2016).
\newblock Low-rank solutions of linear matrix equations via procrustes flow.
\newblock In \textit{International Conference on Machine Learning}. PMLR.

\bibitem[{Van~Essen et~al.(2013)Van~Essen, Smith, Barch, Behrens, Yacoub,
  Ugurbil, Consortium et~al.}]{van2013wu}
\textsc{Van~Essen, D.~C.}, \textsc{Smith, S.~M.}, \textsc{Barch, D.~M.},
  \textsc{Behrens, T.~E.}, \textsc{Yacoub, E.}, \textsc{Ugurbil, K.},
  \textsc{Consortium, W.-M.~H.} et~al. (2013).
\newblock The {WU}-{M}inn human connectome project: {A}n overview.
\newblock \textit{Neuroimage} \textbf{80}, 62--79.

\bibitem[{Vershynin(2010)}]{vershynin2010introduction}
\textsc{Vershynin, R.} (2010).
\newblock Introduction to the non-asymptotic analysis of random matrices.
\newblock \textit{arXiv preprint arXiv:1011.3027} .

\bibitem[{Yu et~al.(2015)Yu, Wang \& Samworth}]{yu2015useful}
\textsc{Yu, Y.}, \textsc{Wang, T.} \& \textsc{Samworth, R.~J.} (2015).
\newblock A useful variant of the {D}avis--{K}ahan theorem for statisticians.
\newblock \textit{Biometrika} \textbf{102}, 315--323.

\end{thebibliography}
